\documentclass[a4paper,10pt,oneside,openany]{article}
\usepackage[a4paper,
            left=1.6cm,
            right=1.6cm,
            top=2cm,
            bottom=2.5cm]{geometry}
\usepackage[utf8]{inputenc}
\usepackage{graphicx}
\usepackage{amsmath,amssymb}
\usepackage{mathtools}
\usepackage{hyperref}
\usepackage{math}
\usepackage{algorithm}
\usepackage{algpseudocode}
\usepackage{amsthm}
\usepackage{bm}
\usepackage{multirow}
\usepackage{placeins}
\usepackage[noabbrev,nameinlink]{cleveref}
\usepackage{appendix}
\usepackage{booktabs}
\usepackage{tabularx}

\newtheorem{theorem}{Theorem}
\newtheorem{definition}{Definition}

\crefname{appendix}{appendix}{appendices}
\Crefname{appendix}{Appendix}{Appendices}

\DeclarePairedDelimiterX\Exp[1]{\mathbb{E}\!\lbrack}{\rbrack}{#1}
\DeclarePairedDelimiterX\Prob[1]{\mathbb{P}\!\lparen}{\rparen}{#1}
\algrenewcommand\algorithmicensure{\textbf{Input:}}

\numberwithin{equation}{section}

\title{On the Use of Bagging for Local Intrinsic Dimensionality Estimation}
\date{}

\author{%
\parbox{\textwidth}{%
\centering
\begin{tabular}{c@{\hspace{2.5em}}c}
\begin{tabular}[t]{c}
{\large Kristóf Péter}$^{1}$\\[0.25em]
University of Southern Denmark\\
Odense, Denmark\\[0.4em]
\href{mailto:krp@imada.sdu.dk}{\texttt{krp@imada.sdu.dk}}\\
\href{https://orcid.org/0009-0008-1552-3361}{ORCID: 0009-0008-1552-3361}
\end{tabular}
&
\begin{tabular}[t]{c}
{\large Ricardo J. G. B. Campello}$^{1}$\\[0.25em]
University of Southern Denmark\\
Odense, Denmark\\[0.4em]
\href{mailto:email}{\texttt{campello@imada.sdu.dk}}\\
\href{https://orcid.org/0000-0003-0266-3492}{ORCID: 0000-0003-0266-3492}
\end{tabular}
\\[2.2em]
\begin{tabular}[t]{c}
\\[0.25em]
{\large James Bailey}$^{1}$\\[0.25em]
Monash University\\
Melbourne, Australia\\[0.4em]
\href{mailto:email}{\texttt{james.a.bailey@monash.edu}}\\
\href{https://orcid.org/0000-0002-3769-3811}{ORCID: 0000-0002-3769-3811}
\end{tabular}
&
\begin{tabular}[t]{c}
\\[0.25em]
{\large Michael E. Houle}$^{1}$\\[0.25em]
New Jersey Institute of Technology\\
Newark, New Jersey, United States\\[0.4em]
\href{mailto:email}{\texttt{michael.houle@njit.edu}}\\
\href{https://orcid.org/0000-0001-8486-8015}{ORCID: 0000-0001-8486-8015}
\end{tabular}
\end{tabular}%
}}

\begin{document}
\maketitle

\begin{abstract}
\phantomsection
\label{sec:Abstract}

The theory of Local Intrinsic Dimensionality (LID) has become a valuable tool for characterizing local complexity within and across data manifolds, supporting a range of data mining and machine learning tasks.
Accurate LID estimation requires samples drawn from small neighborhoods around each query to avoid biases from nonlocal effects and potential manifold mixing, yet limited data within such neighborhoods tends to cause high estimation variance.
As a variance reduction strategy, we propose an ensemble approach that uses subbagging to preserve the local distribution of nearest neighbor (NN) distances. The main challenge is that the uniform reduction in total sample size within each subsample increases the proximity threshold for finding a fixed number $k$ of NNs around the query. As a result, in the specific context of LID estimation, the sampling rate has an additional, complex interplay with the neighborhood size, where both combined determine the sample size as well as the locality and resolution considered for estimation.
We analyze both theoretically and experimentally how the choice of the sampling rate and the $k$-NN size used for LID estimation, alongside the ensemble size, affects performance, enabling informed prior selection of these hyper-parameters depending on application-based preferences. Our results indicate that within broad and well-characterized regions of the hyper-parameters space, using a bagged estimator will most often significantly reduce variance as well as the mean squared error when compared to the corresponding non-bagged baseline, with controllable impact on bias. We additionally propose and evaluate different ways of combining bagging with neighborhood smoothing for substantial further improvements on LID estimation performance.

\end{abstract}

\section{Introduction}\label{sec:Introduction}

Local Intrinsic Dimensionality (LID) is part of a broader notion of Intrinsic Dimension (ID), which has been conceptualized in different ways aiming to characterize the true complexity of a data space beyond its full representation dimension, under the assumption that one or more data submanifolds may exist within that space \cite{id2,id3,id4,id5,id6,id7,id8}. One can think of LID as the minimum number of variables or \emph{effective dimensions} required to explain the underlying distribution at a particular location of the data space, or the smallest dimension surface on which a sample from that distribution would locally lie on. It has found a growing number of applications, including outlier detection \cite{outlierdetection}, similarity search \cite{DBLP:journals/is/AumullerC21,casanova2017dimensional}, IoT intrusion detection \cite{networkintrusion}, adversarial example analysis \cite{Adversarial,adveserial2}, self-supervised learning regularization \cite{LDReg}, graph embedding design \cite{graph}, data segmentation \cite{segmentation}, and granular deformation analysis \cite{materials}, just to mention a few.

A wide range of LID estimators has been proposed, among which a major class is based on the extreme value theoretic (EVT) limiting distribution of distances as a neighborhood radius around a query location shrinks to zero --- e.g., MLE, MoM, PWM, TLE \cite{EstimatingLID,TLE}. Other well-known  approaches include MADA \cite{MADA}, ESS \cite{ESS}, and originally global ID estimators, such as TWO-NN \cite{TWO-NN}, which can be generalized to locally estimate LID \cite{skdim}. The above techniques employ $k$-NN search for sampling around the query location, however, there are model-based alternatives --- such as LIDL (using approximate likelihood) \cite{LIDL} and FLIPD (employing diffusion models) \cite{diffusion} --- with different ways of quantifying locality.

Despite their differences, most estimators evaluate LID at a query location using only nearby observations to capture local properties of the data space. Using faraway data points risks introducing nonlocal biases, e.g., by mixing nearby manifolds with different dimensionalities, or violating local isotropy on an anisotropic surface --- see Figure~\ref{fig:Importance of Locality}. In practice, given a finite size dataset, shrinking an arbitrary neighborhood radius can leave too few points inside, which tend to yield high-variance, unstable estimates; enlarging it reduces variance, but also pulls in nonlocal structure that can increase bias. This bias-variance tradeoff is the central practical obstacle in LID estimation, and it can undermine downstream applications that critically rely on accurate, stable estimations of LID. Tackling this challenge through bootstrapping-based statistical machinery, rather than simply trading variance with bias, is the main goal of this paper.

\begin{figure}[h]
    \centering
\includegraphics[width=\textwidth,height=\textheight,keepaspectratio]{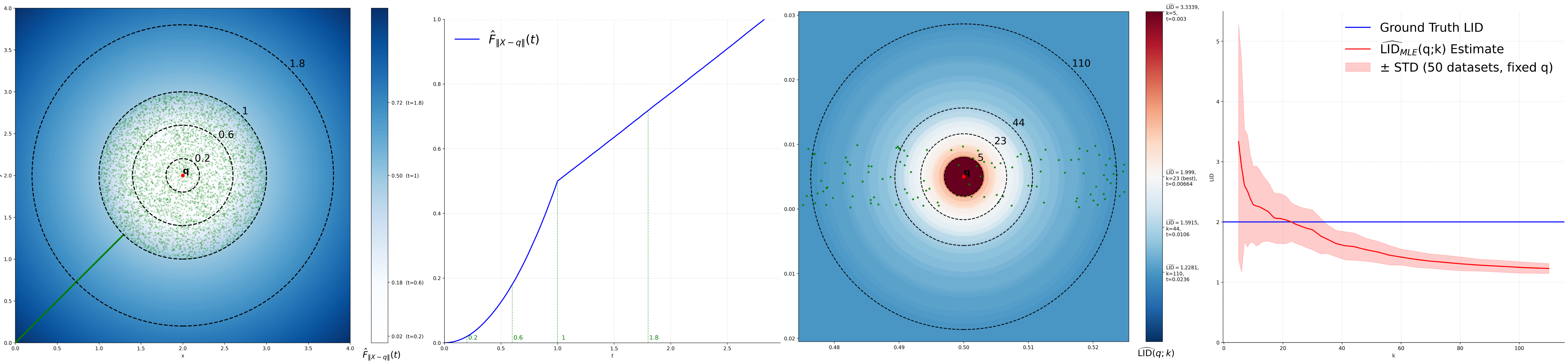}
\caption{The leftmost figure shows the \emph{Lollipop} dataset \cite{LIDL} overlaid with the cumulative distribution function (c.d.f) of the induced empirical distance distribution from the query at $(2,2)$, as a heatmap. The 2nd figure shows that, within close vicinity to the query, $t$, the c.d.f. is proportional to the square of the distance, becoming linear beyond the bounds of the 2D (circle-plate) submanifold, when the 1D (line) submanifold is reached. The 3rd figure shows a heatmap of LID estimates at the query at $(0.5,0.005)$ using MLE \cite{extremevaluetheoretic}, as a function of the $k$-NN hyper-parameter, overlaid with the corresponding dataset distributed uniformly on a thin ribbon-like surface. The rightmost figure displays how LID estimates vary with $k$, not only affected by the interplay between locality preservation and sample size, but also changing with the resolution at which it is measured in practice. Within close vicinity from the query, where neither locality or resolution is critical, the small sample size results in high variance and --- for estimators that are only asymptotically unbiased (including MLE) --- also in bias.}
\label{fig:Importance of Locality}
\end{figure}

We propose to adapt \emph{bootstrap aggregating} (a.k.a. \emph{bagging} for short) \cite{bootstrapping} as a method-agnostic, variance-reduction wrapper for LID estimators. While bagging is a commonly used technique, e.g., in the context of decision trees \cite{randomforests} and ensembles more broadly, both shallow \cite{zimek2013subsampling} and deep \cite{baggingNNref1,baggingNNref2}, to the best of our knowledge it has not been applied to LID estimation before. In this specific setting, the standard behavior of bagging is not trivially guaranteed. To align with continuity assumptions on the underlying data density and to preserve the original distribution of distances, subsampling without replacement is required within each \emph{bag} (i.e., ensemble member), thus causing the in-bag sample size to be reduced. Reducing the sample size increases the distance threshold required to retrieve a fixed number $k$ of neighboring points around a query, making it so that a fixed $k$ on a subsample corresponds to a larger effective NN radius, with potential to introduce undesirable nonlocal effects as variance decreases (e.g., see Figure \ref{fig:Importance of Locality}). The problem has a close analogue in traditional EVT, where the chosen sample size determines the upper-tail threshold. When estimating the extreme value index using the $k$ largest order statistics, the choice of $k$ trades variance (small $k$) against bias (large $k$), and subsample-bootstrap methods show how the MSE-optimal $k$ between subsampled and full-sample estimators can be explicitly rescaled according to the sampling rate \cite{kandr1,kandr2}. This provides critical motivation in the current study for its exploration within the locality restrictions corresponding to a lower-tail threshold for LID estimation.

We provide theoretical and empirical analyses of bagged LID estimation, focusing on how the number of bags $B$ and the sampling rate $r$, together with the locality threshold (the neighborhood size $k$ in case of NN-based LID estimators) jointly shape bias, variance, and mean-squared error (MSE). We show how the selection of these hyper-parameters can result in dramatic reduction of MSE, typically resulting from significant reduction in variance with limited or no associated compromise in bias. Our results support that this is mainly enabled by the additional degree of freedom provided by the sampling rate controlling the tradeoff between the bag size and the level of independence between bags, which often allows for wider ranges of effective locality thresholds for LID estimation, and within which the bagged estimator can achieve superior performance than the baseline equipped with its own MSE-optimal locality threshold. This provides \emph{application-specific} insights as to why and how bagging can be effectively used to mitigate the bias-variance dilemma in LID estimation.

Additionally, we evaluate neighborhood smoothing \cite{smoothing} as another strong standalone variance-reduction approach and show that it is complementary to bagging: applying smoothing within each bag (pre-smoothing) and/or after aggregation (post-smoothing) yields the most substantial MSE reduction in our experiments. Our {\bf contributions} can be summarized as follows:
\begin{itemize}

  \item Introducing bagging for LID estimation.
  
  \item Theoretical analyses shedding further light onto the effects and interactions of the hyper-parameters governing the behavior of bagging.
  
  \item Experimental analyses of the relationship between bagging hyper-parameters as well as their potential interplay with the locality threshold of LID estimators (demonstrated through the $k$ hyper-parameter of NN-based methods).

  \item Experimental results on benchmark datasets showing that bagging, as well as its combinations with smoothing, not only reduce variance but generally reduce the MSE of LID estimates compared to the baseline, in case of mutually optimal independent selection of hyper-parameters.
  
\end{itemize}

\section{Related work}\label{sec:related_work}

Variance reduction in LID estimation is addressed either by designing lower-variance estimators or by applying estimator-agnostic wrappers to stabilize existing methods. Next, we review approaches within these two categories.

TLE \cite{TLE} is closely related to EVT-based LID estimators such as MLE \cite{extremevaluetheoretic}, but it modifies the classic asymptotic distance-based likelihood function by incorporating non-central distance information to derive an extended statistic. It was shown to substantially reduce variance while maintaining comparable bias to MLE. However, TLE is tied to the EVT formulation and is not readily transferable as a generic wrapper to arbitrary LID estimators.

More general stabilization can be achieved via data pre-processing or by post-processing estimates. In \cite{hyperspectral}, it has been shown that pre-processing by scale space filtering methods can contribute to LID estimation in the context of hyperspectral imaging data. These techniques work by denoising the original dataset, thus reducing variance and, accordingly, resulting in more accurate and stable LID estimates. However, the method is designed under assumptions specifically related to the type of datasets in the intended application of hyperspectral images.

A widely applicable post-processing approach is neighborhood smoothing \cite{smoothing-wikipedia}, which is easily translatable to LID estimation. Given per-query LID estimates, regular \emph{smoothing} replaces each by an average over their local neighborhood (e.g., $k$-NN), which can reduce variance assuming that neighboring points share similar LID. To improve alignment with this assumption, the manifold adaptive neighborhood smoothing technique presented in \cite{smoothing} constructs the $k$-NN graph of the data points based on Euclidean distance; smoothing is then performed on a neighborhood of points based on their approximated geodesic distance to the query as the shortest paths along the graph, which locally approximates the query's submanifold. Such \emph{geodesic smoothing} technique has been shown to reduce the variability of the non-linear least-squares estimator described in \cite{smoothing}. 

\section{Background}\label{sec:Background}

\subsection{LID} LID can be defined from a purely distributional perspective, beyond the realm of any particular data sample, as a local notion of ID at a given query location. 

A theoretical, EVT-based approach to this concept that has gained increasing attention in recent years has been formulated in \cite{lidtheory2,LID1}. In this theoretical framework, for a selected query location and distance metric, there exists a \emph{distribution of distances} from the query that is \emph{induced} from an \emph{underlying distribution} of interest defined in the original data space. Under certain assumptions of locality and smoothness, the limiting lower tail of the cumulative distribution function (\emph{c.d.f.}) of such a theoretical distance distribution is fully and uniquely characterized by LID as a quantifiable property at the query. 

In this paper, unless stated otherwise, this is the notion of LID that will be implicitly assumed, e.g., in our experiments involving practical LID estimation. Notice, however, that the bagging techniques introduced in this paper are by no means estimator-specific, but rather broadly applicable to virtually any LID estimator regardless of their supporting LID paradigm.

\vspace*{-2mm}
\subsection{Bagging} Bagging is a resampling technique introduced by Breiman \cite{bagging} as a way of harnessing computational power to improve the stability of models and estimators by aggregating results across multiple resampled versions (\emph{bags}) of the original data. It is especially effective for models that are accurate on average but highly variable. Classic examples include bagging for decision trees and random forests \cite{statlearning,RandomForestsBreiman2001,randomforests}, as well as a variety of modern deep learning applications \cite{baggingNNref2,baggingNNref1}. General theoretical results for the variance-reduction effect of bagging, including the subsampling variant adopted in the present paper, can be found in \cite{analyzingbagging}.

The term \emph{bagging} is commonly referring to resampling \emph{with replacement}, where each bag can have the same size as the original dataset and may contain duplicate samples. In our setting, however, duplicates would distort the distance distribution and violate its absolutely continuous assumption underlying most LID estimators. We therefore use the \emph{subbagging} (\emph{subsample aggregating}) variant, where sampling within an individual bag happens strictly without replacement and each bag thus contains a strictly smaller sample size than the whole dataset.

For this reason and for the sake of simplicity, in the remainder of this paper and unless explicitly stated otherwise, when referring to \emph{bagging} we will assume the following formal definition of \emph{subbagging}, which coincides with that from \cite{analyzingbagging}:

\begin{definition}[Subbagging] \label{def_subagging}
Given $n$ i.i.d. random vectors $D\triangleq(X_1,\dots,X_n)$ following a distribution over the ambient space that relates to a parameter $\theta$, let the $n$ sample estimator of $\theta$ be defined as a statistic $\hat{\theta}_n \triangleq \hat{\theta}(X_1,\dots,X_n)$. Let $m = r n < n$ be an integer for some sampling rate $r \in (0,1)$ and create $B$ random subsets (a.k.a. \emph{bags}) of size $m$ each, $D_{i,m}\triangleq(X_{\Pi_{1}^{(i)}},\dots,X_{\Pi_{m}^{(i)}})$ for $i = 1,\dots,B$, by subsampling from $D$ using random index sets, $\Pi^{(i)} \triangleq \{\Pi_{1}^{(i)},\dots,\Pi_{m}^{(i)}\}$. Such index sets $\Pi^{(\cdot)}$ are i.i.d. (uniformly) over all ${n\choose m}$ possible choices of $m$ different indices out of $n$.

The bagged estimator is then defined as:
\begin{equation}
\begin{aligned}
\label{eqn:130}
& \hat{\theta}_{B,m} \triangleq \frac{1}{B} \sum_{i=1}^B \hat{\theta}_{m, i} % \\
& \text{; ~~~~~~where }\;\;\; \hat{\theta}_{m, i} \triangleq \hat{\theta}(X_{\Pi_{1}^{(i)}},\dots,X_{\Pi_{m}^{(i)}})
\end{aligned}
\end{equation}
\end{definition}

\section{Bagged LID estimators}\label{sec:contributions}

To translate Definition \ref{def_subagging} to the specific context of LID estimation, we consider the sample set $D$ to correspond to our full dataset of $n$ observed points in the $\mathbb{R}^{\dim}$ ambient space, whereas the baseline LID estimator at a fixed query point $q \in \mathbb{R}^{\dim}$, $\widehat{LID}(q)$, is considered to satisfy measurability assumptions and corresponds to the statistic $\hat{\theta}$, such that $\widehat{LID}_n(q)\triangleq\widehat{LID}_n(X_1,\dots,X_n;q)\equiv \hat{\theta}_{n}(X_1,\dots,X_n)$ is the $n$-sample estimator at $q$. We can thereby construct the bags $D_{1,m},\dots, D_{B,m} \subset D$ and the bagged estimator $\widehat{LID}_{B,m}(q)$ according to Definition \ref{def_subagging}.
Figure \ref{fig:simplebagging} illustrates bagged LID estimation for a single query.

Notice that most LID estimators are quantifiable at any location of the data space, albeit they are most often used only at query locations that coincide with points of the sample set ($q \in D$). After subsampling, a query point $q \in D$ will likely only be present in some bags, but the bagged estimator will still require its individual estimation within each and every bag, regardless of whether the query is present or not.

From an implementation viewpoint, the two different cases have to be taken into consideration when a distance of zero between the query location and a sample point is not allowed, such as for any LID estimator under an absolutely continuous distance distribution assumption. For the family of estimators based on the query's $k$-NN, one needs to ensure that the nearest neighbor of a point is never the point itself. In practice, this means that, whenever the query point $q$ is contained in the bag, it is not counted as its own neighbor (for estimating its LID within that bag) and, as such, it does not count towards $k$, thus ensuring the same number of points $k$ is always used for estimation.

\begin{figure}[t]
    \centering
\includegraphics[width=\textwidth,height=\textheight,keepaspectratio]{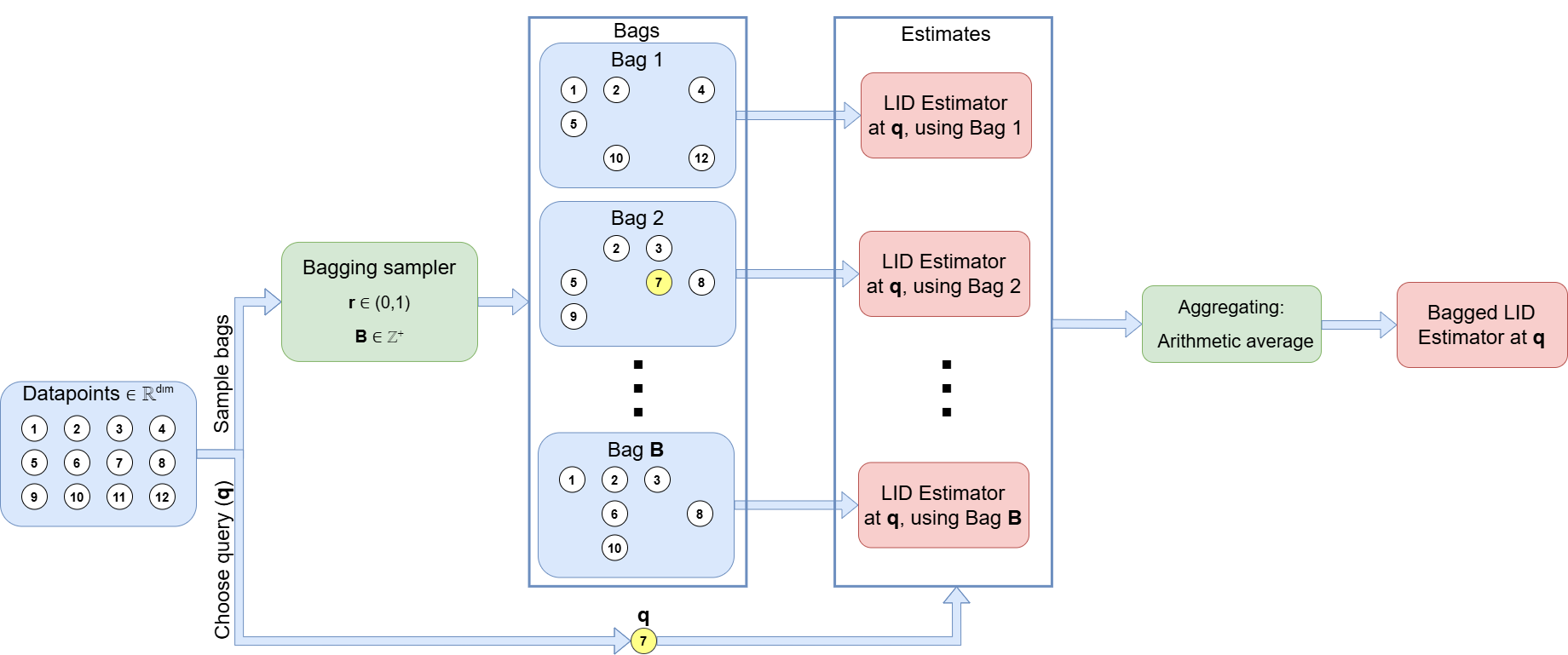}
\caption{Flowchart illustrating bagged LID estimation on a hypothetical, toy dataset. The algorithm is displayed for a single query point (with index 7) from the dataset. Note that it suffices to sample the bags only once and reuse them for different queries. The figure makes it clear that LID estimation can be done independently for each query point and for each bag, so processing is highly parallelizable across multiple nodes. For details on estimation across a dataset of queries, see the pseudo code in Appendix \ref{Asec:Bagging for LID Algorithm}.}
\label{fig:simplebagging}
\end{figure}

\vspace*{-2mm}
\subsection{Mean squared error (MSE) decomposition }\label{th3}
The Mean-squared error (MSE) decomposition is a well-known result \cite{MathematicalStatistics} that relates the theoretic MSE of an estimator to its bias and variance. It can be used to gain deeper insights into an estimator's error by looking at the two components individually, as $\mathrm{E}[(\hat\theta_{B,m}-\theta)^2] = \operatorname{Var}(\hat\theta_{B,m}) + (\mathrm{E}[\hat\theta_{B,m}] - \theta)^2$. Notice that the overall bagged estimator will have at most as large MSE as the $m$-sample estimator deployed within each bag, i.e., $\mathrm{E}[(\hat{\theta}_{B,m}-\theta)^2] \leq \mathrm{E}[(\hat{\theta}_{m,i}-\theta)^2]$, which directly follows by recognizing that $\mathrm{Var}(\hat{\theta}_{B,m}) \leq \mathrm{Var}(\hat{\theta}_{m,i})$ and $\mathrm{E}[\hat{\theta}_{B,m}] = \mathrm{E}[\hat{\theta}_{m,i}]$ (see Appendix \ref{Asec:bagging theory} for details). However, in principle, this result alone does not provide any improvement guarantees for the bagged estimator with respect to the MSE of the $n$-sample estimator, which infers over the entire sample set $D$.

When it comes to variance, \cite{revelas2024doessubaggingwork} shows in the context of subbagging for decision trees that as the number of bags ($B$) tends to infinity, the variance of the bagged estimator decreases towards a known limiting value, which can be expressed as the variance of the expectation of a single-bag estimator conditioned on the data, $\mathrm{Var}(\mathrm{E}(\hat{\theta}_{m,i}|D))$. In Appendix \ref{Asec:bagging theory}, we show more broadly (for any estimator following Definition \ref{def_subagging}) that this lower-limiting variance can also be rewritten as the unconditional covariance between single-bag estimators, $\mathrm{Cov}(\hat{\theta}_{m,i}, \hat{\theta}_{m,j})$.

In Theorem \ref{th2}, we formally outline how decreasing the sampling rate ($r$) is expected to cap the aforementioned limiting variance of the bagged estimator. The result is a combination of the well-known Jensen-gap bound \cite{JensenBound} and the fact that the extent of the overlap $|\Pi^{(i)}\cap\Pi^{(j)}|$ between random index sets drawn without replacement follows the Hypergeometric Distribution \cite{kalinka2014probabilitydrawingintersectionsextending,GrinsteadSnell2003IntroProb}. 

\begin{theorem}\label{th2}
Define $\gamma(h,m) \triangleq \mathrm{Cov}(\hat{\theta}_{m,i}, \hat{\theta}_{m,j} \;|\;|\Pi^{(i)}\cap\Pi^{(j)}|=h)$ for the integer domain given by $m = r n \in \mathbb{Z^+}$ and $h=0,1,\dots,m$, namely, the covariance of two single-bag estimators given that we know the number of dependent variables (provided that such a covariance exists and is finite for any pair of bags). Assume that there exists a twice differentiable, increasing function $\varphi:[0,1] \rightarrow \mathbb{R}$ such that $\forall_{h, m\;: \;h \leq m } \; \gamma(h,m) \leq \varphi(\frac{h}{m})$, then we have:
\begin{equation}
\begin{aligned}
\label{eqn:131.1}
\forall_{r \in [0,1]} \;\;\mathrm{Cov}(\hat{\theta}_{m,i}, \hat{\theta}_{m,j})\leq \varphi(r) + O\left(\frac{1}{n}\right)
\end{aligned}
\end{equation}
\end{theorem}
\begin{proof}
See proof in Appendix \ref{Asec:proof2}
\end{proof}

In simple terms, assuming that the magnitude of the conditional covariances $\gamma(h,m)$ are roughly positively proportional to the fraction of dependent variables between the bags ($h/m$), Theorem \ref{th2} shows that we can set an implicit upper bound on the limiting variance of the bagged estimator, $\lim_{B \rightarrow \infty}\mathrm{Var}(\hat{\theta}_{B,m}) = \mathrm{Cov}(\hat{\theta}_{m,i}, \hat{\theta}_{m,j})$, by reducing $r$ and, accordingly, $\varphi(r)$. This result suggests that we can limit variance from above by decreasing the sampling rate ($r$) and increasing the number of bags ($B$). However, recall that the resulting decrease in the bag size $m$ can lead to more biased in-bag estimates, most noticeably, due to either fewer samples for LID estimation within a fixed local vicinity around the query or, alternatively, a less local vicinity containing a fixed sample size $k$ (e.g., see Figure \ref{fig:Importance of Locality}). Since the exact nature of this bias-variance tradeoff is dependent on the particular LID estimator and its hyper-parameters (e.g., $k$), as well as on the dataset in hand, experimental analyses are required to conclusively show the effects of bagging and its hyper-parameters in LID estimation.

\section{Experimental methodology}\label{sec:Experimental methodology}

\subsection{Estimators} Alongside the classic MLE estimator as a standard baseline \cite{EstimatingLID}, we also include TLE \cite{TLE} and MADA \cite{MADA} to test the general applicability of bagging to different estimators without any intent of promoting a particular method. TLE, while substantially more computationally demanding, is designed to reduce variance, allowing us to test whether bagging can further enhance more sophisticated baselines already equipped with internal variance-reduction mechanisms. Table~\ref{table:estimators} summarizes the tested estimators and their essential hyper-parameters, alongside the two wrapper alternatives to mitigate variance that we experimentally investigate in this paper, namely, bagging and smoothing \cite{smoothing}. The goal is to compare each baseline against their own wrapped counterparts.

\begin{table}[ht]
  \centering
  \caption{LID estimators and hyper-parameters. ``pre'' and ``post'' specify whether smoothing is applied to the in-bag LID estimates (i.e. prior to aggregation) or to the bagged LID estimates, resp. All other hyper-parameters have been previously defined.} \vspace*{2mm}
    \begin{tabular}{|p{3cm}|p{3.5cm}|p{3cm}|p{3.5cm}|p{3cm}|}
    \hline
    Baseline & Baseline hyper-par. & Bagging hyper-par. & Smoothing strategy\\
    \hline
    MLE~\cite{EstimatingLID}  & $k$ & $r$, $B$ & \textit{pre} or \textit{post} \\
    TLE~\cite{TLE}            & $k$ & $r$, $B$ & \textit{pre} or \textit{post} \\
    MADA~\cite{MADA}          & $k$ & $r$, $B$ & \textit{pre} or \textit{post} \\
    \hline
    \end{tabular} \label{table:estimators}
\end{table}

\vspace*{-2mm}
\subsection{Datasets}\label{sec:datasets} As common practice in the LID literature, we use test datasets based on a foundational study of LID estimation on manifolds \cite{OriginalLIDmanifolds} later extended to a comprehensive benchmark framework in \cite{BenchmarkingLIDmanifolds}. The data manifolds contain $n=2500$ points with known LID, which we use as ground truth when computing the MSE of estimators for evaluation purposes. We additionally include the \emph{Lollipop} data from \cite{LIDL}, and a high-dimensional \emph{Uniform} dataset. More detailed descriptions of each dataset and a summary table are provided in Appendix \ref{Asec: Apppendix Dataset descriptions}.

\vspace*{-2mm}
\subsection{Evaluation measures}\label{Evaluation Measures}
We measure the performance of LID estimators using the empirical MSE as well as its bias-squared plus variance decomposition, for the reasons discussed in Section \ref{th3}. Notice that in the presence of multiple ground-truth LID values across a dataset with different submanifolds, the MSE decomposition needs to be applied manifold-wise (details are provided in Appendix \ref{Asec:MSE appendix}).

\vspace*{-2mm}
\subsection{Bagging and smoothing experiments}\label{Bagging and Smoothing}

The first experiment evaluates the overall effectiveness of bagging as well as compares and combines it with \emph{smoothing}, as an alternative or supplementary variance-reduction technique. We obtain smoothed LID estimates for a query by taking the arithmetic average of the estimates over its $k$-NN (the same neighborhood size used by the baseline estimator). This neighborhood includes the query itself only when it is present in the reference sample set (bag or full dataset). When smoothing is performed on the baseline estimates, without any bagging, we will refer to it simply as \emph{smoothing}.

Additionally, we combine bagging and smoothing in three different ways. The simplest idea is to smooth out the aggregated bagged estimates themselves, exactly as explained above, only now for the bagged estimator. We will refer to this combination as ``\emph{post-smoothing}''. 

As an alternative approach, we can instead smooth out estimates based on each bag separately, determining the query's $k$-NN among only the in-bag points. This way, smoothing takes place inside the bags, before the single-bag estimates are aggregated. We refer to this combination as ``\emph{pre-smoothing}''.
Finally, we can also integrate both strategies above by simultaneously applying pre-smoothing (\emph{pre}) and post-smoothing (\emph{post}).

This goes to show that bagging and smoothing are not necessarily competitors. Bagging actually offers new avenues for applying smoothing in different ways. 

We apply bagging, smoothing, and their three combinations independently to the different baseline estimators in Table \ref{table:estimators}, \emph{with the goal to compare each baseline to their bagged/smoothed counterparts, rather than against each other}.

For each case, we sweep the locality hyper-parameter over a 9-step geometric progression from $k=5$ to $k=72$. For bagged variants, we additionally sweep the sampling rate over a 9-step geometric progression from $r=0.042$ to $r=0.6$, evaluating all combinations of $(k,r)$. We are fixing the number of bags to $B=10$, noting that increasing $B$ can only further improve results in favor of bagging. 
Then, for each dataset and method variant, we individually select the hyper-parameter setting that minimizes MSE: $k$ for baseline/smoothing, and $(k,r)$ for all bagged variants, comparing the methods at their best observed performance. 

\vspace*{-2mm}
\subsection{Bagging hyper-parameters experiments}\label{Bagging parameters}

The second set of experiments focuses specifically on bagging and explores how the choice of its hyper-parameters affects performance in terms of MSE, variance, and bias. These effects are discussed using the bagged MLE estimator. The results for MADA and TLE (in Appendix \ref{Asec:Supplementary results}) exhibit similar trends and the main conclusions do not change: 

\paragraph{1st test (sampling rate):}\label{sampling rate}

We saw in Theorem \ref{th2} how decreasing $r$ is expected to de-correlate the bags, and under certain general assumptions, reduce variance. Confirming this expected decrease in variance in the LID setting and to what extent it can overshine any potential increase in bias as a function of the sampling rate ($r$) are the main goals of this experiment. 

For this test, we chose a fixed $k=10$ for the baseline estimators, a relatively small value compared to the dataset size, to attain to the general importance of locality surrounding LID estimation. The effects of $r$ may be exaggerated or reduced depending on $k$; however, it is worth stressing that the general trends remain the same and the main conclusions with respect to $r$ do not change for different, fixed values of $k$.

\paragraph{2nd test (interaction between $k$ and $r$):}\label{Interaction of k and sr}

As previously discussed, both $k$ and $r$ affect the radius determining the local neighborhood around the query used for estimation, which is directly proportional to $k$ and inversely proportional to $r$, with the potential to affect bias through non-local effects and sample size. On the other hand, variance is generally inversely proportional to $k$ and directly proportional to $r$. If both hyper-parameters are variable, then the values of $k$ that result in more favorable bias-variance tradeoffs will depend on $r$ and vice versa. Since the exact relationship may be complex and possibly both dataset- and estimator-dependent, the goal of this experiment is to identify major general trends allowing clear guidance on the combined choice of these hyper-parameters.

We evaluate the bagged estimator over a grid of $(k,r)$ combinations, comparing each against the baseline with the same $k$. The resulting MSE differences reveal ranges of $r$ (and ratios $k/r$) for which bagging improves or worsens performance. 

\paragraph{Additional tests (number of bags):} 

Analogous experimental setups and analyses demonstrating the (more straightforward) effects of the number of bags ($B$) and its interaction with $r$ are available in Appendices A.3 and A.4.

\vspace*{-2mm}
\subsection{Code and data availability}\label{sec:Data availability}
A full code that can be used to reproduce all results in our paper is available from our GitHub page at \newline \href{https://github.com/Campello-Lab/Bagging_for_LID}{https://github.com/Campello-Lab/Bagging\_for\_LID}, which also contains all the datasets used in our experiments. Most of these datasets were generated using the Scikit-dimension \cite{skdim} public Python library,  which was also the source code for the baseline LID estimators adopted in this study.

\section{Results and discussion}\label{sec:results}

\subsection{Bagging and smoothing results}\label{results:Bagging_Smoothing}

Figure~\ref{fig:Bagging and smoothing spider chart} shows the results for the first experimental setup in Section~\ref{Bagging and Smoothing}. We use radar charts to display the \emph{optimal} MSE values, min-max normalized across variants, separately for each dataset, in reverse.\footnote{Tables of the raw MSE values are reported in Appendix \ref{Asec:Supplementary tables for the bagging and smoothing experiments}.} Thus, scores lie in $[0,1]$, where $1$ denotes the lowest relative MSE and $0$ the highest, therefore, larger covered area indicates better performance.

\begin{figure}[t]
    \centering
\includegraphics[width=\textwidth,height=\textheight,keepaspectratio]{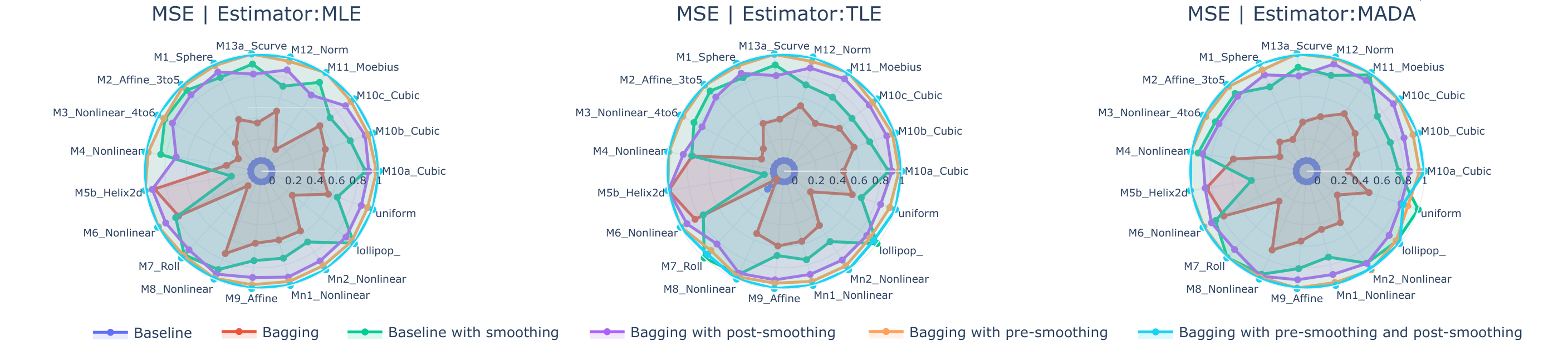}
\caption{Comparison of the relative MSE achieved by three different LID estimators, MLE, TLE, and MADA, with and without smoothing, bagging, and three strategies for combining bagging with smoothing. The results are for $19$ datasets using case-by-case optimal $k$ and $r$ hyper-parameters. The min-max normalized MSE values are subtracted from $1$ before plotting, such that larger scores correspond to smaller relative MSEs.}
\label{fig:Bagging and smoothing spider chart}
\end{figure}

There is a very clear trend. First, both bagging and smoothing improve over the baseline in almost all cases when used independently, with only one exception (TLE on M7\_Roll), noting that smoothing is often the stronger standalone variant in these experimental conditions. Second, it is clear that combining them yields further gains, with a common performance ordering, consistently across the baselines: baseline $<$ bagged $<$ smoothed $<$ bagged with post-smoothing $<$ bagged with pre-smoothing $<$ bagged with pre- and post-smoothing. Therefore, the main conclusion is that not only do both bagging and smoothing achieve their purpose on their own, but combining them provides further improvements.

\vspace*{-2mm}
\subsection{Sampling rate results (1st test on hyper-parameters):}\label{results:Sampling rate test}

\begin{figure}[t]
    \centering
\includegraphics[width=\textwidth,height=\textheight,keepaspectratio]{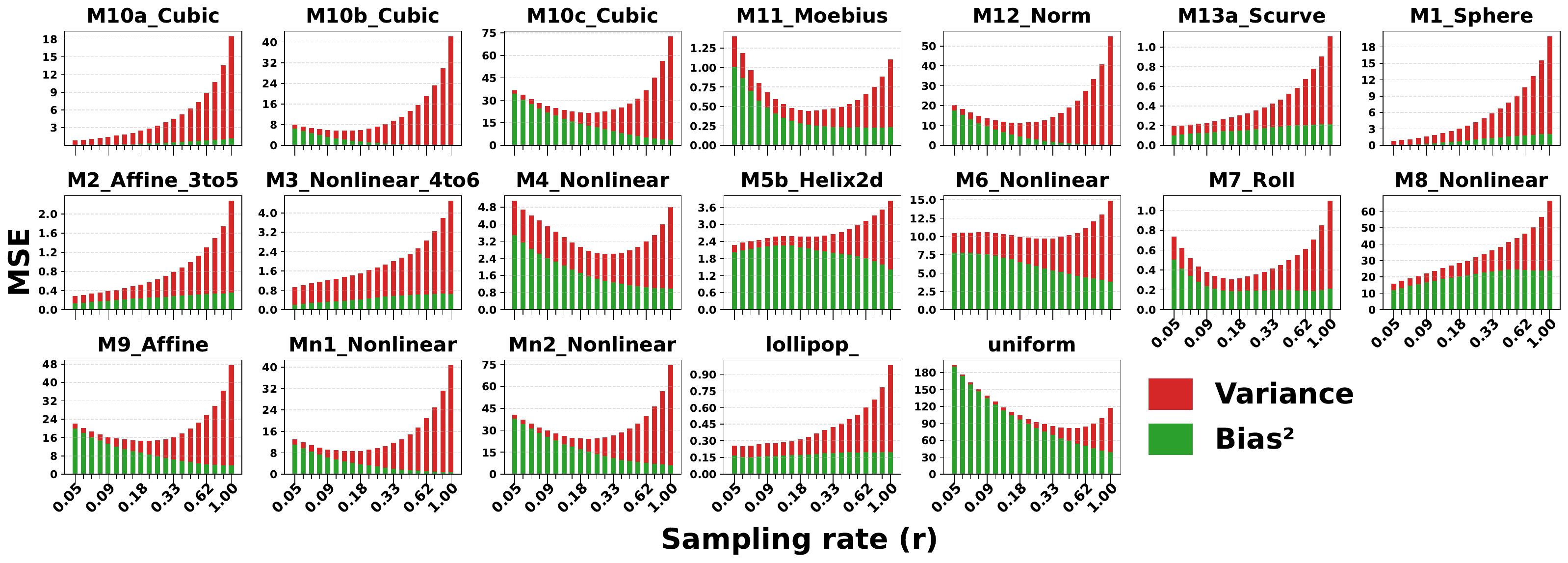}
\caption{MSE and its decomposition for each of the 19 datasets as a function of the sampling rate $r$ used for bagged MLE as the LID estimator. Note that the baseline MLE is equivalent to $r=1$, displayed on the rightmost bar of the individual charts. %These values are not normalized as the hyper-parameter effects are most informative in absolute terms.
}
\label{fig:Sampling rate}
\end{figure}

Figure \ref{fig:Sampling rate} illustrates the experimental behavior of the MSE decomposition for bagging as a function of the sampling rate $r$. The exact relationship is dataset-dependent, but in general, by decreasing $r$ (right-to-left along the x-axis) we observe that: (i) bias$^2$ (green sub-bars) tends to increase or stay relatively stable in most cases; and (ii) variance (red sub-bars) shows a highly consistent, decreasing behavior. This robust effect on variance reduction supports the validity of our assumptions in Theorem \ref{th2} and confirms the bag de-correlating benefits of reducing $r$. 

The opposing trends often yield an intermediate ``sweet spot'' where an optimal bias-variance tradeoff is achieved in terms of a minimal MSE, confirming our arguments in Section \ref{th3}. However, for some datasets (M10a\_Cubic, M13\_Scurve, M1\_Sphere, M2\_Affine\_3to5, M3\_Nonlinear\_4to6, and M8\_Nonlinear), smaller $r$ reduced both variance and bias, so the best MSE occurred at the smallest tested $r$, suggesting that potential non-local biases associated with decreasing $r$ did not manifest within the given range of $r$ values in those datasets. 

Overall, the experiment mostly confirms the anticipated behavior for $r$, noting that results that were somewhat surprising, actually did so in a favorable way.

\vspace*{-2mm}
\subsection{Interaction between $k$ and $r$ results (2nd test on hyper-parameters):}\label{result:Interaction of k and sr}

\begin{figure}[t]
    \centering
\includegraphics[width=\textwidth,height=\textheight,keepaspectratio]{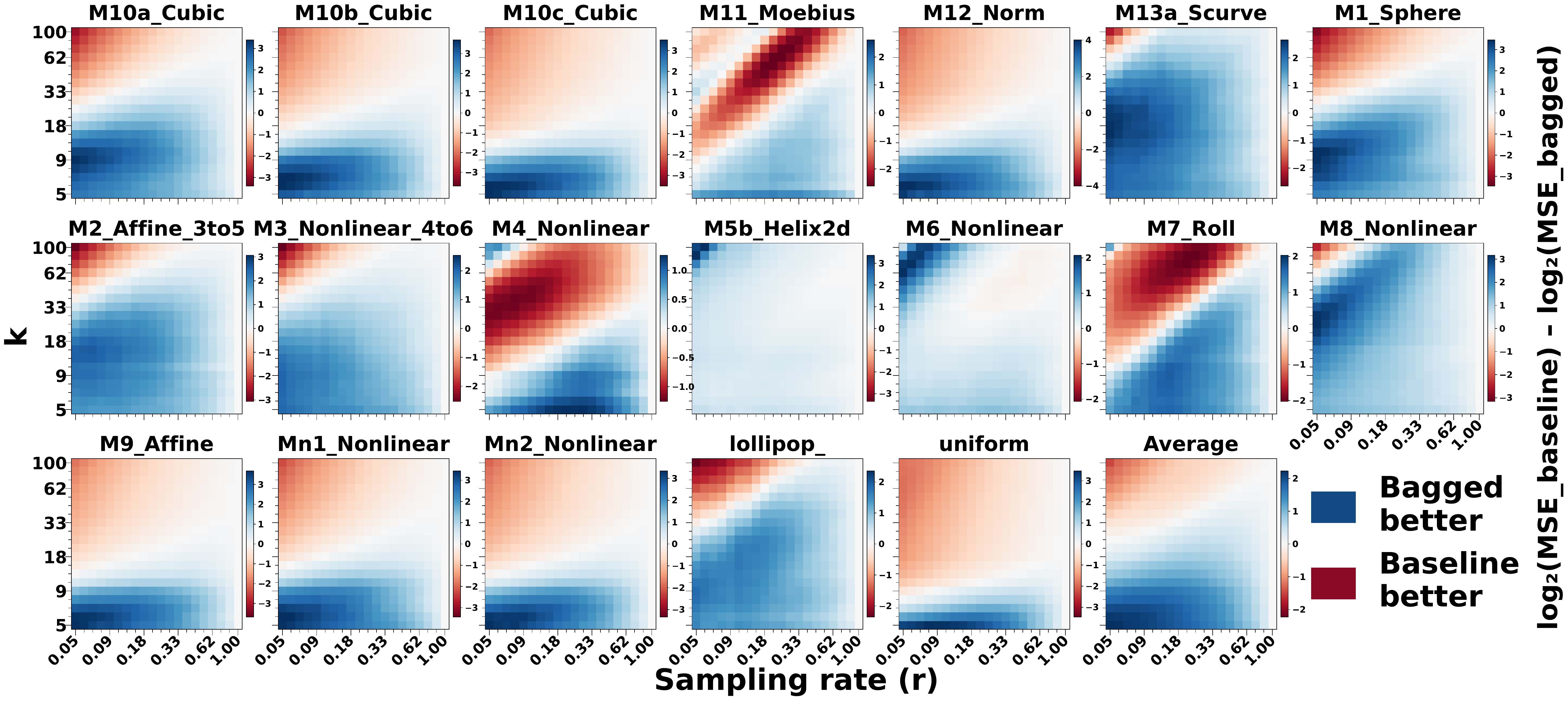}
\caption{Heatmaps of the relative difference (log-ratio) between the MSE achieved by the baseline estimator (MLE) and the MSE of its bagged counterpart for each of the 19 datasets and cross-combinations of values for the sampling rate $r$ (x-axis) and the $k$-NN neighborhood size $k$ (y-axis). Positive values, color-coded as blue, indicate that bagging outperforms the baseline, whereas negative values, color-coded in red, indicate the opposite. The relation is symmetric around zero and the darkness of the cells are proportional to the magnitude of the corresponding absolute value for the given dataset.}
\label{fig:Interaction of ksr}
\end{figure}

Figure~\ref{fig:Interaction of ksr} illustrates how the bagged LID estimator's performance depends jointly on the baseline's $k$-NN hyper-parameter and the sampling rate. As previously conjectured, the bias increase from the expansion of the effective local neighborhood radius, for sufficiently large $k$ compounding with sufficiently small $r$, may overpower the combined variance reduction caused by both these settings, potentially resulting in a larger MSE for the bagged estimator than the baseline.

Overall, the heatmaps in Figure \ref{fig:Interaction of ksr} confirm this expectation, as for most datasets, the hyper-parameter space is divided along a certain left-to-right upwards white line where the two estimators achieve similar MSEs. To the right of this line, bagging tends to win out, while to the left, the baseline tends to perform better. This suggests that there is some dataset specific constant, for which, if the ratio between the two hyper-parameters ($\frac{k}{r}$) stays under, then we can expect the bagged estimator to produce favorable results, while the results are most likely not desirable otherwise. We additionally observe that the strongest gains usually occur at the bottom-left region (small $k$, small $r$), where bagging can strongly reduce MSE by simultaneously de-correlating bag estimates while preserving reasonable locality. The experiment also clarifies how in the previous test with a variable $r$, fixing $k = 10$ may have caused partially unexpected results for some datasets, as non-local effects were not yet expressed. 

\vspace*{-2mm}
\subsection{Main takeaways}\label{result:Main takeaways} From our theoretical and experimental findings, we provide the following guidelines for selecting the hyper-parameters of bagging ($B$, $r$) with respect to the degree of locality for LID estimation (quantified by the $k$ hyper-parameter of the family of NN-based estimators adopted in this study):

\begin{itemize}
    \item \textbf{Number of bags ($B$)}: It is well-known that increasing $B$ has a strictly beneficial but diminishing effect in terms of estimator variance (see Section \ref{th3}), and this has been confirmed in our experiments (Appendices A.3 and A.4). Since the resulting gains come with no counter-effect in terms of bias, this hyper-parameter is not critical. However, the runtime complexity of the bagged estimator is linearly dependent on $B$ (see Section \ref{sec:Computational complexity}). Our experiments consistently support that increasing $B$ beyond $10$ is only recommended if one strongly favors further diminishing gains in stability over runtime.
    
    \item \textbf{Sampling rate ($r$) and locality ($k$)}: As anticipated in Section \ref{th3} and confirmed by our experiments, the hyper-parameters $k$ and $r$ jointly control the bias-variance tradeoff. If the primary goal is to bind variance reduction in the bagged estimator to a positive (or at least neutral) overall impact \emph{in terms of MSE} as compared to the baseline, our results support combining low values of both $k$ and $r$. In practice, we suggest to set $k$ roughly within $[5,10]$ or so as a minimal local sample size for LID estimation, then choose $r$ accordingly, roughly within $[r_{min},0.5]$. The low end of this range should ensure that the in-bag subsamples are still representative enough of the relevant structures in the dataset (e.g., manifolds, classes, clusters) and, as such, it is problem-dependent. Within the (MSE-orientated) rough ranges of $k$ and $r$ suggested above, for a fixed chosen $k$ decreasing $r$ is theoretically expected to decrease variance (Theorem \ref{th2}), often at the expense of bias. In the context of LID, applications that can tolerate large estimation biases, e.g., requiring mostly a reliable ordering for the purpose of relative comparisons of LID values across different queries, may benefit from smaller sampling rates (closer to $r_{min}$), whereas in applications where systematic LID offsets can be critically detrimental, one might prefer larger, more conservative rates. 
    
\end{itemize}

\section{Complexity analysis}\label{sec:Computational complexity}

In this section, we summarize the time complexity of bagging for LID estimation (in a non-parallel setting). An extended analysis, including the different smoothing variants (bagging with pre/post-smoothing) is deferred to Appendix \ref{Asec:Complexity analysis for the bagged LID estimator and its combinations with smoothing}.

Let $p(n)$ denote the runtime cost of computing a baseline LID estimate for a \emph{single query} using a sample of size $n$. Therefore, the \emph{total} runtime complexity for the baseline estimator (no bagging), $T_{\text{base}}$, is $O\!\left(n\,p(n)\right)$ when taken over $n$ different query locations. Bagging with $B$ bags of size $m=rn$, for a \emph{single query}, evaluates the estimator on each bag and averages, resulting in $O\!\left(B\,p(rn)\right)$ time, where the sampling step for bag construction and final averaging contribute only lower-order terms ($O(Bnr)$ and $O(B)$) and can thus be omitted. As the bags do not have to be resampled per-query, the \emph{total} runtime complexity of the bagged estimator taken over $n$ queries, $T_{\text{bag}}$, is simply  $O\!\left(Bn\,p(rn)\right)$. 

For the wide family of NN-based LID estimators, $p(n)$ is dominated by finding the $k$-NN of the query in the sample set of size $n$, where $k$ is assumed to be either constant or $k<<n$. Using a naive exhaustive search, $p(n)$ is $O(n)$, which translates to $T_{\text{base}} \rightarrow O(n^2)$ and $T_{\text{bag}} \rightarrow O(B\,r\,n^2)$. For large enough $n$, bagging may be faster as long as $r < 1/B$.
Assuming a more sophisticated (indexed) search strategy, with average estimation time $p(n)$ of order $O(\log n)$ per query (amortized over all queries), then $T_{\text{base}} \rightarrow O(n\log n)$ and $T_{\text{bag}} \rightarrow O\!\left(B\,n\,\log(rn)\right)$.

\section{Limitations and future work}\label{sec:Limitations}

The objective evaluation measures we used, namely, MSE, Variance and Bias, required datasets commonly used for benchmarking, with ground-truth LID. While these datasets undoubtedly pose a challenge for estimators, due to their manifold curvatures within higher-dimensional spaces, they may not fully span the spectrum of complex behaviors found in real data. Therefore, LID bagging, smoothing, and their combinations as proposed in this paper, are yet to be fully tested in practical scenarios with potentially more complex data distributions, noise, and related challenges. Without ground-truth LID though, this will require indirect evaluation in downstream task applications, which is left for future work.

In future work, we also intend to investigate the use of the out-of-bag (OOB) samples from each bag for evaluation and (unsupervised) model selection.

\section{Conclusions}\label{sec:Conclusion}

We proposed bagging as a variance reduction strategy for LID estimation, alongside different strategies for its combination with neighborhood smoothing. We theoretically and experimentally explored the behavior of bagged LID estimators, with especial focus on the interplay between the sampling rate $r$ and the locality threshold (neighborhood size $k$), whereby in-depth analyses of their joint effect on variance, bias, and overall MSE have been provided. Our results show that within a wide range of these hyper-parameters, clearly characterized by small values of $k$ and small-to-moderate values of $r$, the bagged estimator is not only expected to reduce variance but it also tends to outperform the corresponding baseline estimator in terms of MSE. The higher level of freedom made available by these hyper-parameters in controlling the bias-variance tradeoff also allows a reduction in MSE to be systematically observed when comparing the independently optimal bagged and baseline estimators, i.e., when they are independently set to their best preferred hyper-parameter values for each dataset. These improvements have been achieved across different baseline LID estimators and with a number of bags as small as 10, which comes with little to none additional computational price. Finally, our results robustly show that significant further improvements in performance can be achieved by combining bagging with neighborhood smoothing, which systematically outperforms both standalone approaches. 

\subsection{Use of Generative AI} We used ChatGPT 5.4 strictly to assist with language style and conciseness.

\subsection{Disclosure of Interests}
The authors have no competing interests to declare that are relevant to the content of this article. 

\appendix

\newpage

\section{Supplementary contributions}\label{Asec:Supplementary contributions}

\subsection{Bagging for LID Algorithm}\label{Asec:Bagging for LID Algorithm}
Below is the general algorithm for bagged LID estimation across the whole sample set $D$ considered as query locations, as referenced by the main paper in Section \ref{sec:contributions}. 

Algorithm \ref{alg:Bagging for LID Algorithm} shows that one efficient way of handling the problem of the query being in the bag or otherwise is to loop over the points of the given bag, and then its corresponding complement (also known as the \emph{out-of-bag}). This way, we avoid the need to check for inclusion/exclusion for each query-bag combination. This is especially helpful when the estimator is based on ordered nearest neighbor (NN) distances, in which case, definite knowledge of $q \in D_{i,m}$ or $q \notin D_{i,m}$ can help us decide if the smallest distance should be considered for estimation or not, even if NN distance calculation is only approximate.

\begin{algorithm}[H]\label{alg:Bagging for LID Algorithm}
\caption{Bagged LID estimator across all queries}
\label{alg:Simple Bagging}
\begin{algorithmic}[1]
\Ensure $D \in \mathbb{R}^{n\times \text{dim}}$, $1<< n\in \mathbb{Z}^+$, $\text{dim} \in \mathbb{Z}^+$, $r \in \left[\frac{1}{n},1\right)$, $B \in \mathbb{Z}^+$
\Require Baseline LID estimator oracle $\widehat{LID}_n: \mathbb{R}^{n\times \text{dim}}\times\mathbb{R}^{\text{dim}} \rightarrow \mathbb{R}$ 
\State $m \gets \lceil n\cdot r\rceil$
\For{$i=1$ to $B$}
  \State Sample $\pi^{(i)} \sim \Pi$ ~~ (uniformly over all ${n\choose m}$ possible choices of $m$ different indices out of $n$)
  \State $D_{i, m} \gets \left(D[\pi_{1}^{(i)}], \dots, D[\pi_{m}^{(i)}]\right)$
\EndFor
\For{$i=1$ to $B$}
    \For{$q \in D_{i, m}$}
    \State $\widehat{LID}_{m, i}(q) \gets \widehat{LID}_{m-1}(D_{i, m}\setminus \{q\} \:; q)$
    \EndFor
    \For{$q \in D \setminus D_{i, m}$}
    \State $\widehat{LID}_{m, i}(q) \gets \widehat{LID}_{m}(D_{i, m} \:; q)$
    \EndFor
\EndFor
\For{$q \in D$}
\State $\widehat{LID}_{B,m}(q) \gets \frac{1} {B}\sum_{i=1}^B \widehat{LID}_{m, i}(q)$
\EndFor
\State \Return $\left[\widehat{LID}_{B,m}(D[1]), \dots, \widehat{LID}_{B,m}(D[n])\right]$
\end{algorithmic}
\end{algorithm}

\subsection{Bagging theory revisited}\label{Asec:bagging theory}

In the following, we discuss two results that show the general effectiveness of bagging and demonstrate the roles of its hyper-parameters, namely, the number of bags, $B$, and the sampling rate, $r \triangleq m/n \in (0,1)$, as supporting material to the theoretical analysis in Section \ref{sec:contributions} of the main paper:

\begin{itemize}
    \item Theorem \ref{th1} focuses on the asymptotic variance reduction effect of increasing the number of bags, and has analogous or similar results available in the literature, usually in the general context of classifiers or decision trees, for example, in \cite{revelas2024doessubaggingwork}, which also examines properties of the subbagging alternative. However, the current formulation in the specific context of parameter estimation is particularly convenient within our scope to interpret/explain our results.
    \item Theorem \ref{th2} is concerned with the variance reduction property of decreasing the sampling rate, and is already presented in the main paper; however, this section of the appendix offers additional insights and explanations to help with interpreting its assumptions and its claims.
    \newline
\end{itemize}
 
\begin{theorem}\label{th1}
Let $\hat{\theta}_{m,i}$ and $\hat{\theta}_{m,j}$ be the single bag estimators for bags $i$ and $j$ ($i \neq j$) of a mutual sample set $D$, respectively, as defined in \ref{def_subagging}. Then, the following expressions hold for the unconditional variance of the bagged estimator, $\mathrm{Var}(\hat{\theta}_{B,m})$, defined over $(D, \Pi^{(1)},\dots,\Pi^{(B)})$:

\begin{equation}
\begin{aligned}
\label{eqn:131}
& \mathrm{Cov}(\hat{\theta}_{m,i}, \hat{\theta}_{m,j}) ~~\leq~~ \mathrm{Var}(\hat{\theta}_{B,m}) ~~=~~ \mathrm{Var}(\hat{\theta}_m)\left(\rho_{m} + \frac{1-\rho_{m}}{B}\right) ~~\leq~~ \mathrm{Var}
(\hat{\theta}_m) \\
& \text{and } \; \mathrm{Var}(\hat{\theta}_{B,m}) \rightarrow \mathrm{Cov}(\hat{\theta}_{m,i}, \hat{\theta}_{m,j})\; \text{ as }\; B \rightarrow \infty
\end{aligned}
\end{equation}
\end{theorem}

\noindent where $\rho_{m} \triangleq \mathrm{Corr}(\hat{\theta}_{m,i}, \hat{\theta}_{m,j})$ and $\mathrm{Cov}(\hat{\theta}_{m,i}, \hat{\theta}_{m,j})$ stand for the correlation and covariance of the single bag estimators, respectively.

\begin{proof}
See proof in Section \ref{Asec:proof1}.
\end{proof}

Theorem \ref{th1} illustrates the variance profile of the bagged estimator in terms of the number of bags ($B$) and shows that it is limited between the covariance of the single bag estimators (as a lower-bound) and the variance of the $m$ sample estimator $\hat{\theta}_m$ (as an upper-bound). Assuming that for some $m'$ we have $\mathrm{Cov}(\hat{\theta}_{m',i}, \hat{\theta}_{m',j}) < \mathrm{Var}(\hat{\theta}_{n})$, then by properly selecting the sampling rate $r$ such that $m' = n \cdot r$ and using a sufficiently large $B$, improvement can be achieved in relation to the $n$ sample estimator as the bagged estimator variance tends to the smaller $\mathrm{Cov}(\hat{\theta}_{m',i},\hat{\theta}_{m',j})$ limiting value.

It is worth noticing that while using a smaller subsample size $m$ (or equivalently, a smaller sampling rate $r$) tends to increase the upper-bound term, $\mathrm{Var}(\hat{\theta}_m)$, it conversely tends to decrease the lower-bound term, $\mathrm{Cov}(\hat{\theta}_{m,i}, \hat{\theta}_{m,j})$, by increasing the level of independence between the bags. Given that the variance of the bagged estimator converges towards the latter as we increase the number of bags, this result suggests that low sampling rates associated with a large number of bags would allow for variance reduction as compared to the $n$ sample estimator. Of course, this result does not say anything in terms of bias. However, under the assumption that the $n$ sample estimator is unbiased (potentially asymptotically as $n$ increases), then so is the $m$ sample estimator and (trivially from \eqref{eqn:130}) the bagged estimator as well.

Recalling Theorem \ref{th2} it is clear that, asymptotically as $n \rightarrow \infty$ (i.e., as the size of the original sample set $D$ tends to infinity), the theoretical lower limiting value for the variance of the bagged estimator as stated in Theorem \ref{th1}, $\mathrm{Cov}(\hat{\theta}_{m,i}, \hat{\theta}_{m,j})$, is bounded from above by an increasing function of $r$, as long as the assumptions in the theorem hold.

In simple terms, Theorem \ref{th2} says that if the estimator and data are such that the magnitude of the conditional bag covariances $\gamma(h,m)$ are roughly positively proportional to the fraction of dependent variables between the bags ($h/m$), then from Theorem \ref{th2} we can expect to be able to reduce the upper bound on unconditional bag covariance (lower limiting variance of the bagged estimator) by decreasing the sampling rate ($r$). The assumption is natural in the sense that it captures our intuition about the covariance of the dependent estimators: if a small (large) portion of variables are dependent, we expect lower (higher) covariance. If all variables are independent, i.e. $\frac{h}{m} = 0$, then the covariance is $0$, while if they are all dependent, i.e. $\frac{h}{m} = 1$, then the covariance is maximized as variance. For values in between, so when $0 < \frac{h}{m} < 1$, we can intuitively expect a roughly increasing relationship in terms of $\frac{h}{m}$, and this notion is captured by the assumed upper-bound. 

By knowing more about the behavior of $\theta$ and therefore $\gamma(h,m)$, we could possibly choose $\varphi$ in such a way that it is a good approximation of $\gamma(h,m)$, yielding a tight upper-bound. Apart from particular choices and their usefulness (or lack thereof) for specific purposes, for the sake of explaining the overall relationship in terms of the sampling rate, this general theorem suffices.

Note that by the nature of the expression for $\mathrm{Var}(\hat{\theta}_{B,m})$ in \eqref{eqn:131}, increasing the number of bags ($B$) has diminishing returns, depending on the magnitude of $\rho_{m}$. In case of relatively large $\rho_m$, the $\frac{1-\rho_{m}}{B}$ term accounts for a decreasingly lower portion of the variance, thence $B$ has proportionally less impact, while the increase in runtime remains linear. Notably, according to \eqref{eqn:131.1}, $\rho_{m}\cdot\mathrm{Var}(\hat{\theta}_m)$ could be expected to decrease when decreasing the sampling rate ($r$). Depending on the data and estimator, this gives reason to expect a larger improvement from increasing $B$, for lower values of $r$.

For each dataset and estimator, the exact behavior of the bounds in \eqref{eqn:131} and \eqref{eqn:131.1} can be different, therefore, experimental analyses are required to show the effectiveness of bagging in specific application scenarios, such as LID estimation in this paper.

\subsection{Additional experimental setups}\label{Asec:Additional experimental setups}

This section introduces two additional experimental setups dealing with the role of the number of bags hyper-parameter and its interaction with the sampling rate, to confirm the theoretical predictions of Theorem \ref{th1} and Theorem \ref{th2} in the context of bagged LID estimation. This serves to complete the bagging hyper-parameter selection experiments in Section \ref{sec:Experimental methodology} of the main paper, which focused on the particular role of the sampling rate and its interaction with the $k$ nearest neighbor ($k$-NN) locality hyper-parameter of certain baseline estimators (e.g., MLE).

As explained in the main paper, these experiments are specifically about bagging, designed to explore how exactly the choice of its hyper-parameters affects performance in terms of MSE, variance, and bias, when applied in the context of LID estimation. We show and discuss these effects on the bagged MLE estimator; however, the results for MADA and TLE exhibit similar general trends and can be found in Section \ref{Asec:Supplementary results}. They have been sidelined for the sake of clarity and compactness, since the main conclusions do not change. The results of these experiments are presented and analyzed in Section \ref{sec:results} and serve as the experimental basis for our proposed hyper-parameter selection guidance with regards to the number of bags hyper-parameter as presented in the main paper.

\paragraph{3rd test (number of bags):} \label{number of bags}

As shown in Theorem \ref{th1}, increasing the number of bags ($B$) is most effective when estimates from different bags are low-correlated, which in turn tends to associate with lower sampling rates, as seen in Theorem \ref{th2}. Therefore, for the current experiment, we fix the sampling rate at a relatively small value, $r = 0.05$, allowing for $B$ to range from $3$ to $400$ according to a $20$ step geometric progression.\footnote{We also include the Baseline case as the first experiment.} The progression enables us to see the effect of the hyper-parameter according to proportional increases, while exploring a wide range of values. We mainly expect to observe the diminishing returns of increasing $B$ as predicted by Theorem \ref{th1} and discussed in Section \ref{Asec:bagging theory}.

The results of this experiment are presented and analyzed in Section \ref{results:Number of bags test}

\paragraph{4th test (interaction between $B$ and $r$):}\label{B and sr interaction}

These experiments are meant to test the interplay between the two hyper-parameters $B$ and $r$, combining observations from the 1st sampling-rate test, presented in the main paper, and the previously discussed number of bags test, to show in practice the interacting effects we had already anticipated and discussed previously in Section \ref{Asec:bagging theory}. They are also intended to test the robustness of bagging in terms of a sustained performance across wide ranges of hyper-parameter choices and combinations.

We calculate bagged estimators for every combination of hyper-parameter values, between the $20$ values of sampling rates used in the previous experiments in this section, and a more comprehensive progression for the number of bags, between $1$ and $100$, both following roughly the same geometric rate.

The results of this experiment are presented and analyzed in Section \ref{result:Interaction of sr and B}.

\subsection{Supplementary result analysis}\label{Asec:results}

\subsubsection{Number of bags (3rd test)}\label{results:Number of bags test}

\begin{figure}[H]
    \centering
\includegraphics[width=0.9\textwidth,height=0.9\textheight,keepaspectratio]{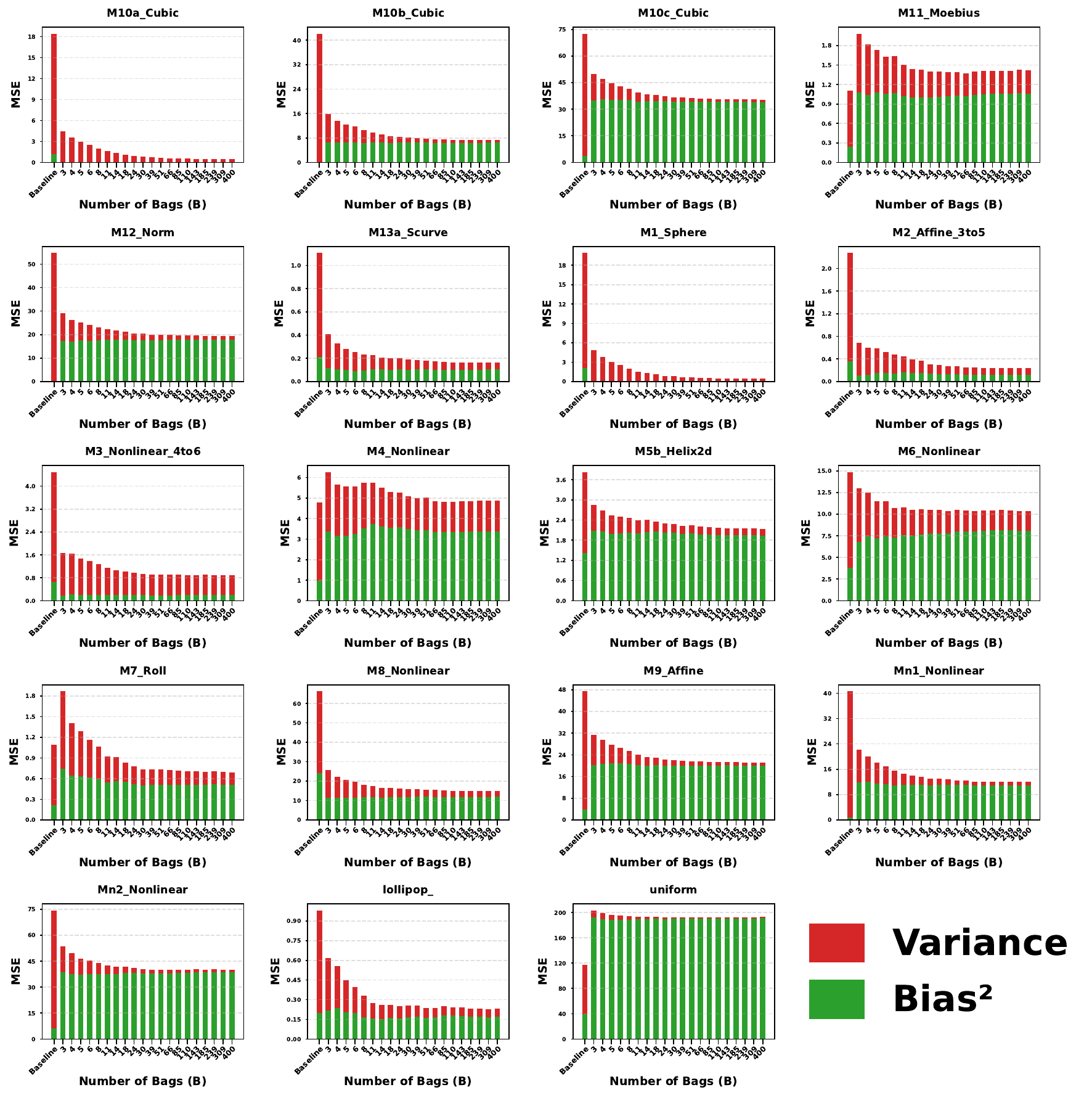}
\caption{MSE and its decomposition for each of the 19 datasets, as a function of the number of bags $B$ used for bagging of MLE as the baseline LID estimator (displayed on the leftmost bar of the individual charts). See Section \ref{number of bags} for the detailed experimental setup.}
\label{fig:Number of bags}
\end{figure}

Figure \ref{fig:Number of bags} illustrates the experimental behavior of the MSE decomposition for bagging as a function of the number of bags $B$. The experiment confirms our expectation that increasing $B$: (i) does not systematically affect bias; and (ii) it reduces variance according to the general trend supported by Theorem \ref{th1}, that is, at a diminishing rate. Regarding the latter, notice that despite the geometric progression of $B$ in the x-axis, the largest proportional drops in variance (red sub-bars) mostly happen at the beginning.

While bias$^2$ (gree sub-bars) is not systematically affected by $B$, in most cases we observe a difference between the bias for the baseline estimator (leftmost bar) and the rest (bagged estimators). This difference can be in favor of either the bagged or the baseline estimator, as in the current test we fix the $k$ and $r$ hyper-parameters, causing a dataset-dependent bias behavior. It is expected that certain hyper-parameter combinations will be better for the baseline than the bagged estimator or vice-versa, which we further analyze in detail in Section \ref{sec:results} of the main paper.

\subsubsection{Interaction between $B$ and $r$ (4th test):}\label{result:Interaction of sr and B}

\begin{figure}[H]
    \centering
\includegraphics[width=0.9\textwidth,height=0.9\textheight,keepaspectratio]{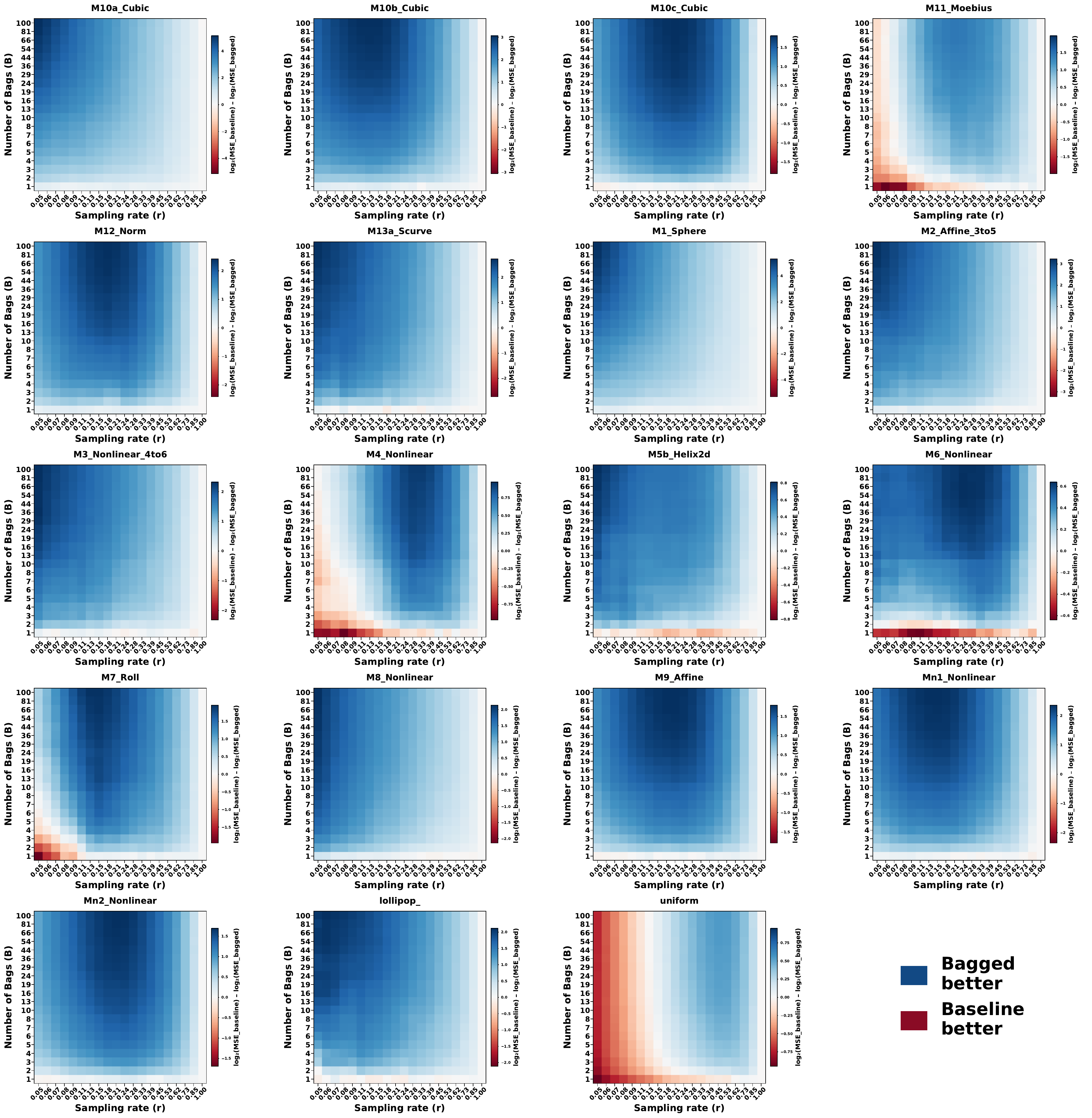}
\caption{Heatmaps of the relative difference (log-ratio) between the MSE achieved by the baseline estimator (MLE) and the MSE of its bagged counterpart for each of the 19 datasets and cross-combinations of values for the sampling rate $r$ (x-axis) and the number of bags $B$ (y-axis). Positive values, color-coded as blue, indicate that bagging outperforms the baseline, whereas negative values, color-coded in red, indicate the opposite. The relation is symmetric around zero and the darkness of the cells are proportional to the magnitude of the corresponding absolute value for the given dataset. See Section \ref{B and sr interaction} for the detailed experimental setup.}
\label{fig:Interaction of Bsr}
\end{figure}

Figure \ref{fig:Interaction of Bsr} illustrates the log-ratio between the MSE achieved by the baseline estimator (MLE) and the MSE of its bagged counterpart for each of the 19 datasets and cross-combinations of values for the sampling rate $r$ and the number of bags $B$. Notice that the baseline estimator does not depend on these hyper-parameters and is only used as a basis for comparison. Also, notice that the differences in logarithmic scale reflect relative changes, as opposed to absolute differences.

White-colored cells indicate that the MSEs of the compared estimators are the same. Therefore, the rightmost column in each heatmap, which corresponds to the $r=1$ case, is all white, since bagging with $r=1$ is equivalent to the baseline estimator. The darkness of the blue color corresponds to the proportional difference in magnitude in favor of the bagged estimator, while the darkness of the red color indicates the opposite.

The most eye-catching result is that most heatmaps are a gradient between white and dark blue, with little to no red to be seen, showing that bagging is very robust to these hyper-parameters, reducing MSE across wide ranges and combinations of their values. There is also a tendency to find darker blue areas towards the top left, which aligns with our expectations from Theorem \ref{th1} and Theorem \ref{th2}, as we find both increasing $B$ and decreasing $r$ generally results in variance reduction with a positive, reinforcing interaction.

It is important to note, however, that for certain datasets the darkest blue areas, representing the optimal hyper-parameters, are not situated at the top-left corner. This may be explained by the upper range of the $B$ hyper-parameter not being large enough to meet the requirements of the lowest sampling rates, as characterized in Theorem \ref{th1}. Another possible explanation is that, with a fixed value of $k$ in this experiment, the lowest sampling rates can result in bias due to the magnified increase in neighborhood radius around the query. This interplay between $k$ and the sampling rate is further explored in the main paper.

\subsection{Complexity analysis for the bagged LID estimator and its combinations with smoothing}\label{Asec:Complexity analysis for the bagged LID estimator and its combinations with smoothing}

Section \ref{sec:Computational complexity} in the main paper presented a complexity analysis for the proposed bagging method for LID estimators. However, in the experiments, there have been several different combinations of bagging and smoothing tested. The current subsection of the appendix extends that analysis with more detail and provides comprehensive tables for the time complexities of the smoothing variants, in general, or assuming a $k$-NN estimator baseline, as well as using a naive search for finding $k$-NNs or an amortized search following preconstructed indices.

When using bagging to obtain LID estimates for a dataset of $n$ data points, which are at the same time the query locations, as visualized in Figure \ref{fig:simplebagging}, we first have to sample $B$ bags with $n \cdot r$ points each. It takes $O(n \cdot r)$ time to generate $n \cdot r$ random indexes for each bag, therefore $O(B \cdot n \cdot r)$ time in total. Then, for each bag, a new LID estimate has to be calculated using its $n \cdot r$ points, for each of the $n$ original query locations. This means that if we assume LID estimation is carried out individually per query, with an estimator whose runtime is $O(p(n))$, as a function $p(n)$ of the sample size $n$ per query, then given that the bags contain $n \cdot r$ points, we get a time complexity of $O(B \cdot p(n \cdot r))$ per query, i.e., $O(B \cdot n \cdot p(n \cdot r))$ in total. Including the final averaging of each of the $n$ LID estimates over the $B$ bags, in $O(B \cdot n)$ time, the time complexity of the whole process can be expressed as $O(B \cdot n \cdot r) + O(B \cdot n \cdot p(n \cdot r)) + O(B \cdot n)$, and recalling that $r \in (0,1)$, it then follows that \emph{bagging} has a total runtime of $O(B \cdot n + B \cdot n \cdot p(n\cdot r) )$. In contrast, it is clear that the time complexity of the \emph{baseline estimator} is simply $O(n \cdot p(n))$.

For the wide family of $k$-NN LID estimators, finding the $k$ nearest neighbors of each query is usually the computational bottleneck, which means the time it takes to estimate LID for a single query within a sample of size $n$, $O(p(n))$, is essentially determined by the time it takes to find its $k$-NN. Using a naive exhaustive search, we have $O(p(n)) \rightarrow O(n)$ per query assuming $k << n$, whereas the use of an index would allow, under the same assumption, an amortized search per query in $O(p(n)) \rightarrow O(\log n)$ time (on average, already considering the construction of the index).

For naive search, the total runtime complexity for bagging then becomes $O(B \cdot n + B \cdot r \cdot n^2) \rightarrow O(n^2)$, and the simple estimator also becomes $O(n \cdot n) \rightarrow O(n^2)$. Either can be faster or slower depending on a combination of $n$, $B$, and $r$. In short, for large enough $n$, bagging can be faster as long as $r < 1/B$.

With an index, the total runtime complexity for bagging becomes $O(B \cdot n + B \cdot n \cdot \log(n \cdot r)) \rightarrow O(n \cdot \log n)$, and the baseline estimator also becomes $O(n \cdot \log n)$. In this case, however, the baseline tends to be faster in practical terms.

Note that in our experiments, as presented in the main paper, when we applied smoothing, we obtained smoothed LID estimates for a query by taking the arithmetic average of the estimates over its $k$-NN (the same neighborhood size used by the baseline estimator). However, in this more general complexity analysis, we are not going to make this assumption, as smoothing can easily be applied\footnote{Without making any claims about qualitative performance.} using any other neighborhood sample size up to $n$, denoted in the following by the integer hyper-parameter $k_s \in \{1, \dots, n\}$.

As a result, the time-complexity of estimating at a single query is multiplied by the number of neighbors ($k_s$) we use to smooth over. Additionally, we have to take into account the time of finding the $k_s$-NNs of the query, which is not generally given in the case of an arbitrary estimator baseline, and is free only when assuming $k \geq k_s$ when considering $k$-NN estimator baselines. Additionally, for bagging with pre-smoothing, this search has to be performed per-bag, instead of the whole dataset, much like the way $k_s$-NN estimators are applied in bagging.

In the other usual case, when estimating at all of the sample points as query locations, the estimates to average may only be calculated once for any of the variants, as the bags can be kept fixed throughout, and smoothing can reuse the precomputed estimates around each query. This results in effectively the same runtimes as for bagging, with only an additional term for finding $k_s$-NNs.

The exact resulting time-complexities for the general and the $k$-NN baseline specific settings are listed in Tables \ref{tab:Table of runtime complexities in general. (Naive Search)}, \ref{tab:Table of runtime complexities in general. (Amortized search, Neighbor indices)}, and Tables \ref{tab:Table of runtime complexities for $k$-NN estimation methods. (Naive Search)}, \ref{tab:Table of runtime complexities for $k$-NN estimation methods. (Amortized search, Neighbor indices)} respectively.

\begin{table}[H]
    \centering
    \caption{Runtime complexities in general. (Naive Search)} \vspace*{2mm}
    \begin{tabular}{|p{5cm}|p{5cm}|p{5cm}|}
    \hline
    Method & Runtime Complexity\newline At a single query & Runtime Complexity\newline For a dataset as queries \\
    \hline
    Baseline & $O(p(n))$ & $O(n  p(n))$ \\
    Bagging & $O(B  p(n r))$ & $O(B  n  p(n r))$\\
    Baseline with Smoothing & $O(n + k_s p(n))$ & $O(n^2 + n p(n))$ \\ 
    Bagging with post-smoothing & $O(n + B  k_s  p(n r))$ & $O(n^2 + B  n  p(n r))$ \\
    Bagging with pre-smoothing & $O(B r n  + B  k_s  p(n r))$ & $O(B r n^2  + B  n  p(n r))$ \\
    Bagging with pre-smoothing and post-smoothing & $O(n + B k_s n r + B  k_s^2  p(n r))$ & $O(n^2 + B  n^2 r + B  n  p(n r))$ \\
    \hline
    \end{tabular}
    \label{tab:Table of runtime complexities in general. (Naive Search)}
\end{table}

\begin{table}[H]
    \centering
    \caption{Runtime complexities in general. (Amortized search, Neighbor indices)} \vspace*{2mm}
    \begin{tabular}{|p{5cm}|p{5cm}|p{5cm}|}
    \hline
    Method & Runtime Complexity\newline At a single query & Runtime Complexity\newline For a dataset as queries \\
    \hline
    Baseline & $O(p(n))$ & $O(n  p(n))$ \\
    Bagging & $O(B  p(n r))$ & $O(B  n  p(n r))$\\
    Baseline with Smoothing & $O(\log(n) + k_s p(n))$ & $O(n\log(n) + n p(n))$ \\ 
    Bagging with post-smoothing & $O(\log(n) + B  k_s  p(n r))$ & $O(n\log(n) + B  n  p(n r))$ \\
    Bagging with pre-smoothing & $O(B  \log(n r) + B  k_s  p(n r))$ & $O(B  n\log(n r) + B  n  p(n r))$ \\
    Bagging with pre-smoothing and post-smoothing & $O(\log(n) + B k_s\log(n r) + B  k_s^2  p(n r))$ & $O(n\log(n) + B  n \log(n r) + B  n  p(n r))$ \\
    \hline
    \end{tabular}
    \label{tab:Table of runtime complexities in general. (Amortized search, Neighbor indices)}
\end{table}

\begin{table}[H]
    \centering
    \caption{Runtime complexities for $k$-NN estimation methods. (Naive Search)} \vspace*{2mm}
    \begin{tabular}{|p{5cm}|p{5cm}|p{5cm}|}
    \hline
    Method & Runtime Complexity\newline At a single query & Runtime Complexity\newline For a dataset as queries \\
    \hline
    Baseline & $O(n)$ & $O(n^2)$ \\
    Bagging & $O(B  r  n)$ & $O(B  r n^2)$\\
    Baseline with Smoothing & $O(n + k_s n)$ & $O(n^2)$ \\ 
    Bagging with post-smoothing & $O(n + B  r  k_s  n)$ & $O(n^2 + B  r n^2)$ \\
    Bagging with pre-smoothing & $O(B  r n + B  r  k_s  n)$ & $O(B r n^2)$ \\
    Bagging with pre-smoothing and post-smoothing & $O(n + B r k_s^2  n )$ & $O(n^2 + B r n^2)$ \\
    \hline
    \end{tabular}
    \label{tab:Table of runtime complexities for $k$-NN estimation methods. (Naive Search)}
\end{table}

\begin{table}[H]
    \centering
    \caption{Runtime complexities for $k$-NN estimation methods. (Amortized search, Neighbor indices)} \vspace*{2mm}
    \begin{tabular}{|p{5cm}|p{5cm}|p{5cm}|}
    \hline
    Method & Runtime Complexity\newline At a single query & Runtime Complexity\newline For a dataset as queries \\
    \hline
    Baseline & $O(\log(n))$ & $O(n  \log(n))$ \\
    Bagging & $O(B  \log(n r))$ & $O(B  n  \log(n r))$\\
    Baseline with Smoothing & $O(k_s \log(n))$ & $O(n\log(n))$ \\ 
    Bagging with post-smoothing & $O(\log(n) + B  k_s  \log(n r))$ & $O(n\log(n) + B  n  \log(n r))$ \\
    Bagging with pre-smoothing & $O(B  \log(n r) + B  k_s  \log(n r))$ & $O(B  n  \log(n r))$ \\
    Bagging with pre-smoothing and post-smoothing & $O(\log(n) + B  k_s^2  \log(n r))$ & $O(n\log(n) + B  n  \log(n r))$ \\
    \hline
    \end{tabular}
    \label{tab:Table of runtime complexities for $k$-NN estimation methods. (Amortized search, Neighbor indices)}
\end{table}

\section{Proofs}\label{Asec:Proofs}

\subsection{Proof of Theorem A.1}\label{Asec:proof1}
\begin{proof}

We start by proving the middle equality. Where first, we note the following. 
It follows from the assumption that $X_1,\dots,X_n$ are i.i.d. (independent and identically distributed), and that $X_1,\dots,X_n$ and $\Pi^{(i)}$ are independent, that for all $i$, we have $\hat{\theta}_{m,i} \overset{d}{=} \hat{\theta}_{m}$.\footnote{We use the $\overset{d}{=}$ notation to signal equality in distribution between two random variables, that is a fairly common notation in statistics textbooks.}
\newline\newline
\textbf{To put it more rigorously: }
\newline\newline
For any Borel set $A \subseteq \mathbb{R}^{dim \times m}$, we have that, using the Law of Total Probability:
\begin{equation}
\begin{aligned}
\label{eqn:proof1101}
& P\left(\left(X_{\Pi^{(i)}_1},\dots,X_{\Pi^{(i)}_m}\right) \in A\right) = \\
& = \mathbb{E}\!\left[P\left(\left(X_{\Pi^{(i)}_1},\dots,X_{\Pi^{(i)}_m}\right) \in A \; | \; \Pi^{(i)}_1, \dots, \Pi^{(i)}_m\right)\right]\\
& = \sum_{\pi \in \text{supp} (\Pi^{(i)})}P((X_{\pi_1},\dots,X_{\pi_m}) \in A \;|\;\Pi^{(i)} = \pi) \cdot P(\Pi^{(i)} = \pi)\\
& = \sum_{\pi \in \text{supp} (\Pi^{(i)})}P((X_{\pi_1},\dots,X_{\pi_m}) \in A) \cdot P(\Pi^{(i)} = \pi)\\
& = \sum_{\pi \in \text{supp}(\Pi^{(i)})}P((X_{\pi_1},\dots,X_{\pi_m}) \in A) \cdot \frac{1}{|\text{supp} (\Pi^{(i)})|}\\
& = \sum_{\pi \in \text{supp}(\Pi^{(i)})}P((X_{1},\dots,X_{m}) \in A) \cdot \frac{1}{|\text{supp} (\Pi^{(i)})|} \\  
& = P((X_{1},\dots,X_{m}) \in A) \cdot\left(\sum_{\pi \in \text{supp}(\Pi^{(i)})} \frac{1}{|\text{supp} (\Pi^{(i)})|}\right) \\
& = P((X_{1},\dots,X_{m}) \in A) \cdot \left(\frac{\sum_{\pi \in \text{supp}(\Pi^{(i)})} 1}{|\text{supp} (\Pi^{(i)})|}\right) \\
& = P((X_{1},\dots,X_{m}) \in A) \cdot \frac{{n \choose m}}{{n \choose m}} \\
& = P((X_{1},\dots,X_{m}) \in A).
\end{aligned}
\end{equation}
Where, to get to line 4 we used the independence of $X_1,\dots,X_n$ and $\Pi^{(i)}$ to simplify the conditional probability, then to get to line 5 we used that $\Pi^{(i)}$ is uniformly distributed over the $n\choose m$ different index permutations. And to get to line 6 we used that we have 
\[(X_{\pi_1},\dots,X_{\pi_m}) \overset{d}{=} (X_{1},\dots,X_{m}),\] 
for any $\pi \in \text{supp}(\Pi^{(i)})$, because for any Borel set $A$, we have $P((X_{\pi_1},\dots,X_{\pi_m}) \in A) = P((X_{1},\dots,X_{m}) \in A)$, for any deterministic choice of indices $\pi$ by the i.i.d. assumption on $X_{1},\dots,X_{n}$. After that it is simple algebraic manipulations.

Therefore, Equation \ref{eqn:proof1101} shows that $(X_{\Pi_1},\dots,X_{\Pi_m}) \overset{d}{=} (X_{1},\dots,X_{m})$, by the definition of equality in distribution, which implies that for a $\hat{\theta}$ statistic, a measurable function of the sample, $\hat{\theta}(X_{\Pi_1},\dots,X_{\Pi_m}) \overset{d}{=} \hat{\theta}(X_{1},\dots,X_{m})$, and using the notation in Definition \eqref{eqn:130} this means $\hat{\theta}_{m,i} \overset{d}{=} \hat{\theta}_{m}$.

Given that $\hat{\theta}_{m,i} \overset{d}{=} \hat{\theta}_{m}$, we have $\text{Var}(\hat{\theta}_{m,i})=\text{Var}(\hat{\theta}_{m})$.

Furthermore, for any $i_1\neq j_1,\; i_2\neq j_2$, we have that $(\hat{\theta}_{m,i_1} , \hat{\theta}_{m,j_1}) \overset{d}{=} (\hat{\theta}_{m,i_2} , \hat{\theta}_{m,j_2})$. Which follows from the assumption that the random index sets $\Pi^{(i_1)}, \Pi^{(j_1)}$ and $\Pi^{(i_2)}, \Pi^{(j_2)}$ are independent and identically distributed respectively. 
\newline\newline
\textbf{To put it more rigorously: }
\newline\newline
For any Borel sets $A_1, A_2 \subseteq \mathbb{R}^{dim \times m}$, using the Law of Total Probability we have that
\begin{equation}
\begin{aligned}
\label{eqn:proof11}
&P\left(\left(X_{\Pi^{(i_1)}_1},\dots,X_{\Pi^{(i_1)}_m}\right) \in A_1 , \; \left(X_{\Pi^{(j_1)}_1},\dots,X_{\Pi^{(j_1)}_m}\right) \in A_2\right) =\\ 
& = \mathbb{E}\!\left[P\left(\left(X_{\Pi^{(i_1)}_1},\dots,X_{\Pi^{(i_1)}_m}\right) \in A_1 , \; \left(X_{\Pi^{(j_1)}_1},\dots,X_{\Pi^{(j_1)}_m}\right) \in A_2 \;|\;X_1,\dots, X_n\right)\right] \\
& = \int P\left(\left(X_{\Pi^{(i_1)}_1},\dots,X_{\Pi^{(i_1)}_m}\right) \in A_1 , \; \left(X_{\Pi^{(j_1)}_1},\dots,X_{\Pi^{(j_1)}_m}\right) \in A_2 \;| \; (X_1,\dots,X_n) = (x_1,\dots, x_n)\right) ~dP(X_1,\dots,X_n)\\
& = \int P\left(\left(x_{\Pi^{(i_1)}_1},\dots,x_{\Pi^{(i_1)}_m}\right) \in A_1 , \; \left(x_{\Pi^{(j_1)}_1},\dots,x_{\Pi^{(j_1)}_m}\right) \in A_2\right)~dP(X_1,\dots,X_n) \\
& = \int P\left(\left(x_{\Pi^{(i_1)}_1},\dots,x_{\Pi^{(i_1)}_m}\right) \in A_1\right) P\left(\left(x_{\Pi^{(j_1)}_1},\dots,x_{\Pi^{(j_1)}_m}\right) \in A_2\right)~dP(X_1,\dots,X_n) \\
& = \int P\left(\left(x_{\Pi^{(i_2)}_1},\dots,x_{\Pi^{(i_2)}_m}\right) \in A_1\right) P\left(\left(x_{\Pi^{(j_2)}_1},\dots,x_{\Pi^{(j_2)}_m}\right) \in A_2\right)~dP(X_1,\dots,X_n) \\ 
& = \int P\left(\left(x_{\Pi^{(i_2)}_1},\dots,x_{\Pi^{(i_2)}_m}\right) \in A_1 , \; \left(x_{\Pi^{(j_2)}_1},\dots,x_{\Pi^{(j_2)}_m}\right) \in A_2\right)~dP(X_1,\dots,X_n) \\
& = P\left(\left(X_{\Pi^{(i_2)}_1},\dots,X_{\Pi^{(i_2)}_m}\right) \in A_1 , \; \left(X_{\Pi^{(j_2)}_1},\dots,X_{\Pi^{(j_2)}_m}\right) \in A_2\right).
\end{aligned}
\end{equation}
Where to get to line 4 we used the independence of $X_1,\dots,X_n$ and $\Pi^{(i_1)}, \Pi^{(j_1)}$. Then, to get to line 5 we used the independence of $\Pi^{(i_1)}$ and $ \Pi^{(j_1)}$. Then, to get to line 6 we used that $\Pi^{(i_1)}$ and $\Pi^{(i_2)}$ and $ \Pi^{(j_1)}$ and $ \Pi^{(j_2)}$ are identically distributed. After that, the same equalities can be repeated backwards with the same reasoning, just using $\Pi^{(i_2)}, \Pi^{(j_2)}$ instead, to arrive at the desired result. \newline
Therefore, we have shown that
\begin{equation}
\begin{aligned}
\label{eqn:proof11}
\left(\left(X_{\Pi^{(i_1)}_1},\dots,X_{\Pi^{(i_1)}_m}\right), \left(X_{\Pi^{(j_1)}_1},\dots,X_{\Pi^{(j_1)}_m}\right)\right) \overset{d}{=}\left(\left(X_{\Pi^{(i_2)}_1},\dots,X_{\Pi^{(i_2)}_m}\right), \left(X_{\Pi^{(j_2)}_1},\dots,X_{\Pi^{(j_2)}_m}\right)\right)
\end{aligned}
\end{equation}
which implies that for the $\hat{\theta}$ statistic, a measurable function of the sample, we have
\begin{equation}
\begin{aligned}
\label{eqn:proof11}
\left(\hat{\theta}\left(X_{\Pi^{(i_1)}_1},\dots,X_{\Pi^{(i_1)}_m}\right), \hat{\theta}\left(X_{\Pi^{(j_1)}_1},\dots,X_{\Pi^{(j_1)}_m}\right)\right) \overset{d}{=}\left(\hat{\theta}\left(X_{\Pi^{(i_2)}_1},\dots,X_{\Pi^{(i_2)}_m}\right), \hat{\theta}\left(X_{\Pi^{(j_2)}_1},\dots,X_{\Pi^{(j_2)}_m}\right)\right) 
\end{aligned}
\end{equation}
with equivalent notation $(\hat{\theta}_{m,i_1} , \hat{\theta}_{m,j_1}) \overset{d}{=} (\hat{\theta}_{m,i_2} , \hat{\theta}_{m,j_2})$. Given that $(\hat{\theta}_{m,i_1} , \hat{\theta}_{m,j_1}) \overset{d}{=} (\hat{\theta}_{m,i_2} , \hat{\theta}_{m,j_2})$, we have $\text{cov}(\hat{\theta}_{m,i_1} , \hat{\theta}_{m,j_1}) = \text{cov}(\hat{\theta}_{m,i_2} , \hat{\theta}_{m,j_2})$.

As we have shown that the covariances $\text{cov}(\hat{\theta}_{m,i} , \hat{\theta}_{m,j})$ are always the same for $i \neq j$, we have that $\text{cov}(\hat{\theta}_{m,i} , \hat{\theta}_{m,j}) =: \gamma_m$ is a constant. We've already shown that $\text{Var}(\hat{\theta}_{m,i})=\text{Var}(\hat{\theta}_{m})$ for all $i$, so that the variances are constant as well. Putting these together, implies that we have 
\[\rho_m := \text{corr}(\hat{\theta}_{m,i} , \hat{\theta}_{m,j}) = \frac{\text{cov}(\hat{\theta}_{m,i} , \hat{\theta}_{m,j})}{\sqrt{\text{Var}(\hat{\theta}_{m,i})\text{Var}(\hat{\theta}_{m,j})}} = \frac{\gamma_m}{\sqrt{\text{Var}(\hat{\theta}_{m})\text{Var}(\hat{\theta}_{m})}} = \frac{\gamma_m}{\text{Var}(\hat{\theta}_{m})},\]
which provides us the $\gamma_m = \rho_m \text{Var}(\hat{\theta}_{m})$ relationship.
\newline\newline
\textbf{Now we can continue with the main part of the proof, showing the middle equality.}
\newline\newline
According to definition \eqref{eqn:130}, and the derived equalities of variances, we have
\begin{equation}
\begin{aligned}
\label{eqn:proof11}
& \text{Var}(\hat{\theta}_{B,m}) = \\
& = \text{Var}\left( \frac{1}{B} \sum_{i=1}^B \hat{\theta}_{m, i}\right) \\
& = \frac{1}{B^2} \sum_{i=1}^B \sum_{j=1}^B \text{cov}(\hat{\theta}_{m, i}, \hat{\theta}_{m, j}) \\ 
& = \frac{1}{B^2} \sum_{i=1}^B  \text{cov}(\hat{\theta}_{m, i}, \hat{\theta}_{m, i}) + \frac{1}{B^2} \sum_{i=1}^B \sum_{j\neq i}^B \text{cov}(\hat{\theta}_{m, i}, \hat{\theta}_{m, j}) \\
& = \frac{1}{B^2} \sum_{i=1}^B  \text{Var}(\hat{\theta}_{m,i}) + \frac{1}{B^2} \sum_{i=1}^B \sum_{j\neq i}^B \gamma_m \\
& = \frac{1}{B^2} \sum_{i=1}^B  \text{Var}(\hat{\theta}_{m}) + \frac{1}{B^2} \sum_{i=1}^B \sum_{j\neq i}^B \rho_m \text{Var}(\hat{\theta}_{m}) \\
& = \frac{1}{B}\text{Var}(\hat{\theta}_{m}) + \frac{B^2-B}{B^2} \rho_m \text{Var}(\hat{\theta}_{m}) \\
& =\text{Var}(\hat{\theta}_{m}) \left(\frac{1}{B} + \rho_m - \frac{\rho_m}{B}\right) \\
& = \text{Var}(\hat{\theta}_{m}) \left(\rho_m + \frac{1-\rho_m}{B}\right)
\end{aligned}
\end{equation}
This proves the middle equality, and from here we only need to know that the correlation is always such that $-1\leq\rho_m \leq 1$. Therefore, for $B > 1$ we have,
\begin{equation}
\begin{aligned}
\label{eqn:proof11}
& \rho_m \leq 1 \\
& \rho_m + \frac{1-\rho_m}{B} = \frac{\rho_m(B-1)+1}{B} \leq \frac{B-1+1}{B} = 1 \\
& \rho_m + \frac{1-\rho_m}{B} \leq  1 \\
& 0 \leq  \text{Var}(\hat{\theta}_{m}) \\
& \text{Var}(\hat{\theta}_{m}) \left(\rho_m + \frac{1-\rho_m}{B}\right)  \leq \text{Var}(\hat{\theta}_{m}) \\
\end{aligned}
\end{equation}
as the variance is always non-negative, and
\begin{equation}
\begin{aligned}
\label{eqn:proof11}
& \rho_m \leq 1 \\
& 0 \leq \frac{1-\rho_m}{B}\\
& \rho_m \leq \rho_m + \frac{1-\rho_m}{B}\\
& 0 \leq  \text{Var}(\hat{\theta}_{m}) \\
& \rho_m\text{Var}(\hat{\theta}_{m}) \leq \text{Var}(\hat{\theta}_{m}) \left(\rho_m + \frac{1-\rho_m}{B}\right) \\
& \text{cov}(\hat{\theta}_{m, i}, \hat{\theta}_{m, j}) = \gamma_m = \rho_m\text{Var}(\hat{\theta}_{m}) \leq \text{Var}(\hat{\theta}_{m}) \left(\rho_m + \frac{1-\rho_m}{B}\right) \\
& \text{cov}(\hat{\theta}_{m, i}, \hat{\theta}_{m, j}) \leq \text{Var}(\hat{\theta}_{m}) \left(\rho_m + \frac{1-\rho_m}{B}\right)
\end{aligned}
\end{equation}
It's also clear that
\begin{equation}
\begin{aligned}
\label{eqn:proof11}
\lim_{B \rightarrow \infty} \text{Var}(\hat{\theta}_{m}) \left(\rho_m + \frac{1-\rho_m}{B}\right) = \rho_m\text{Var}(\hat{\theta}_{m}) = \gamma_m =  \text{cov}(\hat{\theta}_{m, i}, \hat{\theta}_{m, j})
\end{aligned}
\end{equation}
as $\text{Var}(\hat{\theta}_{m})$ and $\rho_m$ are not functions of $B$.
\end{proof}
\subsection{Proof of Theorem 1}\label{Asec:proof2}
\begin{proof}

Let's denote:
\begin{equation}
\begin{aligned}
\label{eqn:proof11}
\gamma(m) := \text{cov}(\hat{\theta}_{m, i}, \hat{\theta}_{m, j} \;|\; |\Pi^{(i)}\cap \Pi^{(j)}|),
\end{aligned}
\end{equation}
where we have that $\gamma(m) = f_m(|\Pi^{(i)}\cap \Pi^{(j)}|)$ is a random variable that is a function of the random variable $|\Pi^{(i)}\cap \Pi^{(j)}|$, for some $f_m$ measurable function. With this notation we may write $\gamma(h, m) = f_m(h)$ and $\gamma(H, m) = f_m(H)$ for a random variable $H$, where $H \overset{d}{=}|\Pi^{(i)}\cap \Pi^{(j)}|$ implies $\gamma(H, m) = f_m(H) \overset{d}{=} f_m(|\Pi^{(i)}\cap \Pi^{(j)}|)= \gamma(m)$. Meaning that for any $H \overset{d}{=}|\Pi^{(i)}\cap \Pi^{(j)}|$, we have that $\gamma(H, m)\overset{d}{=}\gamma(m)$, along the definition of $\gamma(h, m)$ in the theorem statement. \footnote{We use the $\overset{d}{=}$ notation to signal equality in distribution between two random variables, that is a fairly common notation in statistics textbooks.}
\newline\newline
\textbf{First we show,}
\begin{enumerate}
    \item $|\Pi^{(i)}\cap\Pi^{(j)}|\sim \text{Hypergeometric}(n, m, m)$,
    \item $|\Pi^{(i)}\cap\Pi^{(j)}| \perp \Pi^{(i)}$,
    \item $|\Pi^{(i)}\cap\Pi^{(j)}| \perp \Pi^{(j)}$,
\end{enumerate}
which we will make use of in the later parts of the proof.\footnote{We use the $\perp$ notation for independence between random variables, and the $\sim$ notation to point out the distribution of a random variable.}
\newline\newline
\textbf{1.}
\newline\newline
Observe that for a given index set $\pi^{(i)}$ we have that $|\pi^{(i)}\cap\Pi^{(j)}| \sim \text{Hypergeometric}(n, m, m)$ distributed. As for any $h \in \{0,1,\dots,m\}$, we can select, without permutations, in $m\choose h$ different ways, the $h$ elements from $\pi^{(i)}$ to be matched by $\Pi^{(j)}$. Then we can select out of the remaining $n-h$ elements those $m-h$ which are not matching in ${n-m}\choose{m-h}$ different ways, giving us that $P(|\pi^{(i)}\cap\Pi^{(j)}|=h) = \frac{\binom{m}{h} \binom{n - m}{m - h}}{\binom{n}{m}}$ for $h = 0,1,\dots,m$, which is the probability mass function of the $\text{Hypergeometric}(n, m, m)$ distribution.
From this, we can easily deduce using the Law  of Total Probability, that,
\begin{equation}
\begin{aligned}
\label{eqn:proof110001}
& P(|\Pi^{(i)} \cap \Pi^{(j)}| = h) = \\
& = \mathbb{E}\!\left[ P(|\Pi^{(i)} \cap \Pi^{(j)}| = h \mid \Pi^{(i)}) \right] \\
& = \sum_{\pi^{(i)} \in \text{supp}(\Pi^{(i)})} P(|\pi^{(i)} \cap \Pi^{(j)}| = h \;|\; \Pi^{(i)} = \pi^{(i)})\cdot P(\Pi^{(i)} = \pi^{(i)}) \\
&= \sum_{\pi^{(i)} \in \text{supp}(\Pi^{(i)})} P(|\pi^{(i)} \cap \Pi^{(j)}| = h)\cdot \frac{1}{\binom{n}{m}} \\
&= \sum_{\pi^{(i)} \in \text{supp}(\Pi^{(i)})} \frac{\binom{m}{h} \binom{n - m}{m - h}}{\binom{n}{m}}\cdot\frac{1}{\binom{n}{m}} \\
&= \frac{\binom{m}{h} \binom{n - m}{m - h}}{\binom{n}{m}^2}\cdot \left(\sum_{\pi^{(i)} \in \text{supp}(\Pi^{(i)})}1\right) \\
& =  \frac{\binom{m}{h} \binom{n - m}{m - h}}{\binom{n}{m}^2}\cdot \binom{n}{m} \\
& = \frac{\binom{m}{h} \binom{n - m}{m - h}}{\binom{n}{m}}.
\end{aligned}
\end{equation}
Where the argument goes similar to \ref{eqn:proof1101} in the previous proof. To get to line 4 we used the independence of $\Pi^{(i)}$ and $\Pi^{(j)}$, then we use the uniform distributional assumption for $\Pi^{(i)}$, and then we substitute the probability mass function of the Hypergeometric($n$, $m$, $m$) distribution, which we showed before. Completed by simple algebraic manipulations.

Equation \ref{eqn:proof110001} shows that $|\Pi^{(i)}\cap\Pi^{(j)}|$ is also $\text{Hypergeometric}(n, m, m)$ distributed, according its probability mass function.

\textbf{2./3.}
\newline\newline
This also means, that because $\Pi^{(i)}$ and $\Pi^{(j)}$ are independent, we have for any $\pi^{(i)} \in \text{supp}(\Pi^{(i)})$, and for any $h\in \text{supp}\left(|\Pi^{(i)} \cap \Pi^{(j)}|\right)$,
\begin{equation}
\begin{aligned}
\label{eqn:proof11}
P(|\Pi^{(i)} \cap \Pi^{(j)}| = h \mid \Pi^{(i)} = \pi^{(i)}) 
&= \frac{P\left(|\pi^{(i)} \cap \Pi^{(j)}| = h \;\cap\; \Pi^{(i)} = \pi^{(i)}\right)}{P(\Pi^{(i)} = \pi^{(i)})} \\
&= \frac{P(|\pi^{(i)} \cap \Pi^{(j)}| = h) \cdot P(\Pi^{(i)} = \pi^{(i)})}{P(\Pi^{(i)} = \pi^{(i)})} \\
&= P(|\pi^{(i)} \cap \Pi^{(j)}| = h) \\
&= \frac{\binom{m}{h} \binom{n - m}{m - h}}{\binom{n}{m}} \\
&= P(|\Pi^{(i)} \cap \Pi^{(j)}| = h)
\end{aligned}
\end{equation}
and therefore we have shown that $|\Pi^{(i)}\cap\Pi^{(j)}|$ and $\Pi^{(i)}$ are independent. We can make a similar argument to show that $|\Pi^{(i)}\cap\Pi^{(j)}|$ and $\Pi^{(j)}$ are independent, as $i$ and $j$ are interchangeable.
\newline \newline
\textbf{Now, onto the main proof, first we show that,}
\newline \newline
\begin{equation}
\begin{aligned}
\label{eqn:proof11}
\text{cov}(\hat{\theta}_{m, i}, \hat{\theta}_{m, j}) = \mathbb{E}[\gamma(m)].
\end{aligned}
\end{equation}
Apply the Law of Total Covariance to get
\begin{equation}
\begin{aligned}
\label{eqn:proof11totalcov}
& \text{cov}(\hat{\theta}_{m, i}, \hat{\theta}_{m, j}) = \mathbb{E}[\text{cov}(\hat{\theta}_{m, i}, \hat{\theta}_{m, j}\;|\;|\Pi^{(i)}\cap \Pi^{(j)}|)] + \text{cov}(\mathbb{E}[\hat{\theta}_{m, i}\;|\;|\Pi^{(i)}\cap \Pi^{(j)}|],\; \mathbb{E}[\hat{\theta}_{m, j}\;|\;|\Pi^{(i)}\cap \Pi^{(j)}|]) = \\
& = \mathbb{E}[\gamma(m)] + \text{cov}(\mathbb{E}[\hat{\theta}_{m, i}\;|\;|\Pi^{(i)}\cap \Pi^{(j)}|],\; \mathbb{E}[\hat{\theta}_{m, j}\;|\;|\Pi^{(i)}\cap \Pi^{(j)}|])
\end{aligned}
\end{equation}
As we have that $\Pi^{(i)}$ and $|\Pi^{(i)}\cap \Pi^{(j)}|$ are independent, and $X_1,...,X_n$ and $(\Pi^{(i)}, \Pi^{(j)})$ are independent, $X_1,...,X_n$ and $|\Pi^{(i)}\cap \Pi^{(j)}|$ are also independent, and therefore $(\Pi^{(i)},\;X_1,...,X_n)$ and $|\Pi^{(i)}\cap \Pi^{(j)}|$ are independent, implying that for the $\hat{\theta}_{m, i}$ measurable function of $(\Pi^{(i)},\;X_1,...,X_n)$, we have that $\hat{\theta}_{m, i}$ and $|\Pi^{(i)}\cap \Pi^{(j)}|$ are independent. A similar argument can be made to show that $\hat{\theta}_{m, j}$ and $|\Pi^{(i)}\cap \Pi^{(j)}|$ are independent. Therefore,
\begin{equation}
\begin{aligned}
\label{eqn:proof11}
& \mathbb{E}[\hat{\theta}_{m, i}\;|\;|\Pi^{(i)}\cap \Pi^{(j)}|] = \mathbb{E}[\hat{\theta}_{m, i}] \\
& \mathbb{E}[\hat{\theta}_{m, j}\;|\;|\Pi^{(i)}\cap \Pi^{(j)}|] = \mathbb{E}[\hat{\theta}_{m, j}] \\
&\text{cov}(\mathbb{E}[\hat{\theta}_{m, i}\;|\;|\Pi^{(i)}\cap \Pi^{(j)}|],\; \mathbb{E}[\hat{\theta}_{m, j}\;|\;|\Pi^{(i)}\cap \Pi^{(j)}|]) = \text{cov}(\mathbb{E}[\hat{\theta}_{m, i}],\; \mathbb{E}[\hat{\theta}_{m, j}]) = 0.
\end{aligned}
\end{equation}
As the covariance of two constants $\mathbb{E}[\hat{\theta}_{m, i}]$ and $\mathbb{E}[\hat{\theta}_{m, j}]$ is $0$. So, going back to \eqref{eqn:proof11totalcov} we have shown that 
\begin{equation}
\begin{aligned}
\label{eqn:proof12}
& \text{cov}(\hat{\theta}_{m, i}, \hat{\theta}_{m, j}) = \mathbb{E}[\gamma(m)].
\end{aligned}
\end{equation}
\newline
\textbf{Now we can complete the proof.}
\newline \newline
Using the assumption that $\gamma(h,m) \leq \varphi(\frac{h}{m})$, and that $\gamma(H,m) \overset{d}{=} \gamma(m)$ for $H \sim \text{Hypergeometric}(n, m, m)$, we have the bound
\begin{equation}
\begin{aligned}
\label{eqn:proof13}
\mathbb{E}[\gamma(m)] = \mathbb{E}[\gamma(H,m)] \leq \mathbb{E}\!\left[\varphi\left(\frac{H}{m}\right)\right],
\end{aligned}
\end{equation}
now, using our assumptions on the properties of $\varphi$, for the right side we can apply the generalized Jensen Bound \cite{JensenBound}, to arrive at
\begin{equation}
\begin{aligned}
\label{eqn:proof14}
& \text{cov}(\hat{\theta}_{m, i}, \hat{\theta}_{m, j}) = \mathbb{E}[\gamma(m)] = \mathbb{E}\!\left[\varphi\left(\frac{H}{m}\right)\right] \leq\\
& \leq \varphi\left(\mathbb{E}\!\left[\frac{H}{m}\right]\right) + \text{Var}\left(\frac{H}{m}\right) \sup_{x \in [0,1]} \frac{\varphi''(x)}{2} = \\
& = \varphi\left(\frac{\mathbb{E}\!\left[H\right]}{m}\right) + \frac{\text{Var}\left(H\right)}{m^2} \sup_{x \in [0,1]} \frac{\varphi''(x)}{2}\\
& = \varphi\left(r\right) + \frac{n\cdot r \cdot \frac{n\cdot r}{n}\cdot\frac{n-n\cdot r}{n}\cdot\frac{n-n\cdot r}{n-1}}{n^2 \cdot r^2} \sup_{x \in [0,1]} \frac{\varphi''(x)}{2}\\
& = \varphi\left(r\right) + \frac{(1-r)^2}{n-1} \sup_{x \in [0,1]} \frac{\varphi''(x)}{2} \\
& = \varphi\left(r\right) + (1-r)^2 O\left(\frac{1}{n}\right) \leq \\
& \leq\varphi\left(r\right) + O\left(\frac{1}{n}\right)
\end{aligned}
\end{equation}
where in line 4 we substitute in the known theoretical mean and variance of the $\text{Hypergeometric}(n, m, m)$ distribution \cite{wikipedia:hypergeometric_distribution}.
\end{proof}

\textbf{Variants of Theorem \ref{th2}:}
Notice that our main assumption in Theorem \ref{th2} (there exists a $\varphi:[0,1] \rightarrow \mathbb{R}$ such that $\forall_{h, m\;: \;h \leq m } \; \gamma(h,m) \leq \varphi(\frac{h}{m})$) requires the existence of a single such function for all values of $h$ and $m$, and therefore all values of $n$ as well. This allowed us to merge the $\sup_{x \in [0,1]} \frac{\varphi''(x)}{2}$ constant term into the asymptotic expression ($O(\frac{1}{n})$). In the main part of the thesis, we argued that the assumption is natural because the $\gamma(h,m)$ covariance is expected to increase with a larger $\frac{h}{m}$ ratio. While this is a reasonable idea, the assumption has to hold true for the same function for an infinite number of covariances. Whereas, given any fixed $n$-value, assuming the existence of a $\varphi_n:[0,1] \rightarrow \mathbb{R}$ such that $\forall_{h,m\;: \;h \leq m,\; m\leq n} \; \gamma(h,m) \leq \varphi_n(\frac{h}{m})$ may be an easier condition to satisfy. This assumption, however, on its own, is not enough to obtain the same asymptotic relationship, as we arrive at the
\[
\text{cov}(\hat{\theta}_{m, i}, \hat{\theta}_{m, j}) \leq \varphi_n\left(r\right) + \frac{(1-r)^2}{n-1} \sup_{x \in [0,1]} \frac{\varphi_n''(x)}{2}
\]
inequality, following the same reasoning up until this point as in the proof of Theorem \ref{th2}. Here the $\sup_{x \in [0,1]} \frac{\varphi_n''(x)}{2}$ term remains dependent on $n$. An avenue to obtain a similar asymptotic term can be found by additionally assuming $\sup_{n\in\mathbb{Z}^+}\sup_{x \in [0,1]}\varphi''_n(x) < \infty$, and therefore we force the $\sup_{x \in [0,1]} \frac{\varphi_n''(x)}{2} \leq \sup_{m\in\mathbb{Z}^+}\sup_{x \in [0,1]}\varphi''_n(x)$ constant upper bound, allowing us to once again merge this constant term into the asymptotic expression, and arrive at the
\[
\text{cov}(\hat{\theta}_{m, i}, \hat{\theta}_{m, j}) \leq \varphi_n\left(r\right) + O\left(\frac{1}{n}\right)
\]
more specific inequality, for a specific $n$, in terms of $\varphi_n$. Therefore, we can state the following variant of Theorem \ref{th2}:

\begin{theorem}\label{th4}
Define the two-variable deterministic function $\gamma(h,m) := \text{cov}(\hat{\theta}_{m,i}, \hat{\theta}_{m,j} \;|\;|\Pi^{(i)}\cap\Pi^{(j)}|=h)$ for the integer domain given by $m = r \cdot n \in \mathbb{Z^+}$ and $h=0,1,\dots,m$, namely, the covariance of two single bag estimators given that we know the number of dependent variables, provided that such a covariance exists and is finite for any pair of bags. Assume that there exists for each $n$, a twice differentiable, increasing function $\varphi_n:[0,1] \rightarrow \mathbb{R}$ such that $\forall_{h\;: \;h \leq m,\; m\leq n} \; \gamma(h,m) \leq \varphi_n(\frac{h}{m})$ and $\sup_{n\in\mathbb{Z}^+}\sup_{x \in [0,1]}\varphi''_n(x) < \infty$, then we have for every $n$ and $m = r \cdot n \in \mathbb{Z^+}$:
\begin{equation}
\begin{aligned}
\label{eqn:131.2}
\forall_{r \in [0,1]} \;\;\text{cov}(\hat{\theta}_{m,i}, \hat{\theta}_{m,j})\leq \varphi_n(r) + O\left(\frac{1}{n}\right)
\end{aligned}
\end{equation}
\end{theorem}
Of course, this is only required if we want to be able to retain the same order, namely $O\left(\frac{1}{n}\right)$. We can even further generalize our assumptions as $\sup_{x \in [0,1]}\varphi''_n(x) = O(n^t)$ with a fixed $t\in[0,1)$, in which case we still arrive at a vanishing asymptotic term $O(n^{t-1})$.

\section{Pointwise MSE decomposition for submanifolds}\label{Asec:MSE appendix}
The following is a simple generalization of the common MSE decomposition formula, for the case where queries are allowed to have different ground truth LID.

Given an LID estimator $\widehat{LID}$ and a union of manifolds $D = D_{1}\cup D_2\cup\dots\cup D_{L},\; D_{i} = \{X_{i1},\dots,X_{i|D_{i}|}\}$, with known, constant ground truth LIDs per manifold $LID(D_1),\dots,LID(D_L)$, we calculate the total pointwise mean squared error and its decomposition as follows,
\begin{equation}
\begin{aligned}
\label{eqn:25}
& \overline{LID}(D_i) := \frac{1}{|D_i|}\sum_{j = 1}^{|D_i|} \widehat{LID}(D_i;X_{ij}), 
\qquad
n := \sum_{i=1}^{L} |D_i|, \\
& MSE := \frac{1}{n}\sum_{i = 1}^{L}\sum_{j = 1}^{|D_i|} \left(\widehat{LID}(D_i;X_{ij}) - LID(D_i)\right)^2 \\
& \phantom{MSE} = \sum_{i = 1}^{L} \frac{|D_i|}{n} \frac{1}{|D_i|}\sum_{j = 1}^{|D_i|} \left(\widehat{LID}(D_i;X_{ij}) - LID(D_i)\right)^2\\
& \phantom{MSE} = \sum_{i = 1}^{L}\frac{|D_i|}{n}MSE_{D_i} \\
& \phantom{MSE} = \sum_{i = 1}^{L}\frac{|D_i|}{n}\left[VAR_{D_i} + BIAS_{D_i}^2  \right] \\
& \phantom{MSE} = \sum_{i = 1}^{L}\frac{|D_i|}{n}VAR_{D_i} + \sum_{i = 1}^{L}\frac{|D_i|}{n}BIAS_{D_i}^2 \\
& \phantom{MSE} =
\underbrace{\sum_{i = 1}^{L}\frac{|D_i|}{n}\left[\frac{1}{|D_i|}\sum_{j = 1}^{|D_i|} \left(\widehat{LID}(D_i;X_{ij})- \overline{LID}(D_i)\right)^2\right]}_{"~\mathrm{VAR}~"}
+
\underbrace{\sum_{i = 1}^{L}\frac{|D_i|}{n}\left(\overline{LID}(D_i) - LID(D_i)\right)^2}_{"~\mathrm{BIAS}^2~"}.
\end{aligned}
\end{equation}
where total empirical variance $"~\mathrm{VAR}~"$ and total bias squared $"~\mathrm{BIAS}^2~"$ are interpreted as the respective weighted sums of the manifold-wise empirical variances $VAR_{D_i}$, and bias squares $BIAS_{D_i}^2$, accordingly, for the purposes of the result data in our experiments involving datasets with different LID submanifolds (\emph{Lollipop} data \ref{lollipop}). Note that in the case of varying LID between manifolds, the error is influenced more by estimates of the high ground truth LID points. Therefore, we have to be careful when using evaluation by MSE if we want proportionally low errors for each manifold.

\section{Dataset descriptions}\label{Asec: Apppendix Dataset descriptions}

\begin{table}[H]
    \centering
    \caption{Collection of all datasets used in the experiments. Here $d$ represents the ground-truth LID (GT LID) of the datasets, $\dim$ is their representation dimension (i.e., the full dimensionality of the data space), and the last column displays references to detailed descriptions of their sampling function as well as an example visualization.}\vspace*{2mm}
    \begin{tabular}{|p{3cm}||p{2.5cm}|p{2.5cm}|p{3cm}|}
    \hline
    Dataset Name & $d$ (GT LID) & $\dim$ & Ref. (Appendix) \\
    \hline
    M1\_Sphere          & 10 & 11 & \ref{M1Sphere}    \\
    M2\_Affine\_3to5    &  3 &  5 &  \ref{M2Affine3to5}    \\
    M3\_Nonlinear\_4to6 &  4 &  6 &   \ref{M3Nonlinear4to6}   \\
    M4\_Nonlinear       &  4 &  8 & \ref{M4Nonlinear/M6Nonlinear/M8Nonlinear}\\
    M5b\_Helix2d        &  2 &  3 &   \ref{M5bHelix2d}   \\
    M6\_Nonlinear       &  6 & 36 &   \ref{M4Nonlinear/M6Nonlinear/M8Nonlinear}\\
    M7\_Roll            &  2 &  3 &  \ref{M7Roll}    \\
    M8\_Nonlinear       & 12 & 72 &   \ref{M4Nonlinear/M6Nonlinear/M8Nonlinear}\\
    M9\_Affine          & 20 & 20 &   \ref{M9Affine}   \\
    M10a\_Cubic         & 20 & 11 &  \ref{M10aCubic/M10bCubic/M10cCubic}    \\
    M10b\_Cubic         & 17 & 18 & \ref{M10aCubic/M10bCubic/M10cCubic}     \\
    M10c\_Cubic         & 24 & 25 &  \ref{M10aCubic/M10bCubic/M10cCubic}    \\
    M11\_Moebius        &  2 &  3 &   \ref{M11Moebius}    \\
    M12\_Norm           & 20 & 20 &   \ref{M12Norm}    \\
    M13a\_Scurve        &  2 &  3 &   \ref{M13aScurve}    \\
    Mn1\_Nonlinear      & 18 & 72 &   \ref{Mn1Nonlinear/Mn2Nonlinear}   \\
    Mn2\_Nonlinear      & 24 & 96 &   \ref{Mn1Nonlinear/Mn2Nonlinear}    \\
    Lollipop            & 1,2 & 2 &   \ref{lollipop} \\
    Uniform            & 30 & 100 &      \ref{uniform}\\
    \hline
    \end{tabular} \label{table:datasets}
\end{table}

\subsection{M1\_Sphere}\label{M1Sphere}
This manifold is a partial sphere surface with local intrinsic dimension $d$ at it's points, created using a transformation from a $d+1$ dimensional parameter space, embedded in $m$ dimensions. We sample from the parameter space using the standard normal distribution and at the end the data points come from the distribution of $\phi(X_1,\dots,X_{d+1})$ where specifically:
% --- Map \phi -----------------------------------------------------------
\begin{equation}
\begin{aligned}
\label{eqn:data1}
& X_i \;\overset{\text{iid.}}{\sim}\; \mathcal N(0,1),\; i = 1,\dots,d+1. \\
& \phi:\mathbb R^{d+1} \longrightarrow \mathbb R^{m},\\
& \phi(x_1,\dots,x_{d+1}) :=
   \bigl(\tfrac{x_1}{r},\dots,\tfrac{x_{d+1}}{r},\underbrace{0,\dots,0}_{m-d-1}\bigr),\\
& r = \sqrt{x_1^2+\dots+x_{d+1}^2}.
\end{aligned}
\end{equation}

\begin{figure}[H]
    \centering
\includegraphics[width=0.4\textwidth,height=0.4\textheight,keepaspectratio]{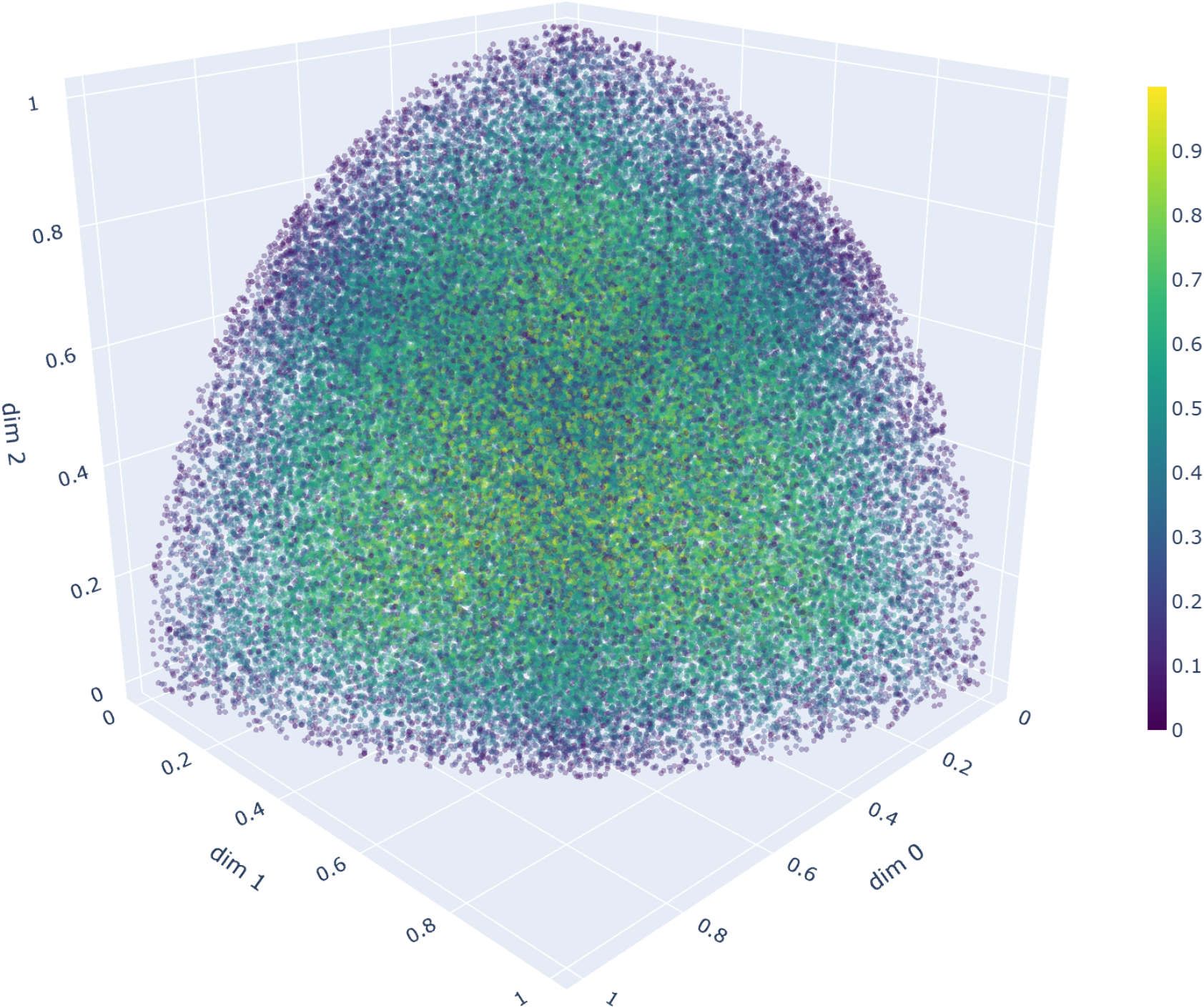}
\caption{Example visualization of the M1\_Sphere dataset, with $d=3, m=4$. Visualizing all $4$ dimensions. The color range represents the fourth-dimensional axes.}
\label{fig:data1}
\end{figure}

\subsection{M2\_Affine\_3to5}\label{M2Affine3to5}
This manifold is an affine hyperplane with local intrinsic dimension $3$ at it's points. We sample parameters randomly according to the uniform distribution, and transform them into $5$ dimensional space using the below function, so, the final data points follow the distribution of $\phi(X_1,X_2,X_3)$.
\begin{equation}
\begin{aligned}
\label{eqn:data_affine3_5}
& X_k \;\overset{\text{iid.}}{\sim}\; \operatorname{Unif}(0,4), 
\; k = 1,2,3. \\
& \phi:\mathbb R^{3} \longrightarrow \mathbb R^{5},\\
& \phi(x_1,x_2,x_3)=
  \begin{pmatrix}
    1.2\,x_1 - 0.5\,x_2 + 3           \\[2pt]
    0.5\,x_1 + 0.9\,x_2 - 1           \\[2pt]
   -0.5\,x_1 - 0.2\,x_2 + x_3         \\[2pt]
    0.4\,x_1 - 0.9\,x_2 - 0.1\,x_3    \\[2pt]
    1.1\,x_1 - 0.3\,x_2 + 8
  \end{pmatrix}.
\end{aligned}
\end{equation}

\begin{figure}[H]
    \centering
\includegraphics[width=0.4\textwidth,height=0.4\textheight,keepaspectratio]{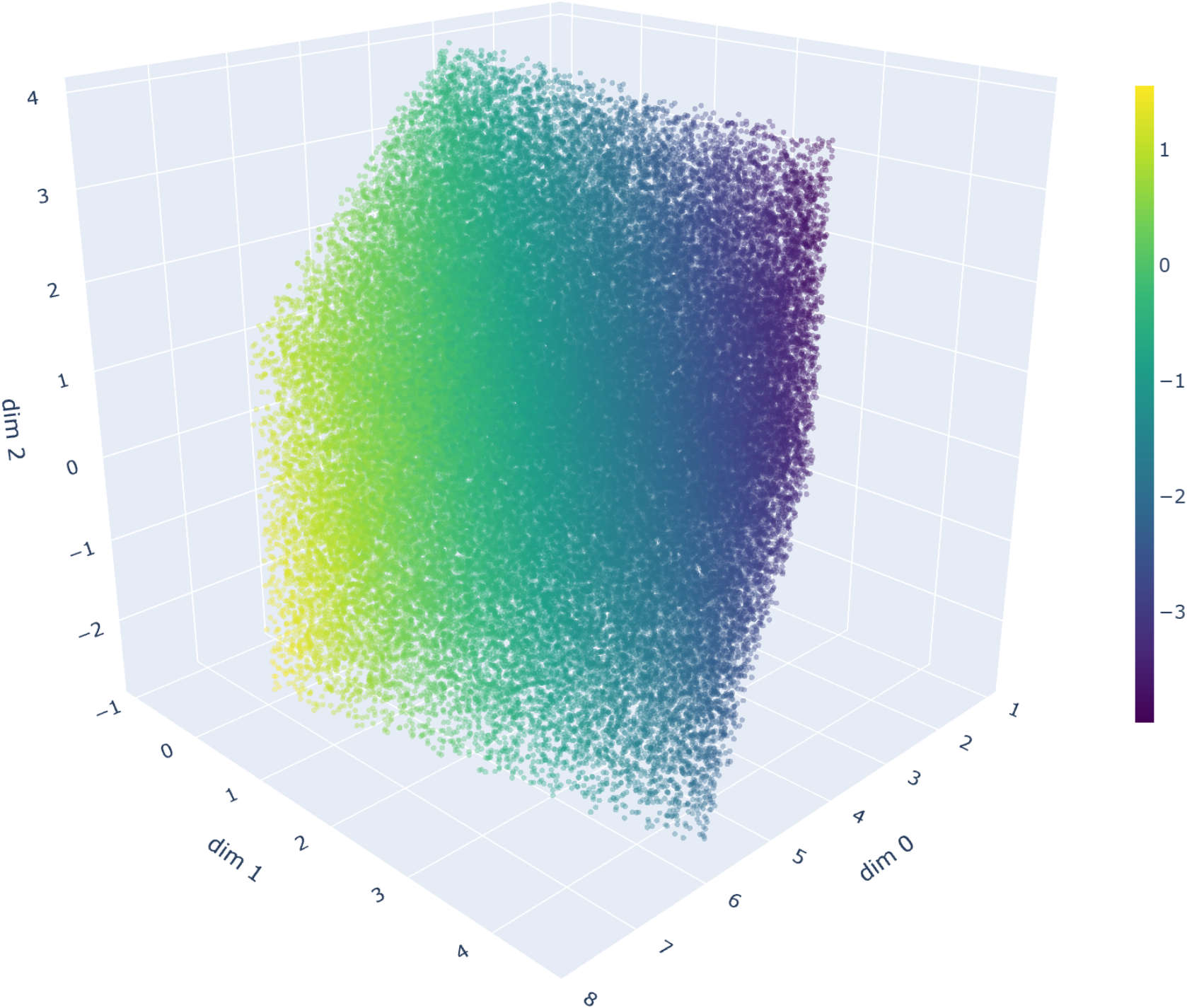}
\caption{Example visualization of the M2\_Affine\_3to5 dataset, with $d=3$. Showing 4 of the 5 dimensions. The color range represents the fourth-dimensional axes.}
\label{fig:data_affine3_5}
\end{figure}

\subsection{M3\_Nonlinear\_4to6}\label{M3Nonlinear4to6}
A nonlinear manifold with $4$ parameters sampled from the uniform distribution. To get a feel for the shape, we can look at it in a way that the first two coordinates define a distribution on the unit disk with larger probability density near the axes because of the $x_1^2$ and $x_2^2$ multipliers, and the other coordinates in a sense pull this distribution into the other dimensions according to various quadratic structures. As data points come from the distribution of $\phi(X_0,X_1,X_2,X_3)$, where:
\begin{equation}
\begin{aligned}
\label{eqn:data_nonlinear4_6}
& X_k \;\overset{\text{iid.}}{\sim}\; \operatorname{Unif}(0,1),\; k=0,\dots,3 . \\
& \phi:\mathbb R^{4} \longrightarrow \mathbb R^{6},\\
& \phi(x_0,x_1,x_2,x_3)=
\begin{pmatrix}
      x_1^{2}\cos(2\pi x_0) \\[2pt]
      x_2^{2}\sin(2\pi x_0) \\[2pt]
      x_1 + x_2 + (x_1 - x_3)^{2} \\[2pt]
      x_1 - 2x_2 + (x_0 - x_3)^{2} \\[2pt]
      -x_1 - 2x_2 + (x_2 - x_3)^{2} \\[2pt]
      x_0^{2} - x_1^{2} + x_2^{2} - x_3^{2} \\[2pt]
\end{pmatrix}.
\end{aligned}
\end{equation}

\begin{figure}[H]
    \centering
\includegraphics[width=0.4\textwidth,height=0.4\textheight,keepaspectratio]{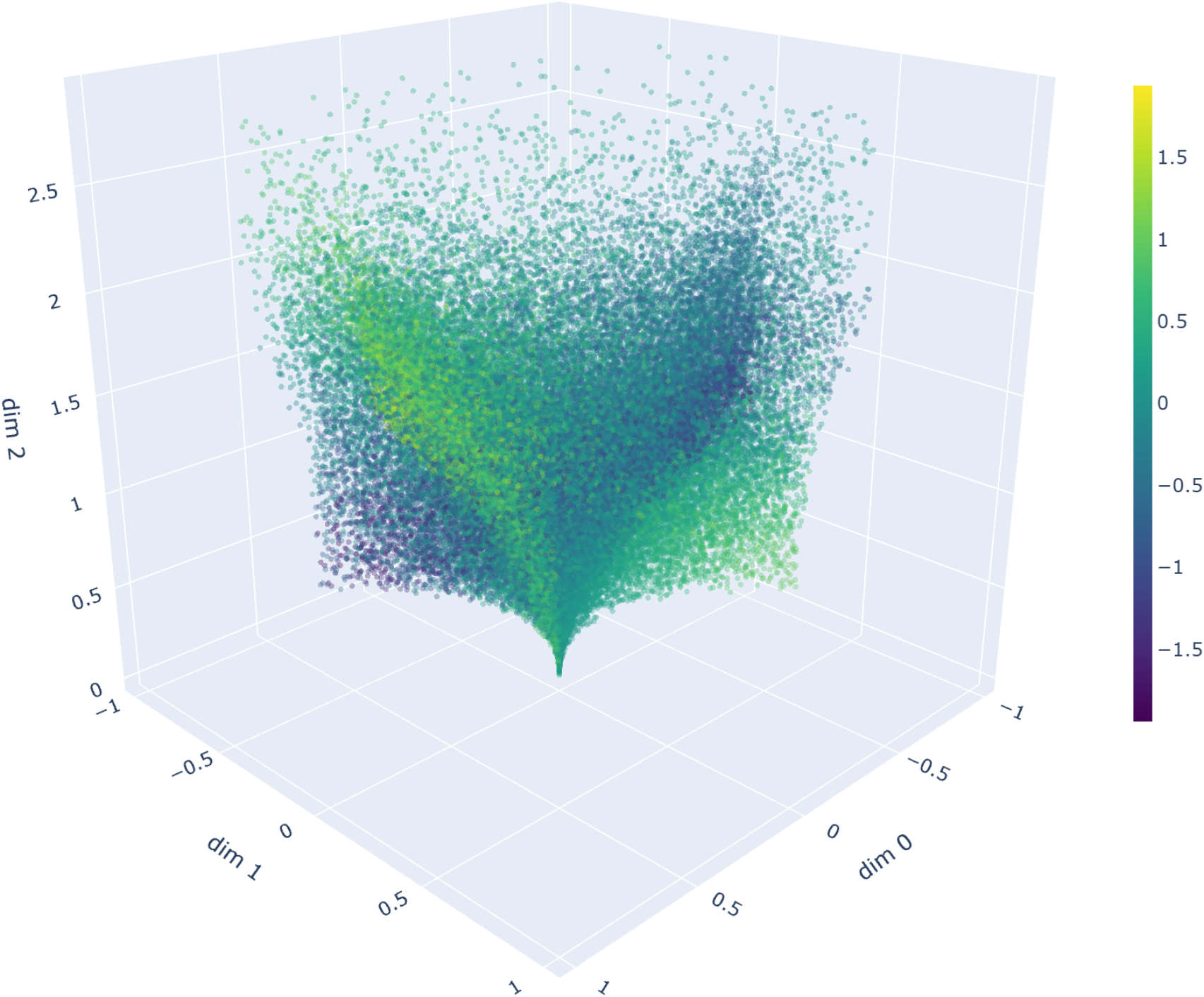}
\caption{Example visualization of the M3\_Nonlinear\_4to6 dataset, with $d=4$. Showing 4 of the 6 dimensions. The color range represents the fourth dimensional axes.}
\label{fig:data_nonlinear4_6}
\end{figure}

\subsection{M4\_Nonlinear/M6\_Nonlinear/M8\_Nonlinear}\label{M4Nonlinear/M6Nonlinear/M8Nonlinear} 
A highly curved nonlinear manifold. The basic shape is a hypersurface on the cartesian product of $d$ unit disks, obeying the constraints that when thinking of each disk in terms of polar coordinates, the radius of the point according to one disk defines the angle on the next disk, and so on through the $d$ unit disks, in a cyclic way getting back to the first disk at the end. This shape is then repeated $m$ times along a diagonal linear subspace, or in other words is multiplied with the $(1,\dots,1)$ $m$-long vector as a Kronecker product. As mathematically described, data points come from the distribution of $\phi(X_0,\dots, X_{d-1})$, where
\begin{equation}
\begin{aligned}
\label{eqn:data_nonlinear_general}
& X_i \;\overset{\text{iid.}}{\sim}\; \operatorname{Unif}(0,1), \; i=0,\dots,d-1. \\[4pt]
& \phi:\mathbb R^{d} \longrightarrow \mathbb R^{\dim}, \qquad
  m=\frac{\mathrm{dim}}{2d}. \\[4pt]
& \text{For }k=0,\dots,d-1,\;\ell=0,\dots,m-1: \\[2pt]
& \phi_{\,2k+2d\ell}(x_0,\dots,x_{d-1}) =
\begin{cases}
x_{k+1}\,\cos\!\bigl(2\pi x_k\bigr), & 0 \le k \le d-2,\\[4pt]
x_{0}\,\cos\!\bigl(2\pi x_{d-1}\bigr), & k = d-1,
\end{cases}
\\[10pt]
& \phi_{\,2k+1+2d\ell}(x_0,\dots,x_{d-1}) =
\begin{cases}
x_{k+1}\,\sin\!\bigl(2\pi x_k\bigr), & 0 \le k \le d-2,\\[4pt]
x_{0}\,\sin\!\bigl(2\pi x_{d-1}\bigr), & k = d-1.
\end{cases}
\end{aligned}
\end{equation}

\begin{figure}[H]
    \centering
\includegraphics[width=0.4\textwidth,height=0.4\textheight,keepaspectratio]{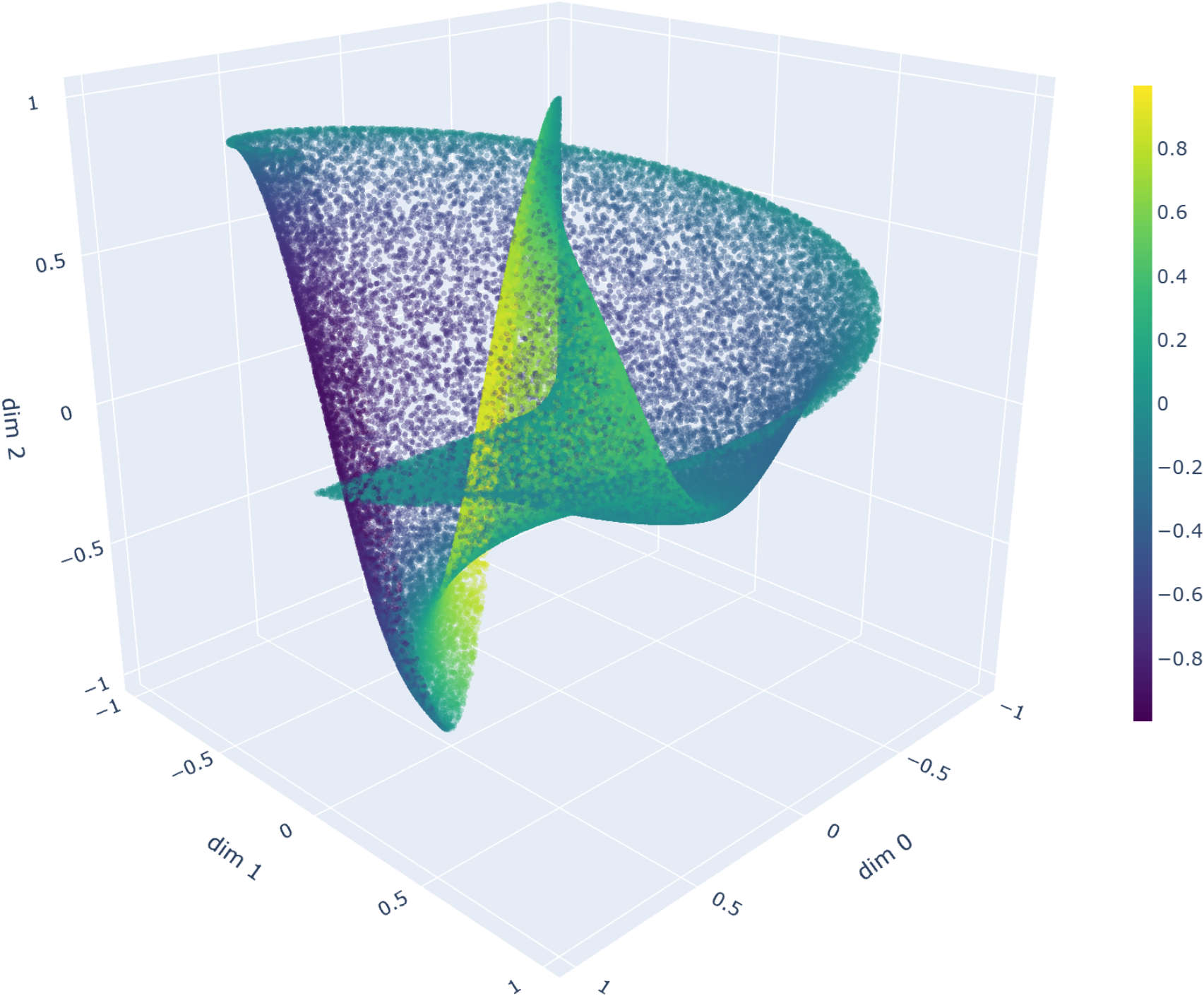}
\caption{Example visualization of the M4\_Nonlinear/M6\_Nonlinear/M8\_Nonlinear datasets, with $d=4, m=1$. Showing 4 of the 8 dimensions. The color range represents the fourth-dimensional axes.}
\label{fig:data_nonlinear_general}
\end{figure}

\subsection{M5b\_Helix2d}\label{M5bHelix2d} 
A helicoid surface embedded in $\dim$ dimensions via $\dim-3$ extra, all $0$ coordinates. Visually speaking, the interesting part is like a disk surface that has been cut into radially and is continually curved upwards into the $3$rd dimension. Mathematically, the data points follow the distribution of $\phi(R,P)$, where:
\begin{equation}
\begin{aligned}
\label{eqn:data_helix2}
& R \sim \operatorname{Unif}(0,10\pi), \;
  P \sim \operatorname{Unif}(0,10\pi), \; \\[4pt]
& \phi:\mathbb R^{2} \longrightarrow \mathbb R^{\dim}, \quad \text{dim}\ge 3, \\[4pt]
& \phi(r,p)=
   \bigl(
     r\cos p,\;
     r\sin p,\;
     \tfrac12 p,\;
     \underbrace{0,\dots,0}_{\dim-3}
   \bigr).
\end{aligned}
\end{equation}

\begin{figure}[H]
    \centering
\includegraphics[width=0.4\textwidth,height=0.4\textheight,keepaspectratio]{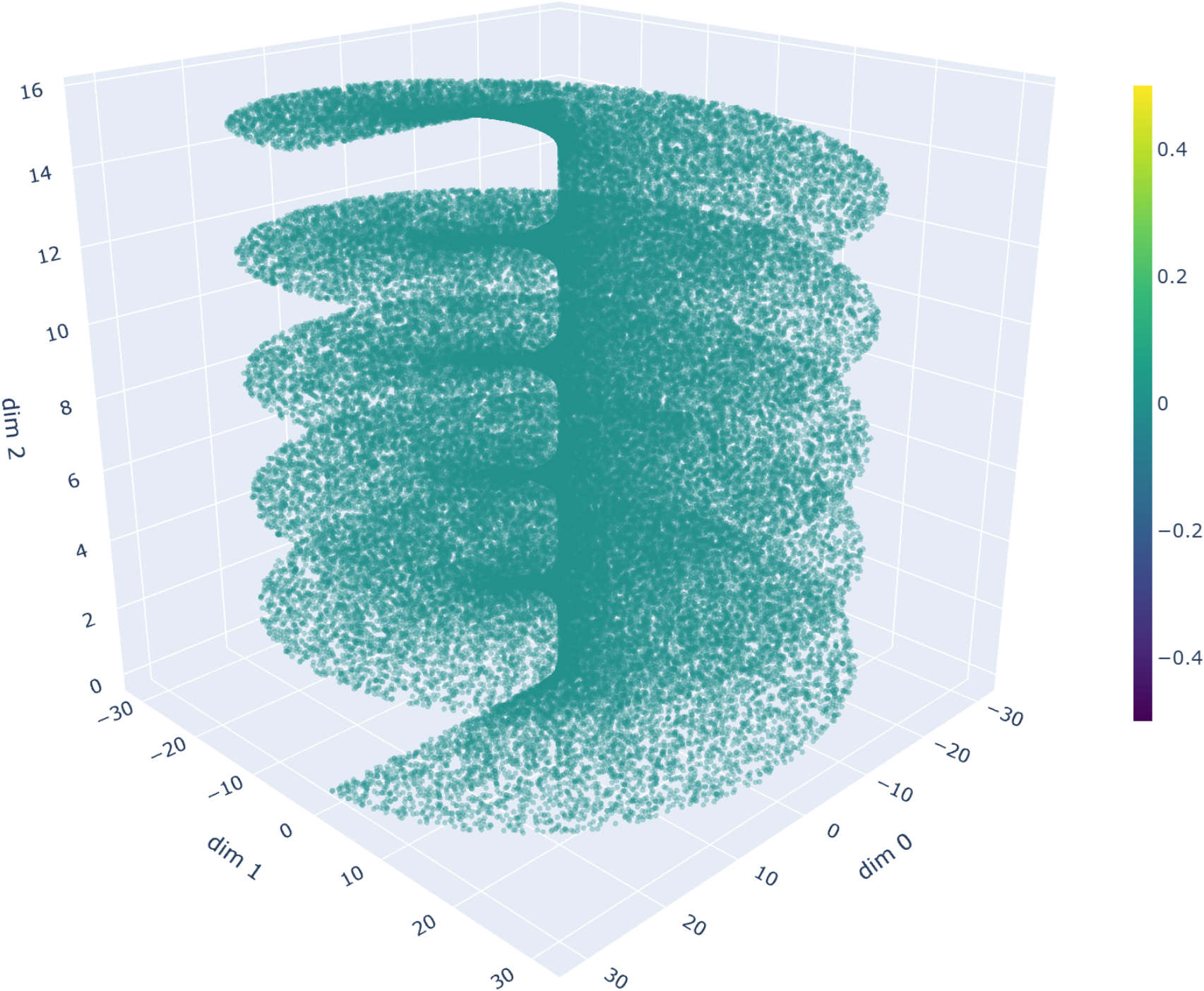}
\caption{Example visualization of the M5b\_Helix2d dataset, with $d=2$. Showing all $3$ dimensions.}
\label{fig:data_helix2}
\end{figure}

\subsection{M7\_Roll}\label{M7Roll}
A loose roll surface, as if a rectangular band has been rolled up, like tape or toilet paper, but in case of this manifold, with a fairly large space in between the surface. The surface is embedded in $\dim$ dimensions via $\dim-3$ extra, all $0$ coordinates. Concretely, the data points follow the distribution of $\phi(T, P)$, where:
\begin{equation}
\begin{aligned}
\label{eqn:data_roll}
& T \sim \operatorname{Unif}\!\bigl(1.5\pi,\,4.5\pi\bigr),\;
  P \sim \operatorname{Unif}(0,21), \; \\[6pt]
& \phi:\mathbb R^{2} \longrightarrow \mathbb R^{\dim}, \quad \text{dim}\ge 3, \\[6pt]
& \phi(t,p)=
   \bigl(
     t\cos t,\;
     p,\;
     t\sin t,\;
     \underbrace{0,\dots,0}_{\dim-3}
   \bigr).
\end{aligned}
\end{equation}

\begin{figure}[H]
    \centering
\includegraphics[width=0.4\textwidth,height=0.4\textheight,keepaspectratio]{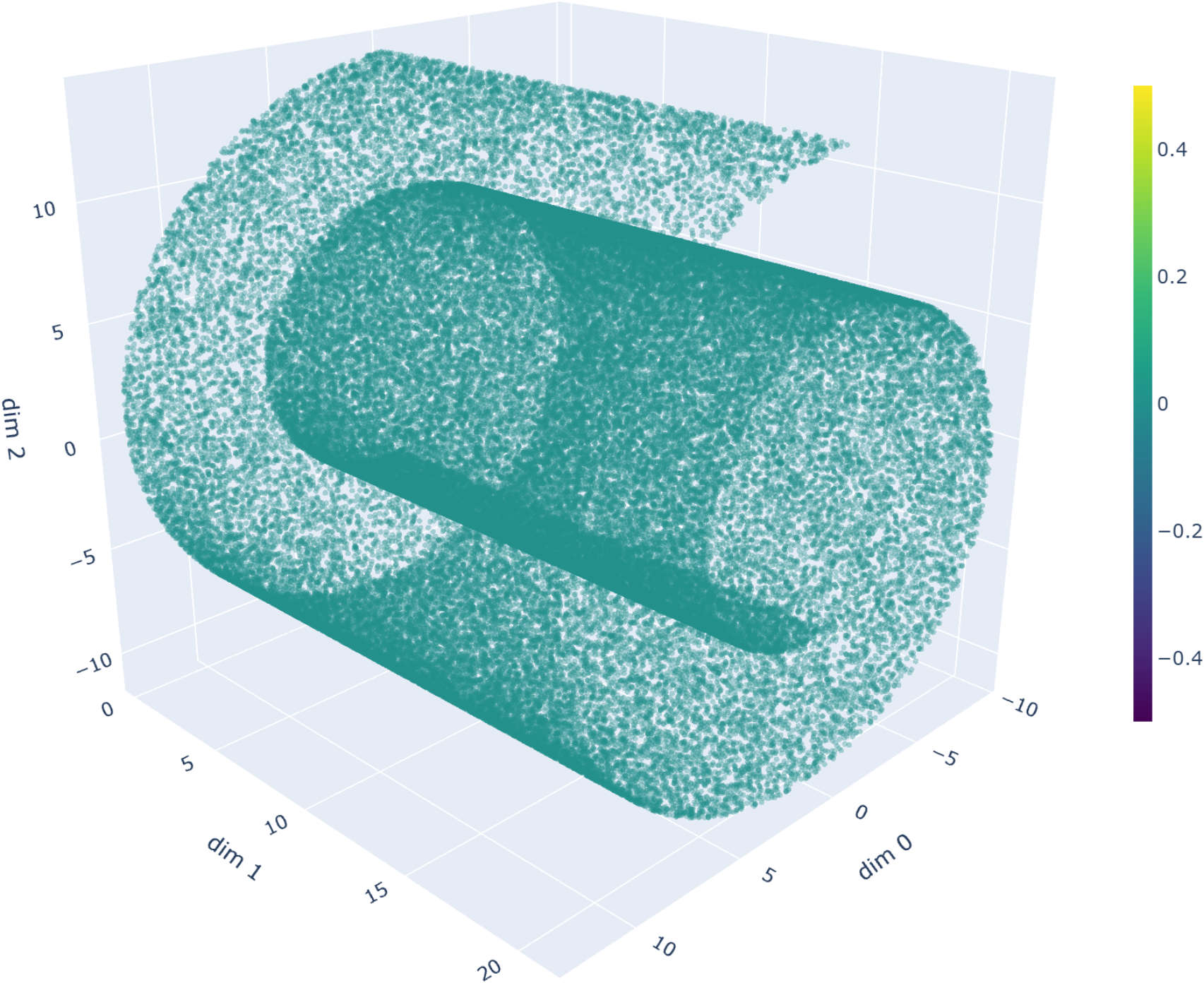}
\caption{Example visualization of the M7\_Roll dataset, with $d=2$. Showing all $3$ dimensions.}
\label{fig:M7Roll}
\end{figure}

\subsection{M9\_Affine}\label{M9Affine}
A simple $d$ dimensional hypercube embedded in $\dim$ dimensions via $\dim-d$ extra, all $0$ coordinates. The data points follow the distribution of $\phi(X_1,\dots, X_d)$, where
\begin{equation}
\begin{aligned}
\label{eqn:data_affine_proj}
& X_k \;\overset{\text{iid.}}{\sim}\; \operatorname{Unif}(-2.5,\,2.5),\; k=1,\dots,d. \\[6pt]
& \phi:\mathbb R^{d} \longrightarrow \mathbb R^{\dim}, \quad \text{dim}\ge d, \\[6pt]
& \phi(x_1,\dots,x_d)=
   \bigl(
     x_1,\dots,x_d,\;
     \underbrace{0,\dots,0}_{\dim-d}
   \bigr).
\end{aligned}
\end{equation}

\begin{figure}[H]
    \centering
\includegraphics[width=0.4\textwidth,height=0.4\textheight,keepaspectratio]{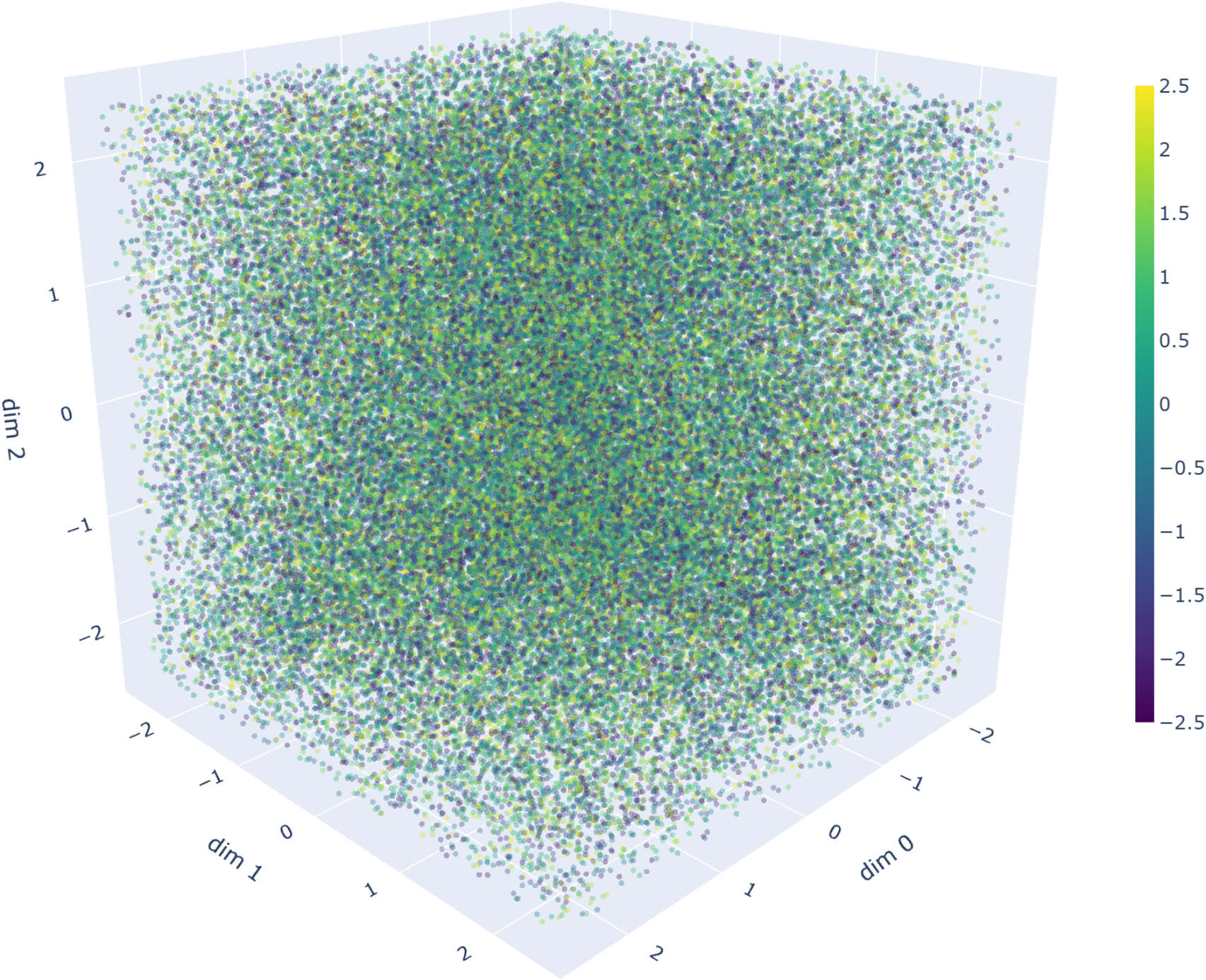}
\caption{Example visualization of the M9\_Affine dataset, with $d=4$. Showing all $4$ dimensions. The color range represents the fourth-dimensional axes.}
\label{fig:M9Affine}
\end{figure}

\subsection{M10a\_Cubic/M10b\_Cubic/M10c\_Cubic}\label{M10aCubic/M10bCubic/M10cCubic} 

The hypersurface of a $d+1$ dimensional hypercube, meaning that this manifold has intrinsic dimension $d$. The manifold is embedded in $\dim$ dimensions via $\dim-d-1$ extra, all $0$ coordinates. Mathematically we can describe the points as coming from the distribution of $\psi_{I,S}(U_1,\dots,U_d)$, where $(I,S)$ is a uniform categorical random vector over the $\{0,\dots,d\}\times\{0,1\}$ support, and 
\begin{equation}
\begin{aligned}
\label{eqn:data_cubic_face}
& U_j \;\overset{\text{iid.}}{\sim}\; \operatorname{Unif}(0,1),\; j=1,\dots,d. \\[4pt]
& \text{Hypercube facet indices } (i,s)\in\{0,\dots,d\}\times\{0,1\}. \\[4pt]
& \psi_{i,s}:\mathbb R^{d}\longrightarrow\mathbb R^{\dim}, \\[2pt]
& \psi_{i,s}(u_1,\dots,u_d)=
   \bigl(
     u_1,\dots,u_i,\;
     s,\;
     u_{i+1},\dots,u_d,\;
     \underbrace{0,\dots,0}_{\dim-d-1}
   \bigr). \\[6pt]
\end{aligned}
\end{equation}

\begin{figure}[H]
    \centering
\includegraphics[width=0.4\textwidth,height=0.4\textheight,keepaspectratio]{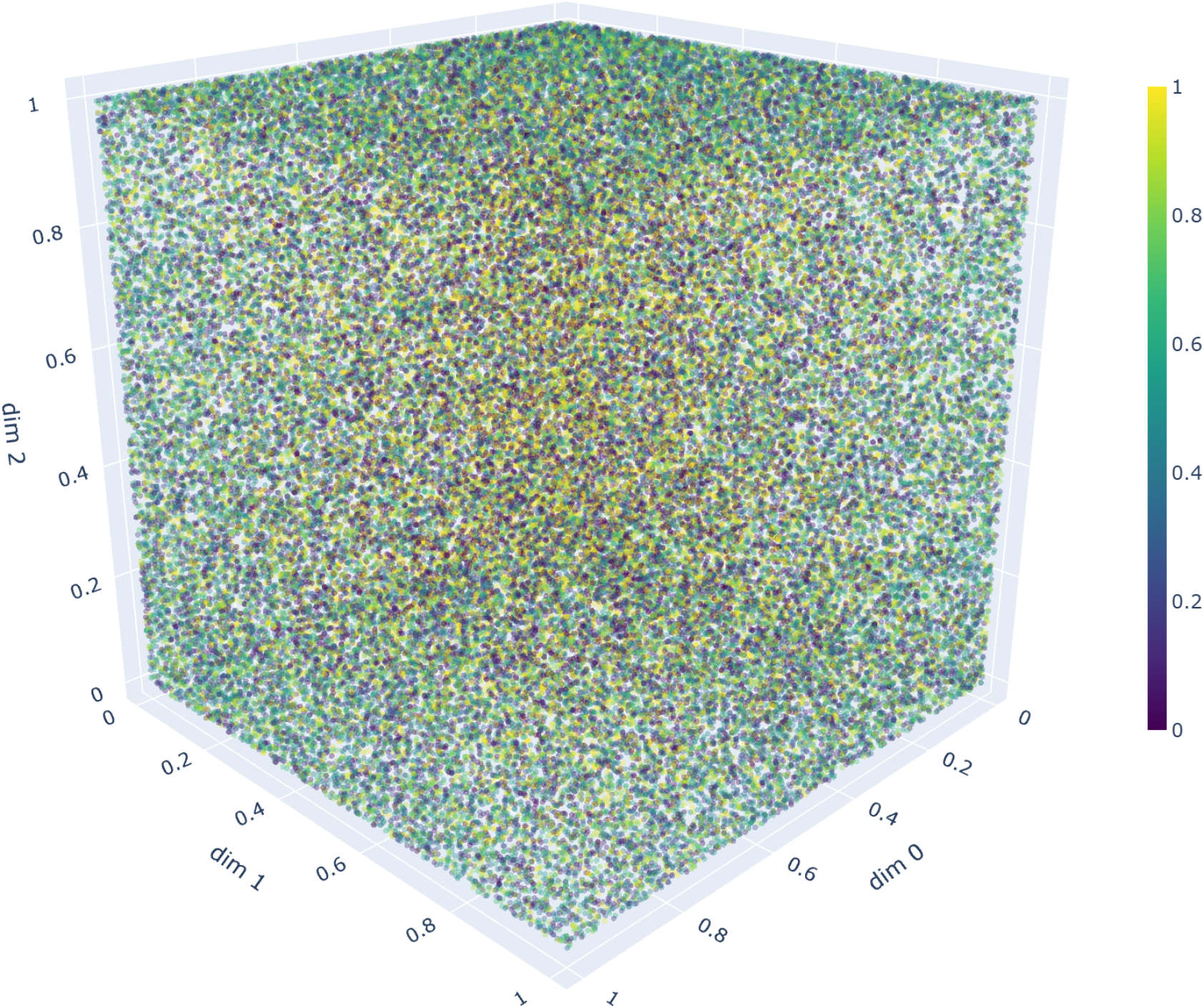}
\caption{Example visualization of the M10\_Cubic dataset, with $d=4$. Showing all $4$ dimensions. The color range represents the fourth-dimensional axes.}
\label{fig:M10Cubic}
\end{figure}

\subsection{M11\_Moebius}\label{M11Moebius} 
This manifold is a Moebius strip surface. A $2$ dimensional surface, densely twisted in $3$ dimensional space along a circular pattern around a central axis. Specifically, the data points follow the distribution of $\phi(\Phi, R)$ where:
\begin{equation}
\begin{aligned}
\label{eqn:data_moebius}
& \Phi \sim \operatorname{Unif}(0,2\pi),\;
  R \sim \operatorname{Unif}(-1,1),\; \\[6pt]
& \phi:\mathbb R^{2} \longrightarrow \mathbb R^{3}, \\[4pt]
& \phi(\varphi,r)=
   \bigl(
      \bigl[1+\tfrac12 r\cos(5\varphi)\bigr]\cos(\varphi),\;
      \bigl[1+\tfrac12 r\cos(5\varphi)\bigr]\sin(\varphi),\;
      \tfrac12 r\sin(5\varphi)
   \bigr).
\end{aligned}
\end{equation}

\begin{figure}[H]
    \centering
\includegraphics[width=0.4\textwidth,height=0.4\textheight,keepaspectratio]{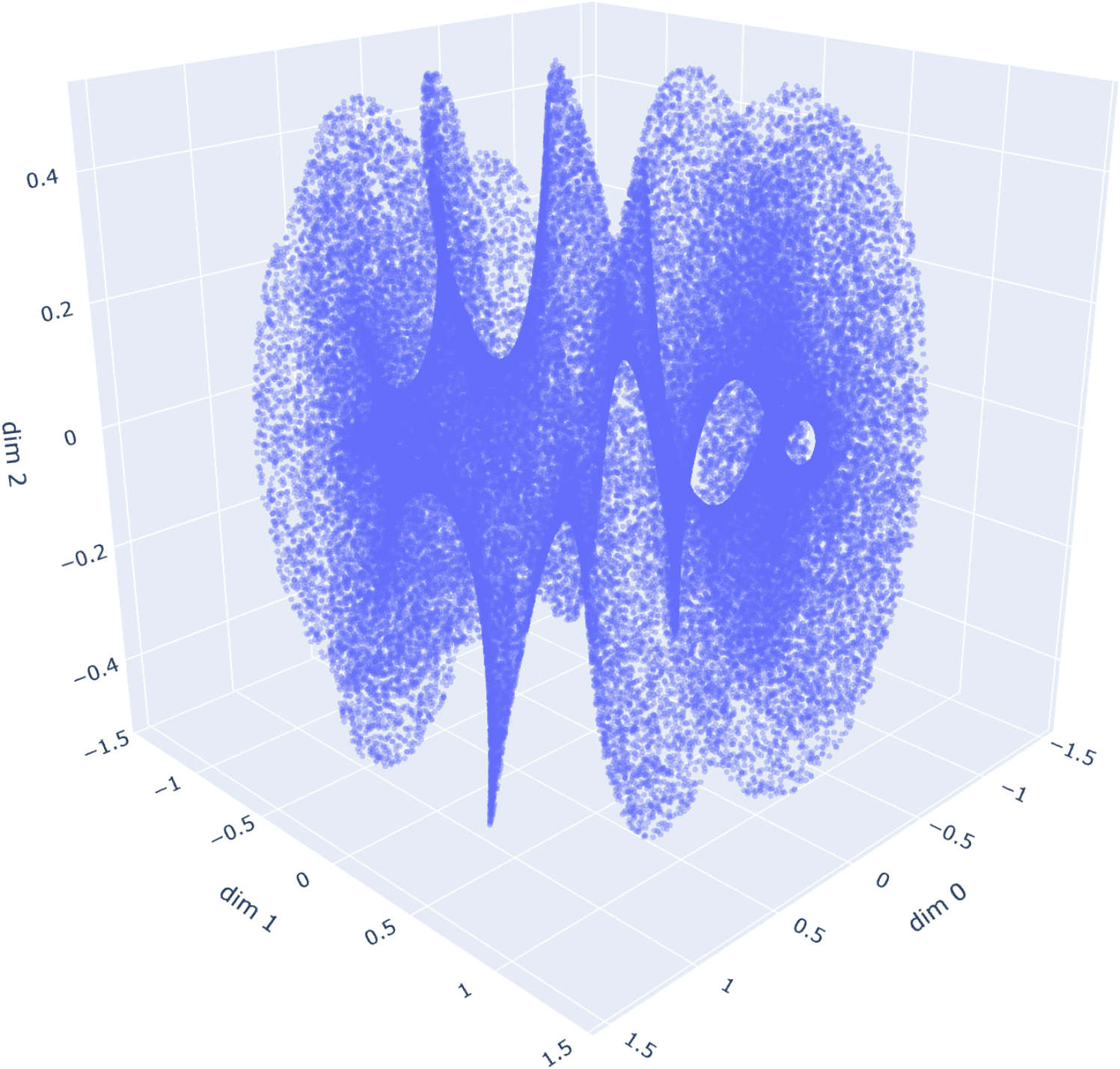}
\caption{Example visualization of the M11\_Moebius dataset, with $d=2, \text{dim}=3$. Showing all $3$ dimensions.}
\label{fig:data_moebius}
\end{figure}

\subsection{M12\_Norm} \label{M12Norm}
A $d$ dimensional hypercube embedded in $\dim$ dimensions via $\dim-d$ extra all $0$ coordinates. However, the parameters are sampled from the standard normal distribution, therefore the data points follow the distribution of $\phi(X_1,\dots,X_d)$ where:
\begin{equation}
\begin{aligned}
\label{eqn:data_norm}
& X_i \;\overset{\text{iid.}}{\sim}\; \mathcal N(0,1),\; i=1,\dots,d. \\[6pt]
& \phi:\mathbb R^{d} \longrightarrow \mathbb R^{\dim}, \quad \text{dim}\ge d, \\[4pt]
& \phi(x_1,\dots,x_d)=
   \bigl(
     x_1,\dots,x_d,\;
     \underbrace{0,\dots,0}_{\,dim-d}
   \bigr).
\end{aligned}
\end{equation}

\begin{figure}[H]
    \centering
\includegraphics[width=0.4\textwidth,height=0.4\textheight,keepaspectratio]{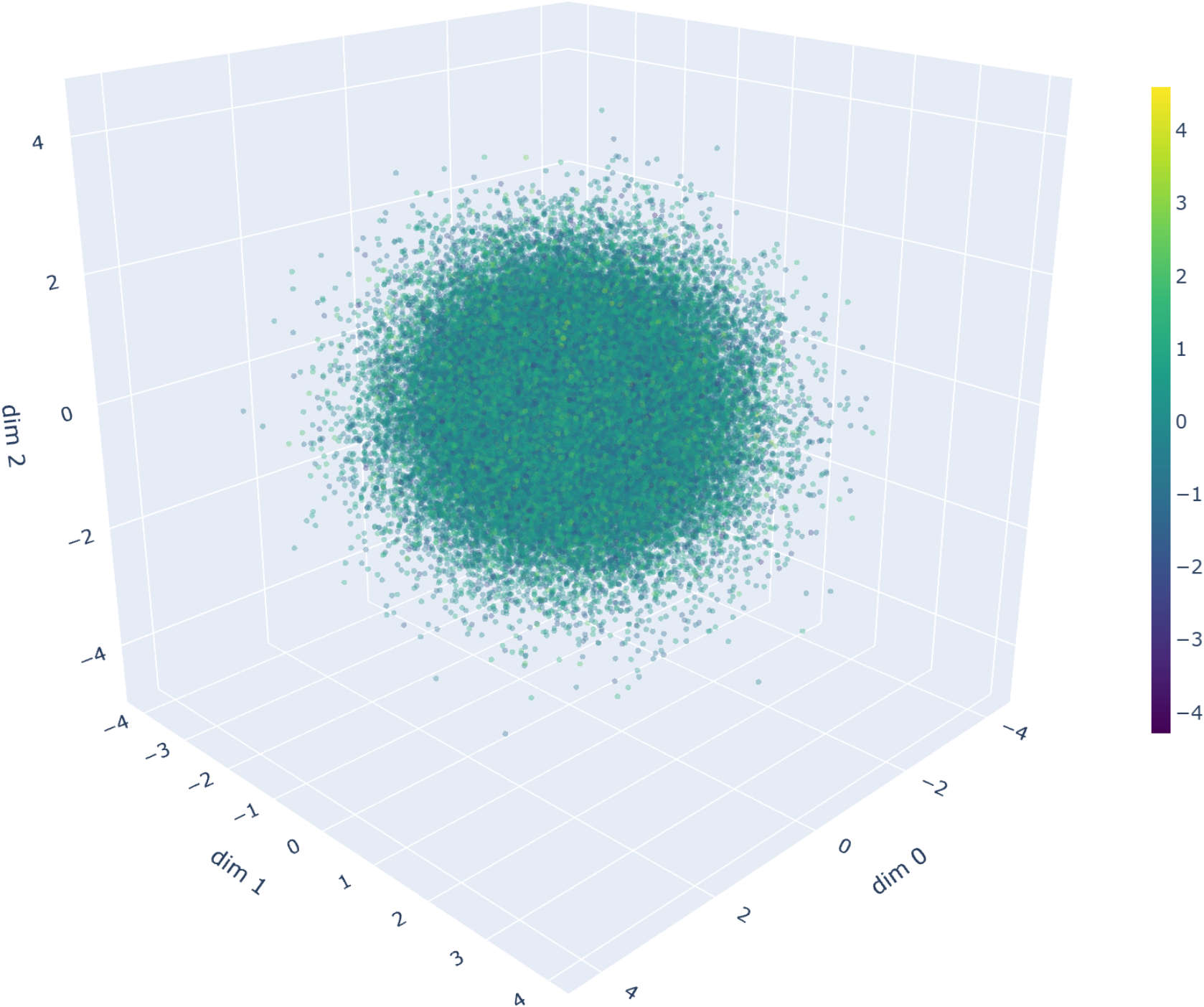}
\caption{Example visualization of the M12\_Norm dataset, with $d=4$. Showing all $4$ dimensions. The color range represents the fourth-dimensional axes.}
\label{fig:M12Norm}
\end{figure}

\subsection{M13a\_Scurve} \label{M13aScurve}
A rectangular band surface with intrinsic dimensionality $2$ curved in an "S" shape in $3$ dimensions, embedded in $\dim$ dimensions via $\dim-3$ extra all $0$ coordinates. The data points follow the distribution of $\phi(T,P)$, where
\begin{equation}
\begin{aligned}
\label{eqn:data_scurve}
& T \sim \operatorname{Unif}(-1.5\pi,\,1.5\pi),\;
  P \sim \operatorname{Unif}(0,2),\; \\[6pt]
& \phi:\mathbb R^{2} \longrightarrow \mathbb R^{\dim}, \quad \text{dim}\ge 3, \\[4pt]
& \phi(t,p)=
   \bigl(
      \sin t,\;
      p,\;
      \operatorname{sign}(t)\!\bigl(\cos t - 1\bigr),\;
      \underbrace{0,\dots,0}_{\dim-3}
   \bigr).
\end{aligned}
\end{equation}

\begin{figure}[H]
    \centering
\includegraphics[width=0.4\textwidth,height=0.4\textheight,keepaspectratio]{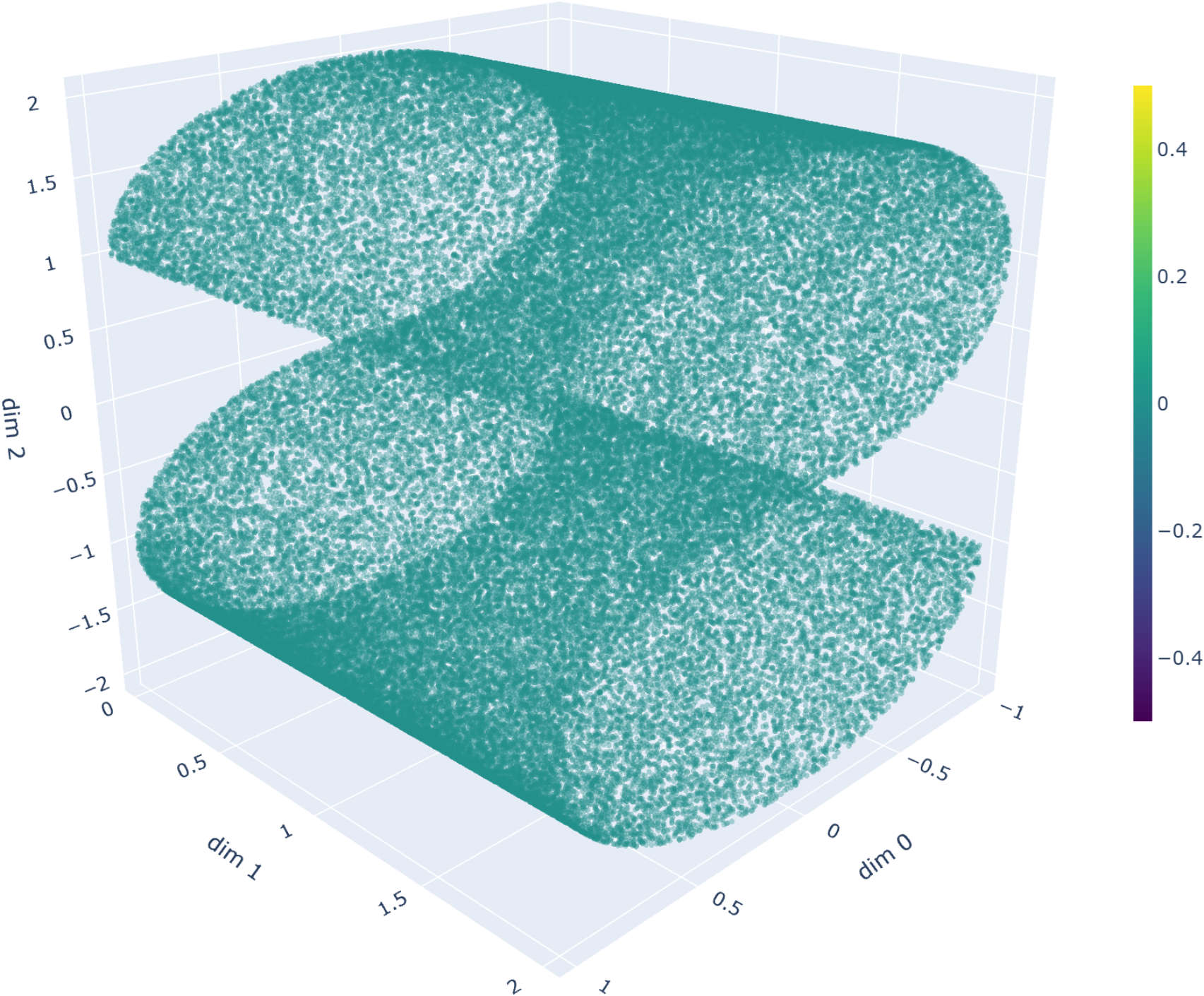}
\caption{Example visualization of the M13a\_Scurve dataset, with $d=2, \text{dim}=3$. Showing all $3$ dimensions.}
\label{fig:M13aScurve}
\end{figure}

\subsection{Mn1\_Nonlinear/Mn2\_Nonlinear} \label{Mn1Nonlinear/Mn2Nonlinear}
A hypersurface with gradual, multidirectional curvature, repeated twice along a diagonal linear subspace, or in other words is multiplied by $(1,1)$ as a Kronecker product. The data points are drawn from the distribution of $\phi(X_0,\dots,X_{d-1})$, described by
\begin{equation}
\begin{aligned}
\label{eqn:data_campadelli}
& X_i \;\overset{\text{iid.}}{\sim}\; \operatorname{Unif}(0,1),\; i=0,\dots,d-1. \\[6pt]
& \phi:\mathbb R^{d} \longrightarrow \mathbb R^{\,4d}. \\[6pt]
& \text{For } i=0,\dots,d-1 \text{ (with } j=d-1-i\text{)}: \\[2pt]
& \qquad
  \phi_{\,i+1}(x)        = \tan\!\bigl(x_i \cos x_j\bigr), \\[2pt]
& \qquad
  \phi_{\,d+i+1}(x)      = \arctan\!\bigl(x_j \sin x_i\bigr), \\[2pt]
& \qquad
  \phi_{\,2d+i+1}(x)     = \phi_{\,i+1}(x), \\[2pt]
& \qquad
  \phi_{\,3d+i+1}(x)     = \phi_{\,d+i+1}(x).
\end{aligned}
\end{equation}

\begin{figure}[H]
    \centering
\includegraphics[width=0.4\textwidth,height=0.4\textheight,keepaspectratio]{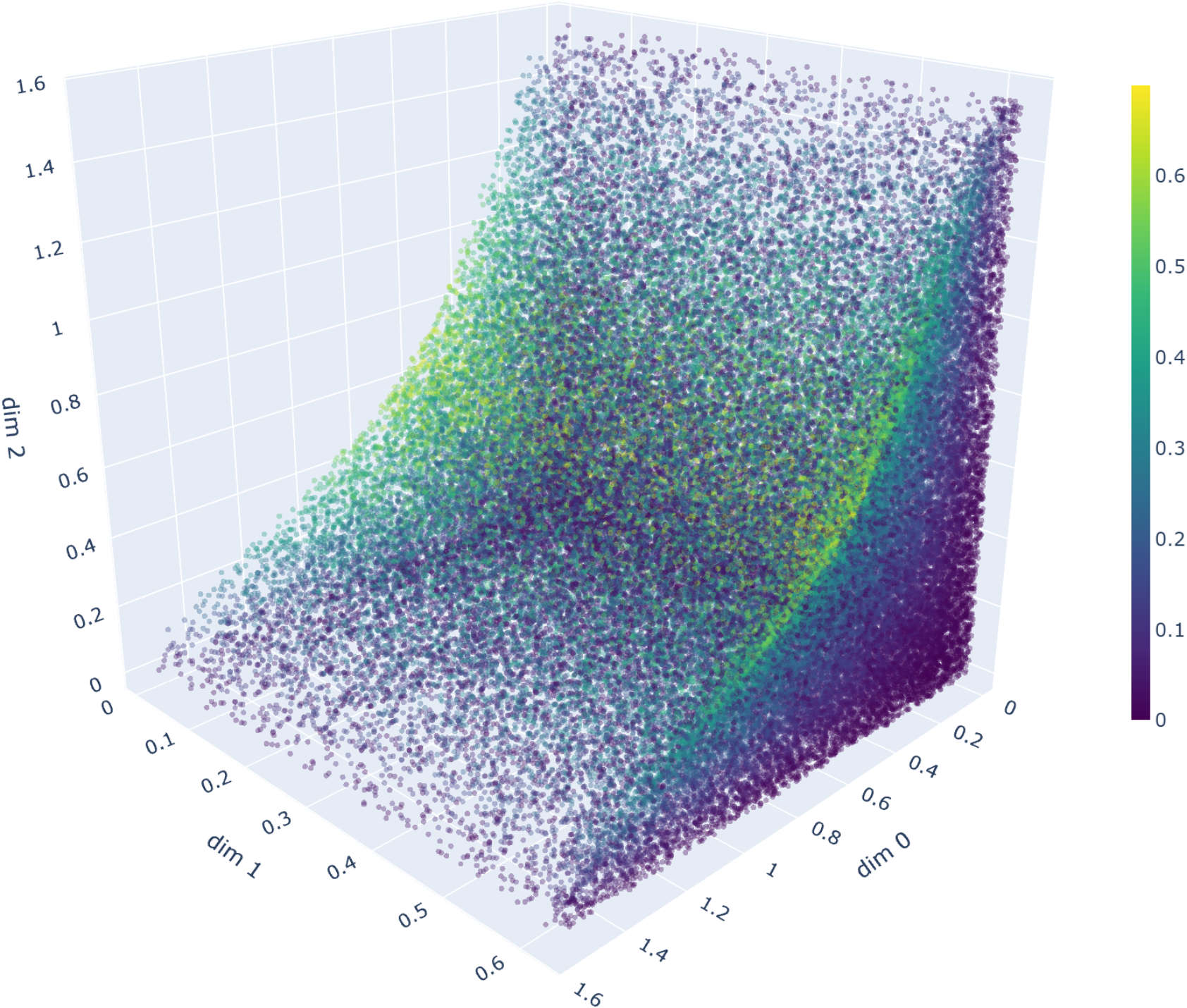}
\caption{Example visualization of the Mn1\_Nonlinear/Mn2\_Nonlinear dataset, with $d=3,\; \text{dim}=12$. Showing $4$ out of the $12$ dimensions.}
\label{fig:Mn1Nonlinear/Mn2Nonlinear}
\end{figure}

\subsection{Lollipop}\label{lollipop}

This manifold is called the "lollipop" dataset because it looks like a lollipop. It consists of two parts, a line section, or stick, which has intrinsic dimensionality $1$, and at one end of it, without the two densities intersecting, there is the other part, a disk, or candy, which has intrinsic dimensionality $2$. Specifically, sample points are drawn from the distribution of $\phi_Z(R, \Phi, T)$, where
$Z\sim\operatorname{Cat}\{C:0.95,\;S:0.05\}$ categorical random variable indicates the component, $C$ the candy part with $95\%$ probability and $S$ the stick part with $5\%$ probability, and we have the following component wise distributions
\begin{equation}
\begin{aligned}
\label{eqn:data_lollipop}
& R \sim \operatorname{Unif}(0,1),\qquad
\Phi \sim \operatorname{Unif}(0,2\pi),\qquad
T \sim \operatorname{Unif}\!\Bigl(0,\ 2-\tfrac{1}{\sqrt{2}}\Bigr). \\[6pt]
& \phi_C:\mathbb R^2\!\to\mathbb R^2,\ \ \phi_S:\mathbb R\!\to\mathbb R^2.\\[2pt]
&\phi_C(R, \Phi)\;=\;\bigl(2+\sqrt{R}\,\sin\Phi,\ \ 2+\sqrt{R}\,\cos\Phi\bigr),\\[2pt]
& \phi_S(T)\;=\;(T,\ T). \\[8pt]
\end{aligned}
\end{equation}

\begin{figure}[H]
    \centering
\includegraphics[width=0.4\textwidth,height=0.4\textheight,keepaspectratio]{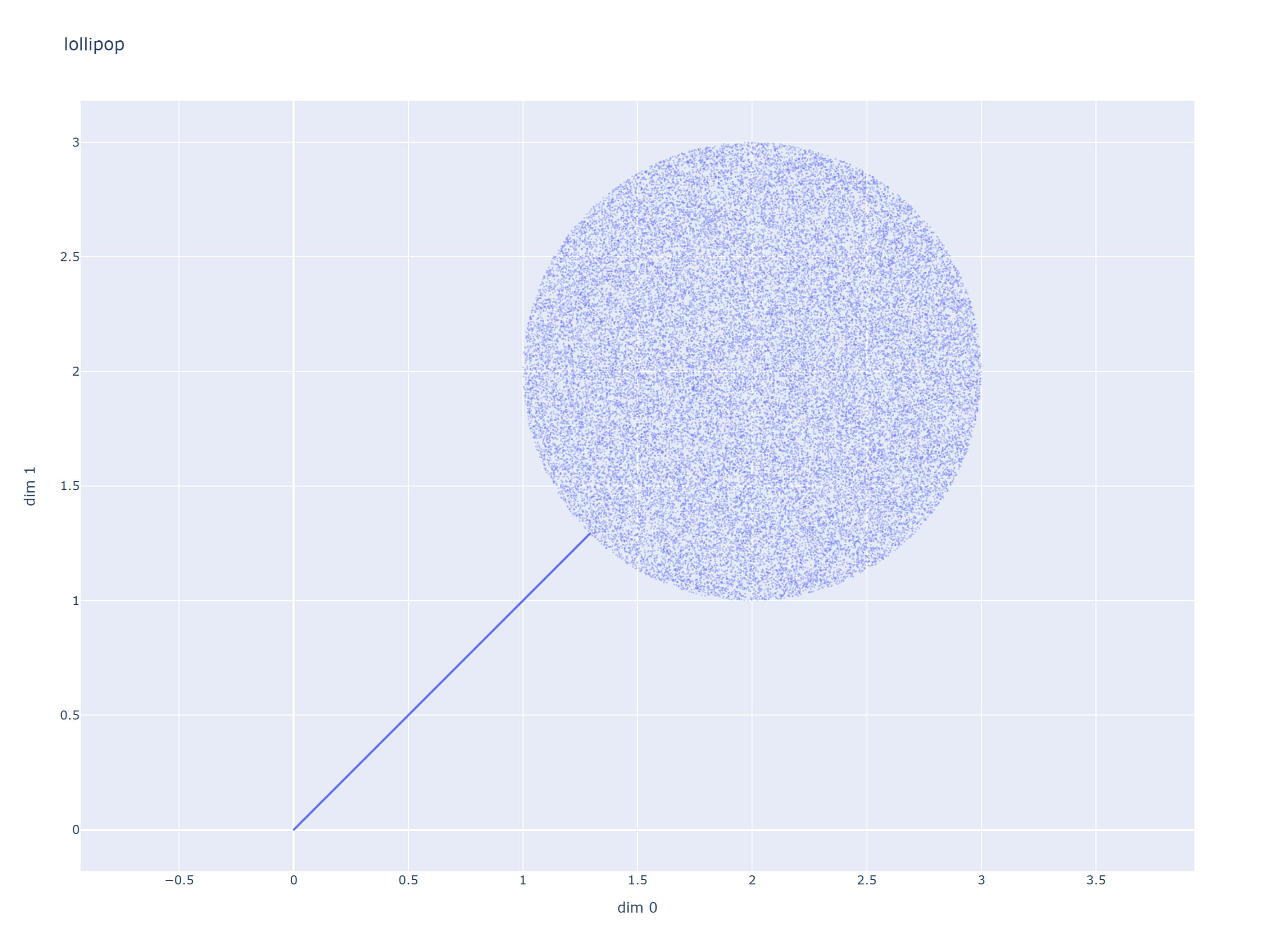}
\caption{Example visualization of the lollipop dataset, with $d=2$. Showing all $2$ dimensions.}
\label{fig:lollipop}
\end{figure}

\subsection{Uniform}\label{uniform} 
Lastly we have this manifold, that we additionally included for computationally efficient testing in high ambient and local intrinsic dimensionalities. The first $d$ coordinates are i.i.d. uniform, comprising a $d$ dimensional hypercube, which is embedded in $\dim$ dimensions via $\dim\!-d\!-\!1$ exact copies of the last variable, extending the hypercube into $\dim$ dimensions along a simple linear subspace. So, sample points are drawn i.i.d. from the distribution of $\phi(V_1,\dots,V_{\mathrm{d}-1},T)$, where
\begin{equation}
\begin{aligned}
\label{eqn:data_simple_d}
& V_i \;\overset{\text{iid.}}{\sim}\; \operatorname{Unif}(0,1),\quad i=1,\dots,\mathrm{d}-1,\qquad
T \sim \operatorname{Unif}(0,1). \\[4pt]
& \phi:\mathbb R^{\mathrm{d}}\longrightarrow\mathbb R^{\mathrm{dim}}, \\[2pt]
& \phi(v_1,\dots,v_{\mathrm{d}-1},t)
   \;=\;
   \bigl(v_1,\dots,v_{\mathrm{d}-1},
         \underbrace{t,\dots,t}_{\mathrm{dim}-\mathrm{d}+1}\bigr).
\end{aligned}
\end{equation}

\begin{figure}[H]
    \centering
\includegraphics[width=0.4\textwidth,height=0.4\textheight,keepaspectratio]{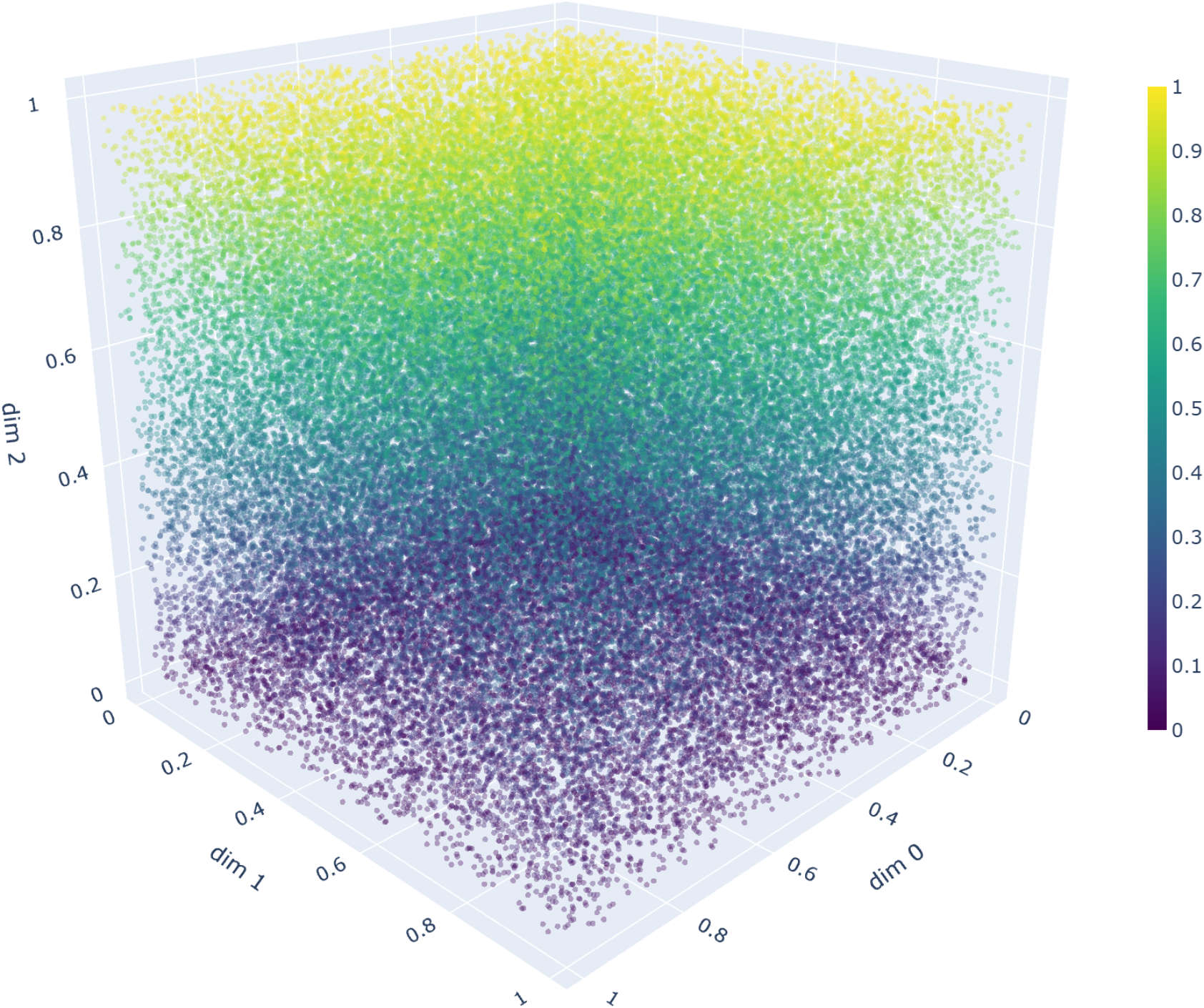}
\caption{Example visualization of the uniform dataset, with $d=3$, $\dim=4$. Showing all $4$ dimensions.}
\label{fig:uniform}
\end{figure}
\newpage
\section{Supplementary results}\label{Asec:Supplementary results}
\begin{enumerate}
    \item Supplementary tables for the bagging and smoothing experiments \ref{Asec:Supplementary tables for the bagging and smoothing experiments}:
        \begin{itemize}
            \item Table of optimal raw results (MSE) for MLE \ref{fig:MLE.MSE.smooth}
            \item Table of optimal raw results (MSE) for TLE \ref{fig:TLE.MSE.smooth}
            \item Table of optimal raw results (MSE) for MADA \ref{fig:MADA.MSE.smooth}
        \end{itemize}
    \item Supplementary bar-chart figures for the effect of $r$ and $B$ experiments \ref{Asec:supplementarybarchartfigure}:
        \begin{itemize}
            \item Bar chart with $B$ varying for TLE \ref{fig:TLE.barchart.Nbag}
            \item Bar chart with $B$ varying for MADA \ref{fig:MADA.barchart.Nbag}
            \item Bar chart with $r$ varying for TLE \ref{fig:TLE.barchart.sr}
            \item Bar chart with $r$ varying for MADA \ref{fig:MADA.barchart.sr}
        \end{itemize}
    \item Supplementary heatmaps for $B$ and $r$ interaction through comparing bagging and baseline \ref{Asec:supplementarysrB}:
        \begin{itemize}
            \item Interaction heatmap between $B$ and $r$ for TLE \ref{fig:TLE.heatmap.sr.Nbag.mse.difference}
            \item Interaction heatmap between $B$ and $r$ for MADA \ref{fig:MADA.heatmap.sr.Nbag.mse.difference}
        \end{itemize}
    \item Supplementary heatmaps for $k$ and $r$ interaction through comparing bagging and baseline \ref{Asec:supplementarysrk}:
        \begin{itemize}
            \item Interaction heatmap between $k$ and $r$ for TLE \ref{fig:TLE.heatmap.sr.k.mse.difference}
            \item Interaction heatmap between $k$ and $r$ for MADA \ref{fig:MADA.heatmap.sr.k.mse.difference}
        \end{itemize}
\end{enumerate}

\subsection{Supplementary tables for the bagging and smoothing experiments}\label{Asec:Supplementary tables for the bagging and smoothing experiments}
\begin{table}[H]
    \centering
    \caption{Table showing the optimal results for the MLE estimator across the different bagging variants with smoothing. The shown $k$ and $r$ values are the ones that achieved the lowest MSE across experiments for that specific dataset and estimator variant. While the MSE is the lowest one, the Variance and Bias-squared represent the decomposition of this MSE, and therefore, not necessarily the lowest on their own. The table is colored using a per-dataset (one for each row) logarithmic color scale based on the MSE to indicate large to small values as mild red to mild blue colors across the table cells (estimator variants). The best results in a row are additionally written with a bold font. This raw data is the basis of the MLE part of the radar charts analyzed in Section \ref{sec:results} of the main paper.}\vspace*{4mm}
\includegraphics[width=\textwidth,height=\textheight,keepaspectratio]{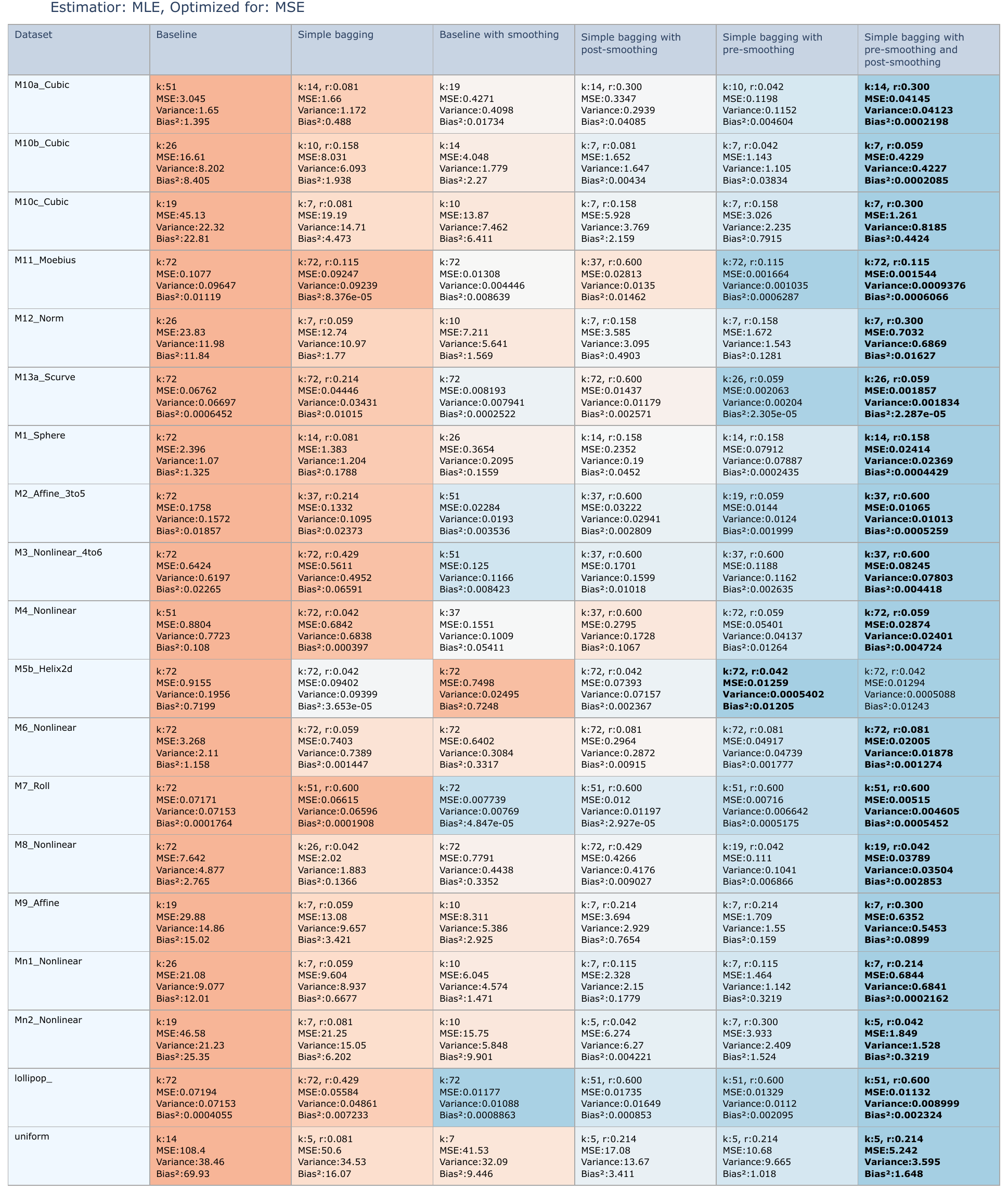}
\label{fig:MLE.MSE.smooth}
\end{table}

\begin{table}[H]
    \centering
    \caption{Table showing the optimal results for the TLE estimator across the different bagging variants with smoothing. The shown $k$ and $r$ values are the ones that achieved the lowest MSE across experiments for that specific dataset and estimator variant. While the MSE is the lowest one, the Variance and Bias-squared represent the decomposition of this MSE, and therefore, not necessarily the lowest on their own. The table is colored using a per-dataset (one for each row) logarithmic color scale based on the MSE to indicate large to small values as mild red to mild blue colors across the table cells (estimator variants). The best results in a row are additionally written with a bold font. This raw data is the basis of the TLE part of the radar charts analyzed in Section \ref{sec:results} of the main paper.}\vspace*{4mm}
\includegraphics[width=\textwidth,height=\textheight,keepaspectratio]{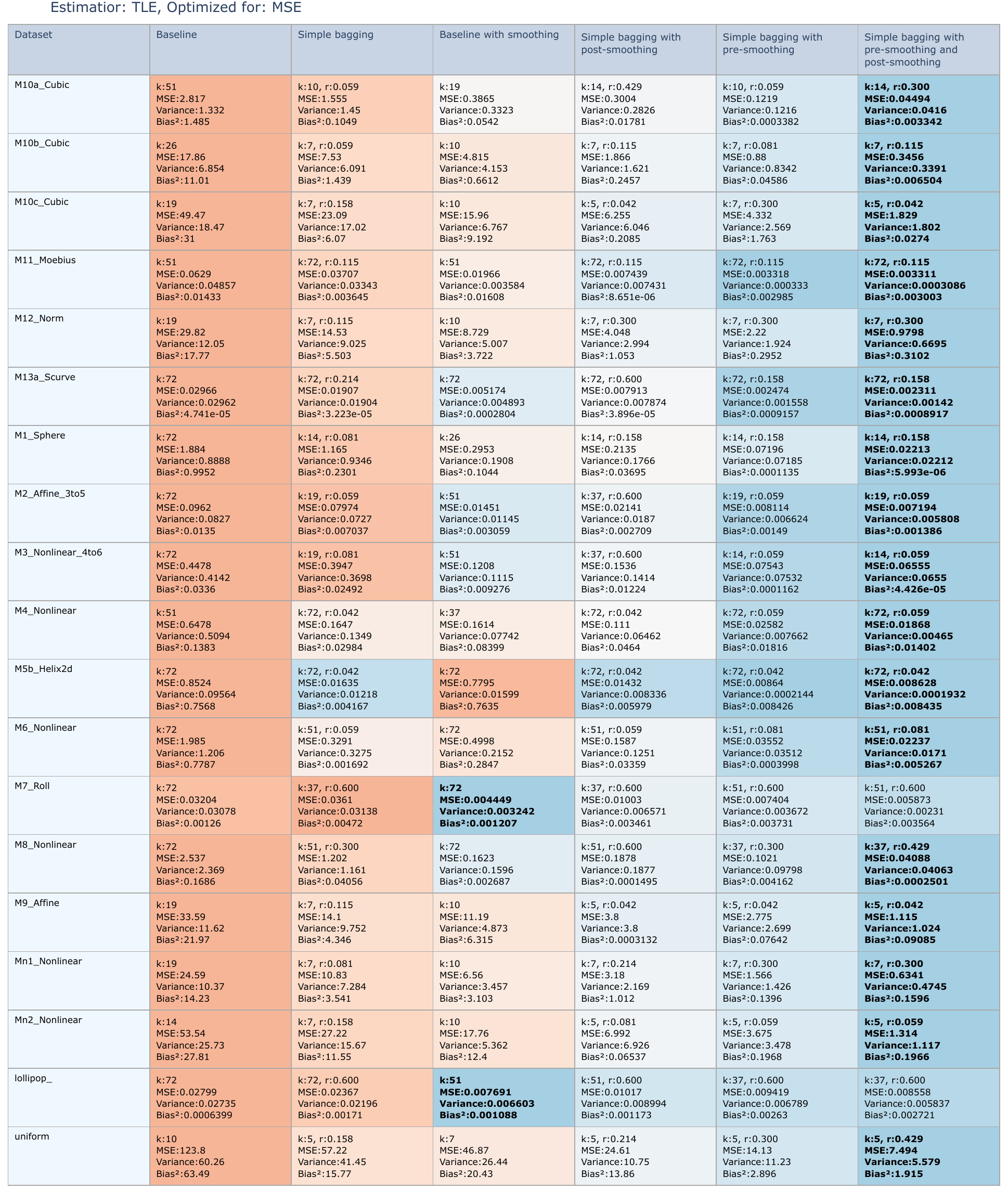}
\label{fig:TLE.MSE.smooth}
\end{table}

\begin{table}[H]
    \centering
    \caption{Table showing the optimal results for the MADA estimator across the different bagging variants with smoothing. The shown $k$ and $r$ values are the ones that achieved the lowest MSE across experiments for that specific dataset and estimator variant. While the MSE is the lowest one, the Variance and Bias-squared represent the decomposition of this MSE, and therefore, not necessarily the lowest on their own. The table is colored using a per-dataset (one for each row) logarithmic color scale based on the MSE to indicate large to small values as mild red to mild blue colors across the table cells (estimator variants). The best results in a row are additionally written with a bold font. This raw data is the basis of the MADA part of the radar charts analyzed in Section \ref{sec:results} of the main paper.}\vspace*{4mm}
\includegraphics[width=\textwidth,height=\textheight,keepaspectratio]{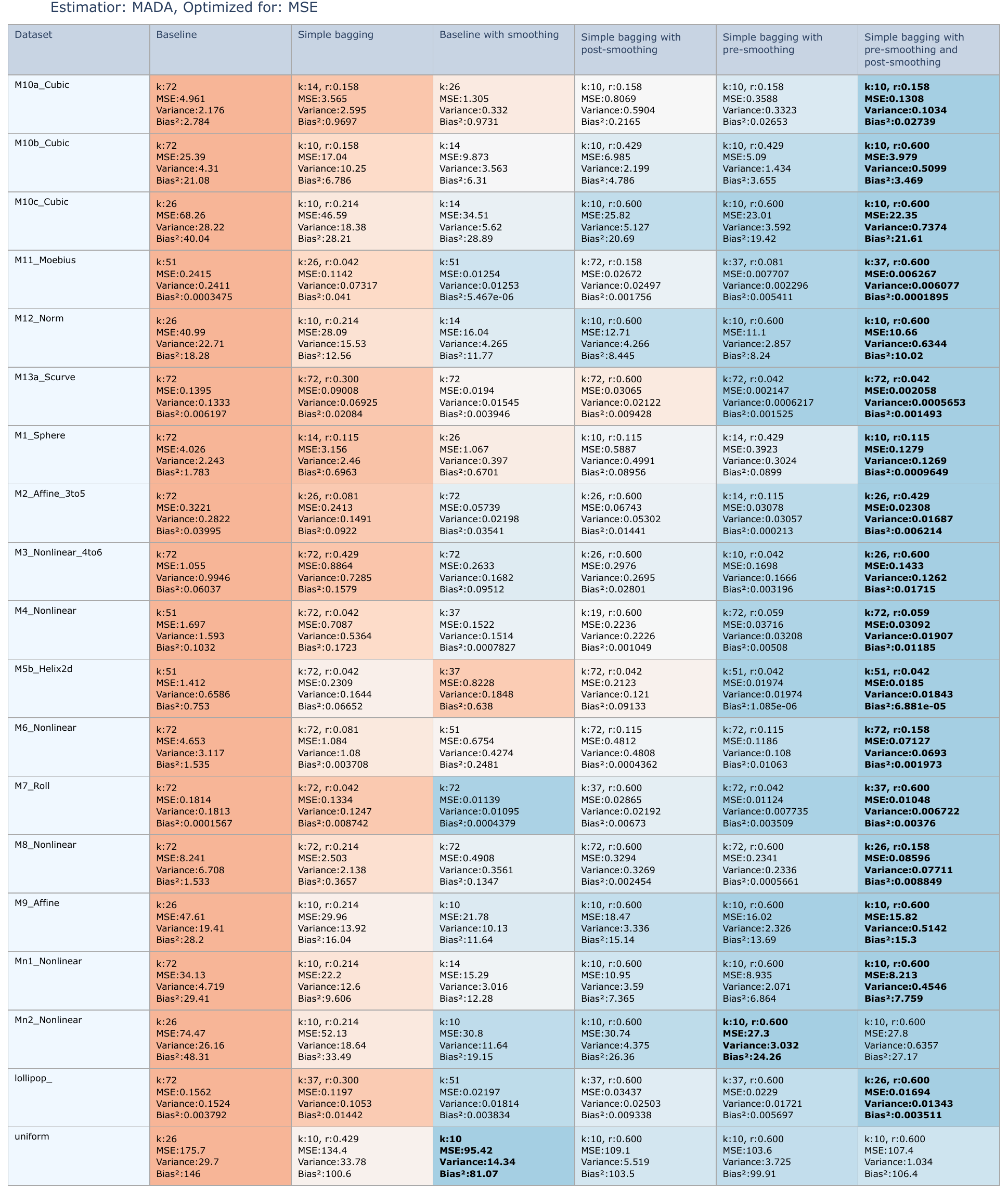}
\label{fig:MADA.MSE.smooth}
\end{table}

\subsection{Supplementary bar-chart figures for the effect of $r$ and $B$ experiments}\label{Asec:supplementarybarchartfigure}
\begin{figure}[H]
    \centering
\includegraphics[width=\textwidth,height=\textheight,keepaspectratio]{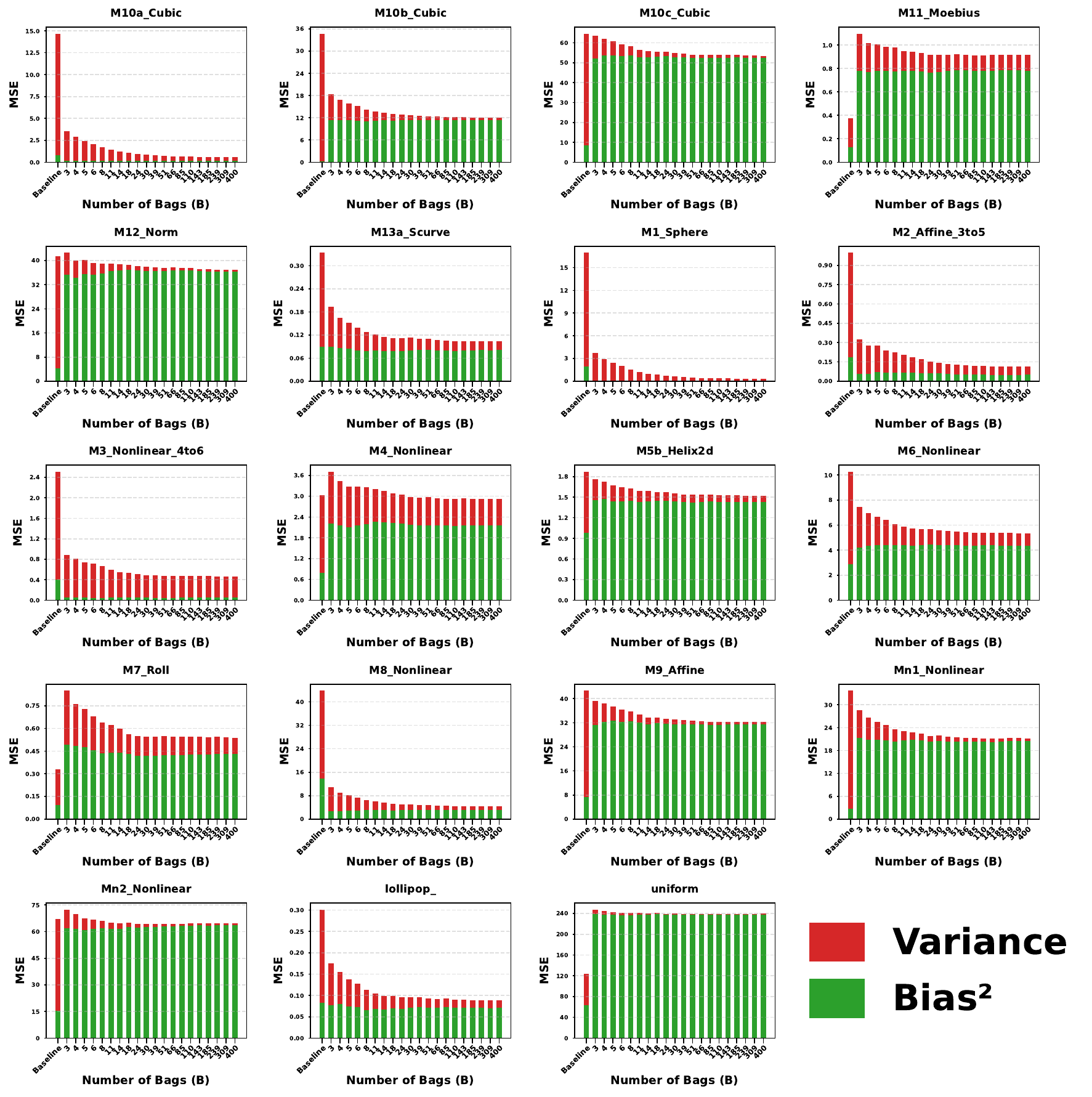}
\caption{19 separate bar charts, one for each data manifold. On the x-axis is a range of $B$ (number of bags) hyper-parameters, set for the simple bagged estimator used to estimate LID, except for the first bar from the left, which represents the baseline estimator. The baseline estimator selected is the TLE. On the y-axis are the raw values of the MSE achieved by the given estimator. The bars are vertically split into variance and $\text{bias}^2$ parts to illustrate the mean squared error decomposition, signaled by green and red colors. See Section \ref{number of bags} for the detailed experimental setup and Section \ref{results:Number of bags test} for the in-depth explanation of these results (originally illustrated with the MLE version of this plot).}
\label{fig:TLE.barchart.Nbag}
\end{figure}
\begin{figure}[H]
    \centering
\includegraphics[width=\textwidth,height=\textheight,keepaspectratio]{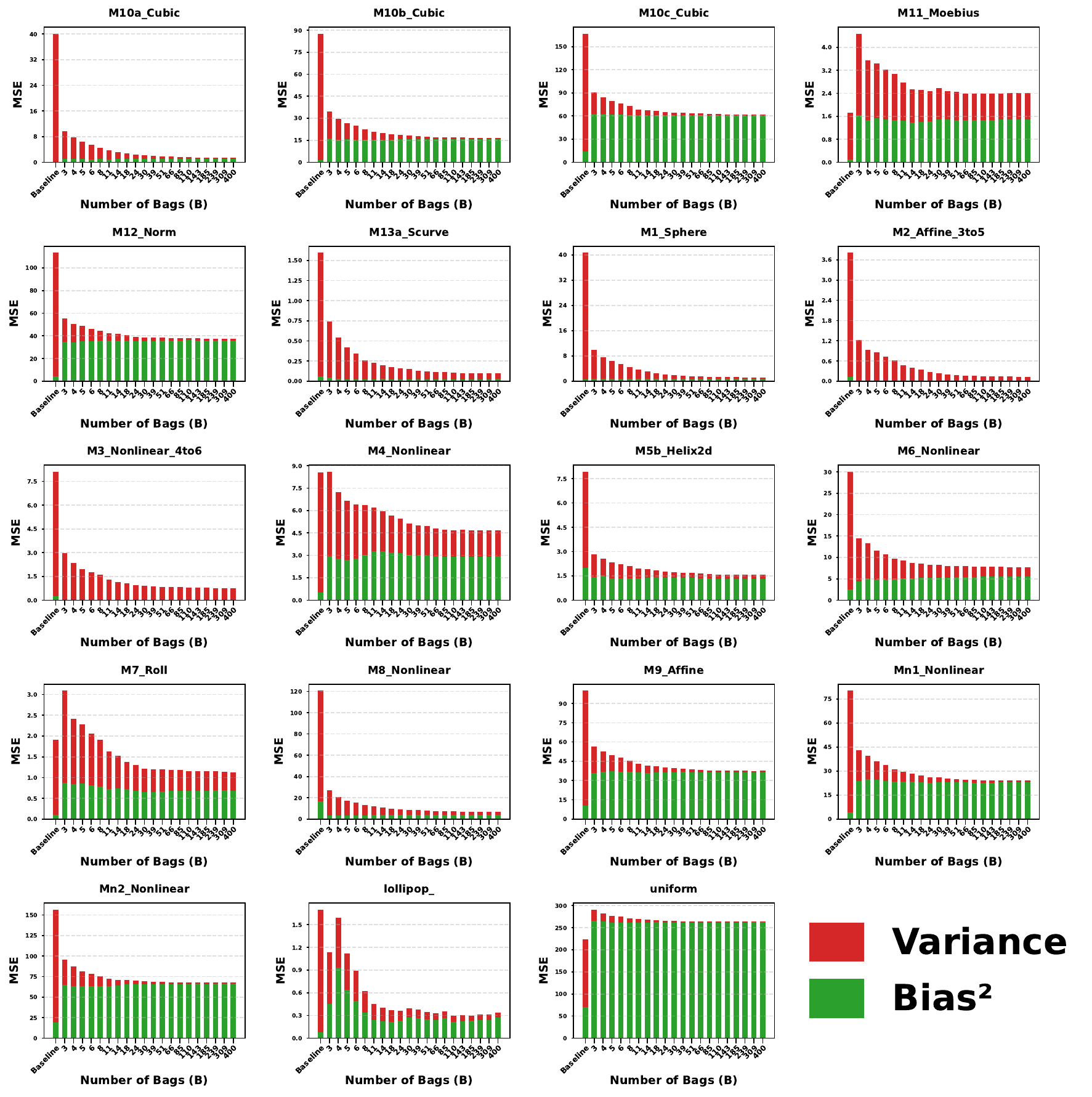}
\caption{19 separate bar charts, one for each data manifold. On the x-axis is a range of $B$ (number of bags) hyper-parameters, set for the simple bagged estimator used to estimate LID, except for the first bar from the left, which represents the baseline estimator. The baseline estimator selected is the MADA. On the y-axis are the raw values of the MSE achieved by the given estimator. The bars are vertically split into variance and $\text{bias}^2$ parts to illustrate the mean squared error decomposition, signaled by green and red colors. See Section \ref{number of bags} for the detailed experimental setup and Section \ref{results:Number of bags test} for the in-depth explanation of these results (originally illustrated with the MLE version of this plot).}
\label{fig:MADA.barchart.Nbag}
\end{figure}
\textbf{Supplementary bar-charts with varying $r$ for TLE and MADA}
\begin{figure}[H]
    \centering
\includegraphics[width=\textwidth,height=\textheight,keepaspectratio]{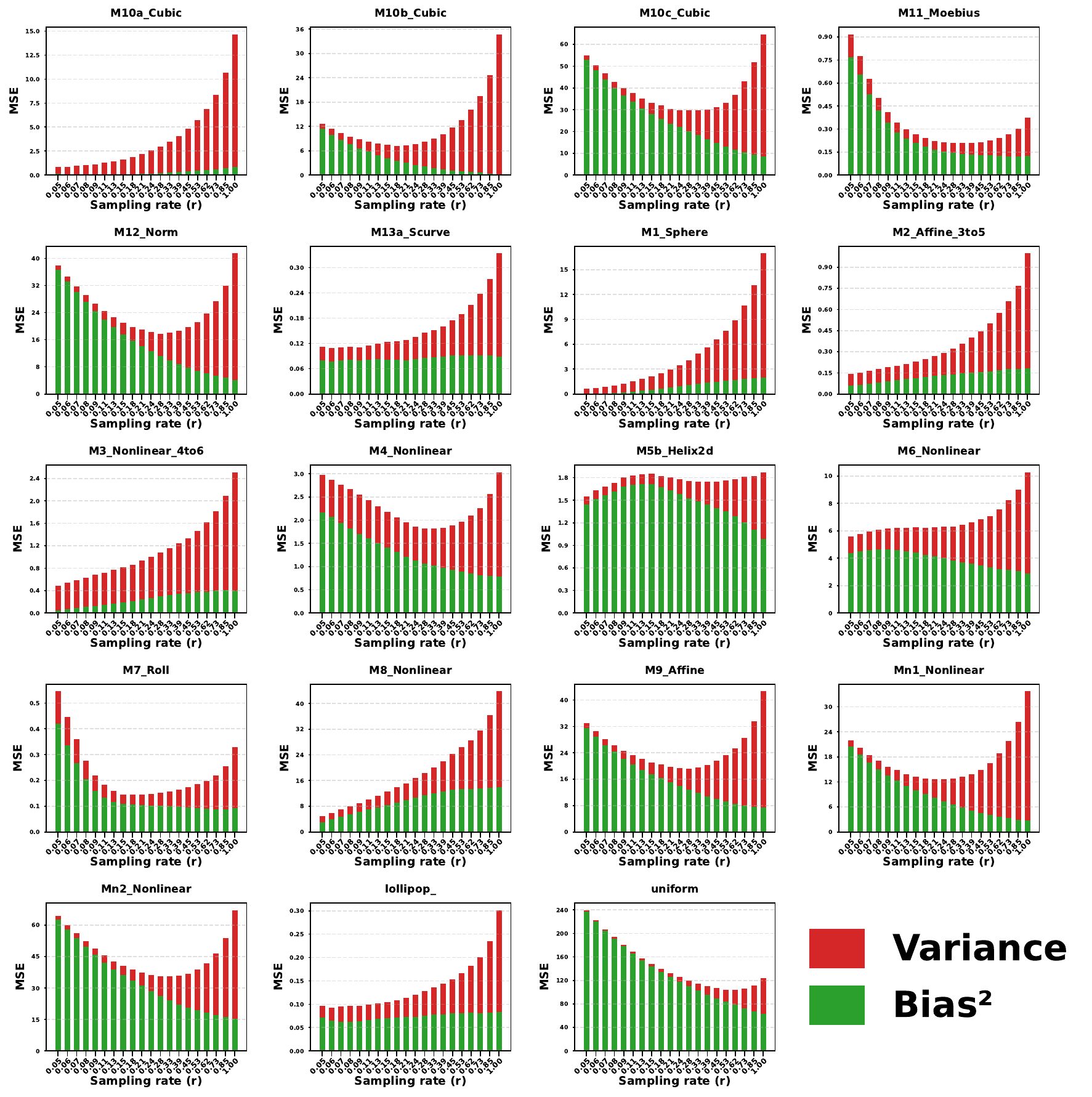}
\caption{19 separate bar charts, one for each data manifold. On the x-axis is a range of $r$ (sampling rate) hyper-parameters, set for the simple bagged estimator used to estimate LID, including the last bar from the left, which represents the baseline estimator but is equivalent with the $r=1$ case. The baseline estimator selected is the TLE. On the y-axis are the raw values of the MSE achieved by the given estimator. The bars are vertically split into variance and $\text{bias}^2$ parts to illustrate the mean squared error decomposition, signaled by green and red colors. See Section \ref{sec:Experimental methodology} of the main paper for the detailed experimental setup and Section \ref{sec:results} for the in-depth explanation of results (originally illustrated with the MLE version of this plot).}
\label{fig:TLE.barchart.sr}
\end{figure}
\begin{figure}[H]
    \centering
\includegraphics[width=\textwidth,height=\textheight,keepaspectratio]{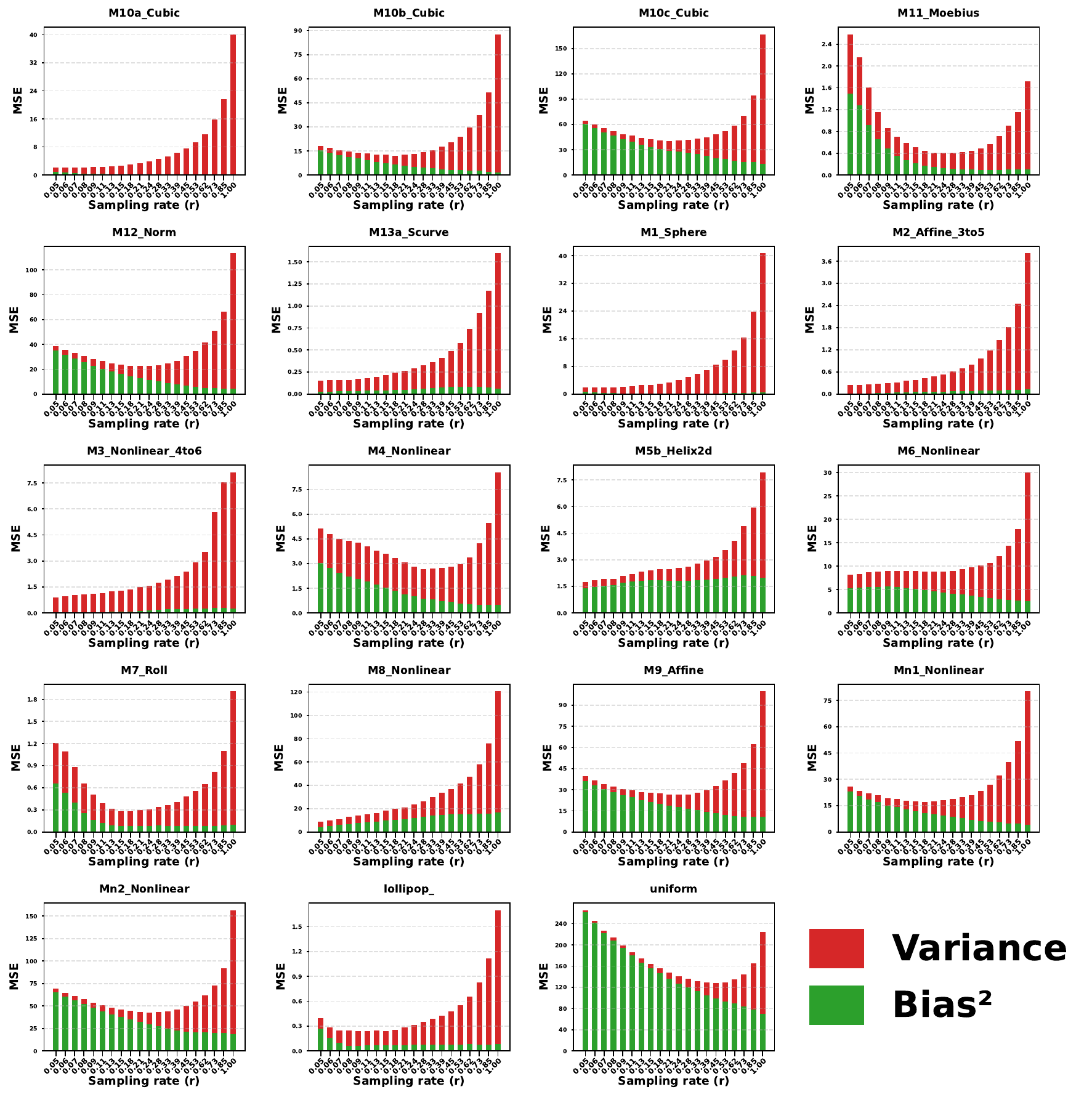}
\caption{19 separate bar charts, one for each data manifold. On the x-axis is a range of $r$ (sampling rate) hyper-parameters, set for the simple bagged estimator used to estimate LID, including the last bar from the left, which represents the baseline estimator but is equivalent with the $r=1$ case. The baseline estimator selected is the MADA. On the y-axis are the raw values of the MSE achieved by the given estimator. The bars are vertically split into variance and $\text{bias}^2$ parts to illustrate the mean squared error decomposition, signaled by green and red colors. See Section \ref{sec:Experimental methodology} of the main paper for the detailed experimental setup and Section \ref{sec:results} for the in-depth explanation of results (originally illustrated with the MLE version of this plot).}
\label{fig:MADA.barchart.sr}
\end{figure}
\subsection{Supplementary heatmaps for $B$ and $r$ interaction through comparing bagging and baseline}\label{Asec:supplementarysrB}
\begin{figure}[H]
    \centering
\includegraphics[width=0.98\textwidth,height=0.98\textheight,keepaspectratio]{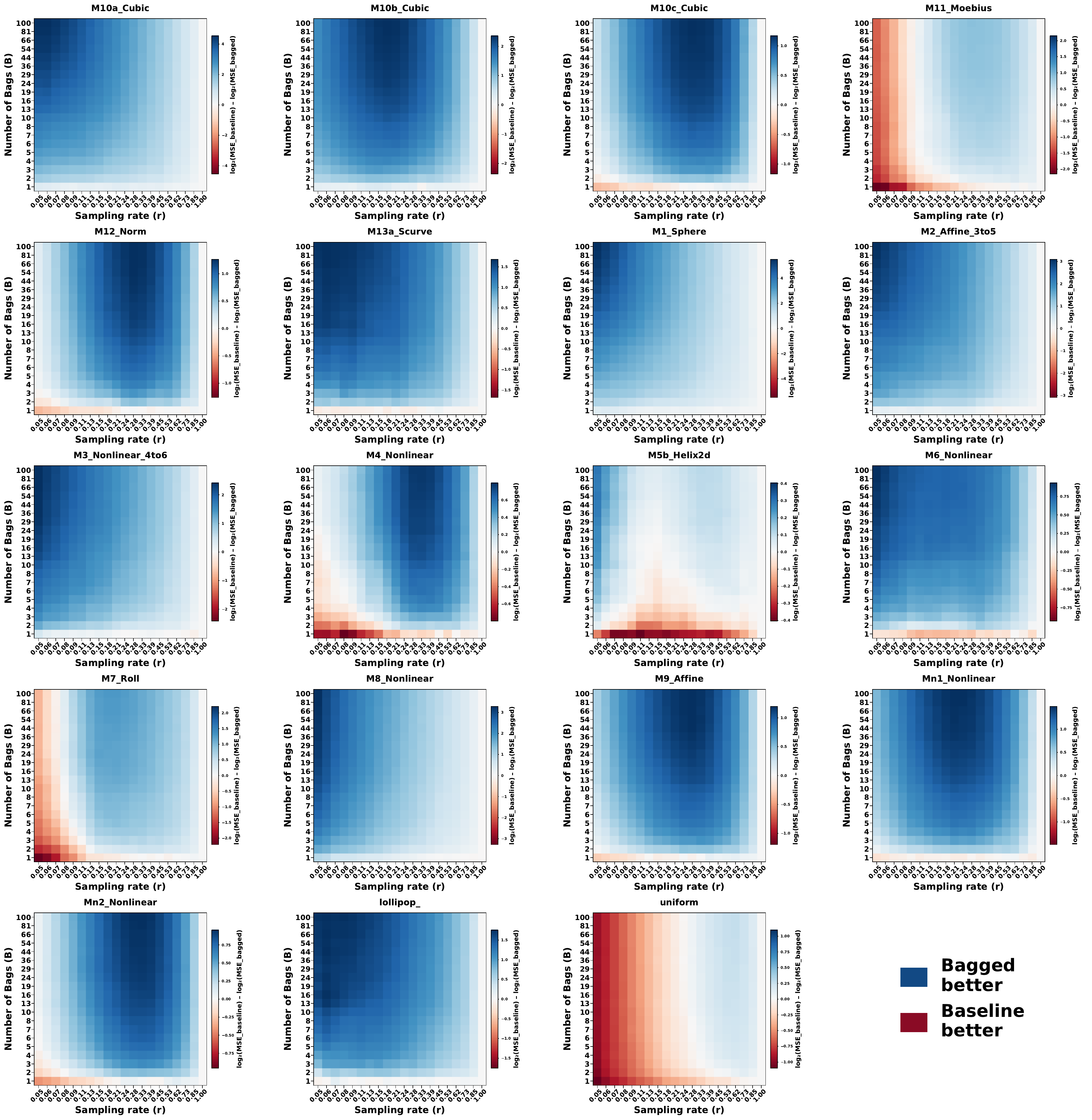}
\caption{19 separate heat maps, one for each data manifold. On the x-axis is a range of $r$ (sampling rate) hyper-parameters, while on the y-axis is a range of $B$ (number of bags) hyper-parameters. At each coordinate, the combination of the two hyper-parameters is set for the simple bagged estimator used to estimate LID. The baseline estimator selected is the TLE. The squares on the grid are colored based on the logarithm of the ratio between the MSE achieved by the baseline estimator and the MSE of the simple bagged estimator using that specific hyper-parameter combination. See Section \ref{B and sr interaction} for the detailed experimental setup and Section \ref{result:Interaction of sr and B} for the in-depth explanation of results (originally illustrated with the MLE version of this plot).}
\label{fig:TLE.heatmap.sr.Nbag.mse.difference}
\end{figure}
\begin{figure}[H]
    \centering
\includegraphics[width=0.98\textwidth,height=0.98\textheight,keepaspectratio]{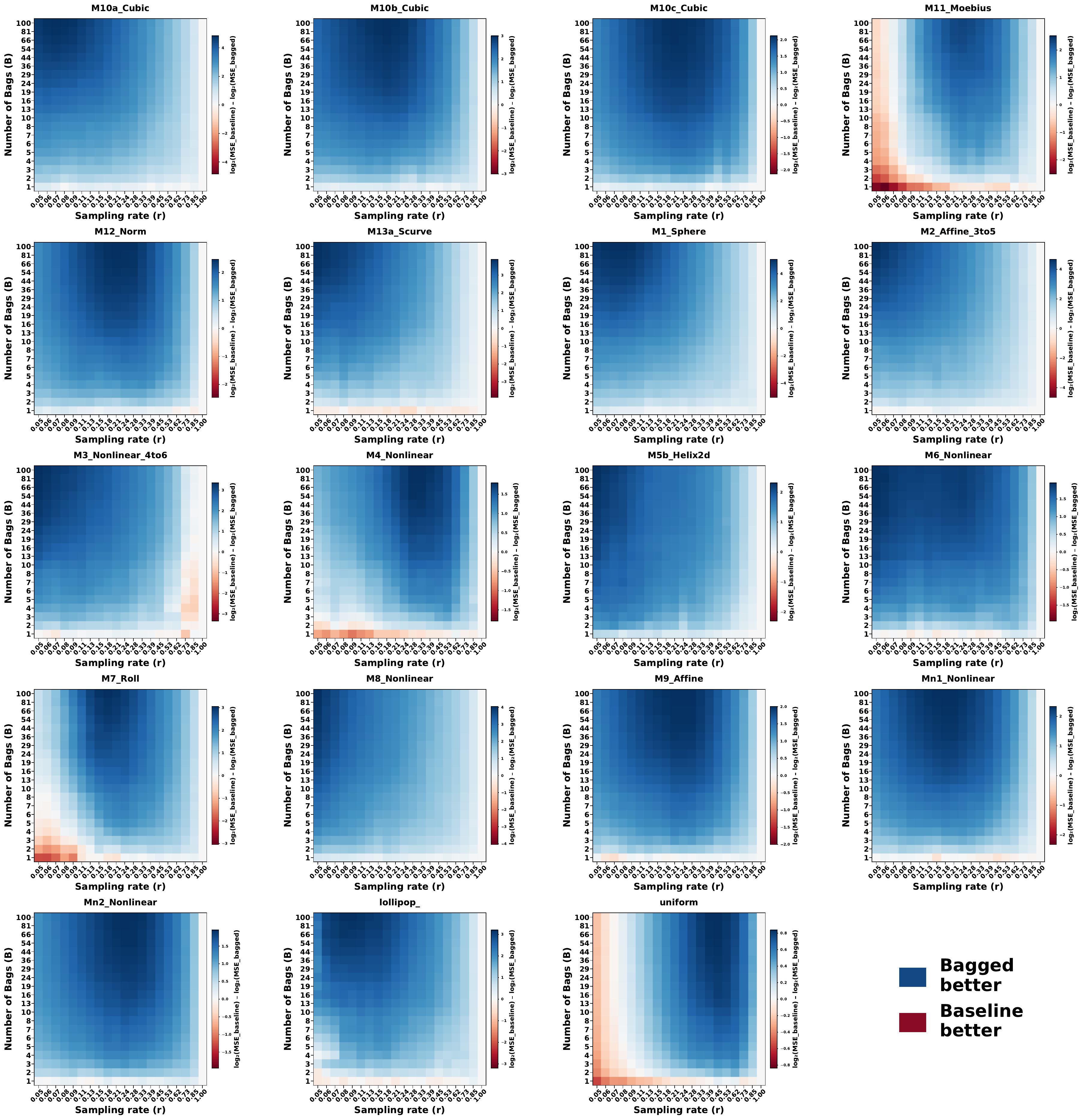}
\caption{19 separate heat maps, one for each data manifold. On the x-axis is a range of $r$ (sampling rate) hyper-parameters, while on the y-axis is a range of $B$ (number of bags) hyper-parameters. At each coordinate, the combination of the two hyper-parameters is set for the simple bagged estimator used to estimate LID. The baseline estimator selected is the MADA. The squares on the grid are colored based on the logarithm of the ratio between the MSE achieved by the baseline estimator and the MSE of the simple bagged estimator using that specific hyper-parameter combination. See Section \ref{B and sr interaction} for the detailed experimental setup and Section \ref{result:Interaction of sr and B} for the in-depth explanation of results (originally illustrated with the MLE version of this plot).}
\label{fig:MADA.heatmap.sr.Nbag.mse.difference}
\end{figure}

\subsection{Supplementary heatmaps for $k$ and $r$ interaction through comparing bagging and baseline}\label{Asec:supplementarysrk}
\begin{figure}[H]
    \centering
\includegraphics[width=0.98\textwidth,height=0.98\textheight,keepaspectratio]{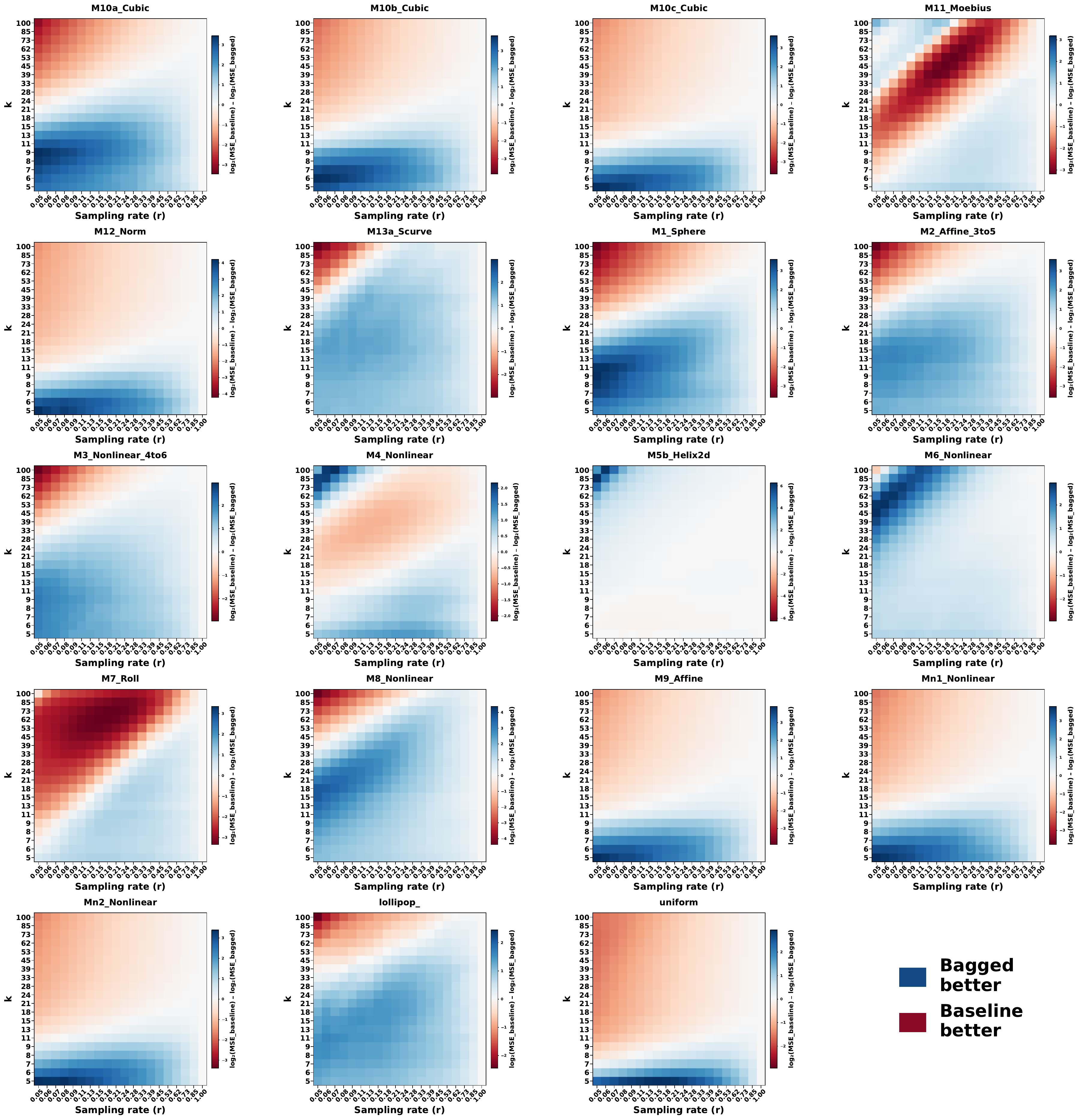}
\caption{19 separate heat maps, one for each data manifold. On the x-axis is a range of $r$ (sampling rate) hyper-parameters, while on the y-axis is a range of $k$-NN hyper-parameters. At each coordinate, the combination of the two hyper-parameters is set for the simple bagged estimator used to estimate LID. The baseline estimator selected is the TLE. The squares on the grid are colored based on the logarithm of the ratio between the MSE achieved by the baseline estimator and the MSE of the simple bagged estimator using that specific hyper-parameter combination. See Section \ref{sec:Experimental methodology} of the main paper for the detailed experimental setup and Section \ref{sec:results} for the in-depth explanation of results (originally illustrated with the MLE version of this plot).}
\label{fig:TLE.heatmap.sr.k.mse.difference}
\end{figure}

\begin{figure}[H]
    \centering
\includegraphics[width=0.98\textwidth,height=0.98\textheight,keepaspectratio]{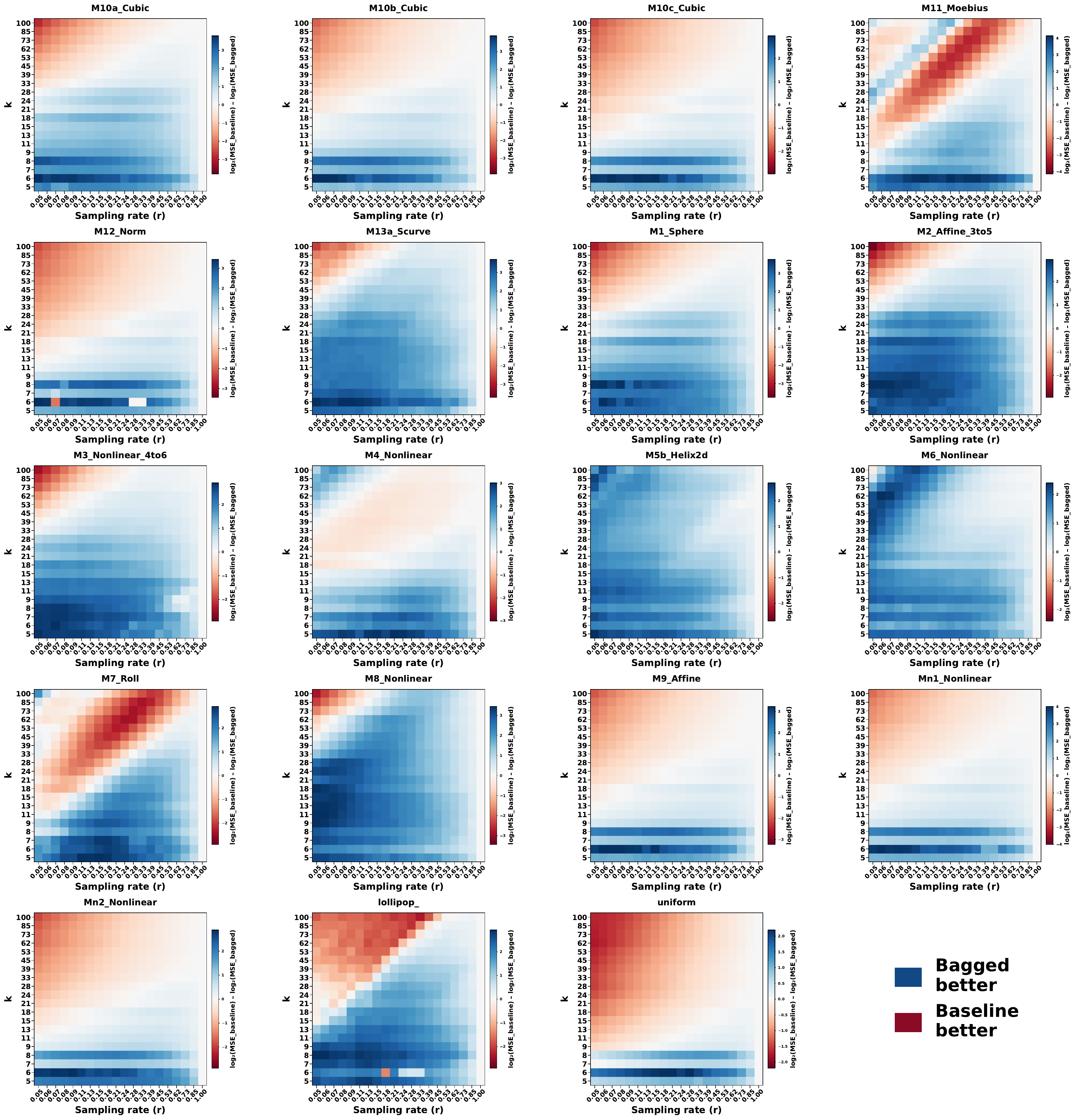}
\caption{19 separate heat maps, one for each data manifold. On the x-axis is a range of $r$ (sampling rate) hyper-parameters, while on the y-axis is a range of $k$-NN hyper-parameters. At each coordinate, the combination of the two hyper-parameters is set for the simple bagged estimator used to estimate LID. The baseline estimator selected is the MADA. The squares on the grid are colored based on the logarithm of the ratio between the MSE achieved by the baseline estimator and the MSE of the simple bagged estimator using that specific hyper-parameter combination. See Section \ref{sec:Experimental methodology} of the main paper for the detailed experimental setup and Section \ref{sec:results} for the in-depth explanation of results (originally illustrated with the MLE version of this plot).}
\label{fig:MADA.heatmap.sr.k.mse.difference}
\end{figure}

\FloatBarrier
\bibliographystyle{splncs04}
\bibliography{reference}
\end{document}